\documentclass[preprint,superscriptaddress,amsmath,amssymb,aps,
showkeys,showpacs,twoside,final,secnumarabic]{revtex4-2}

\usepackage[paperwidth=205mm,paperheight=290mm,top=17mm,bottom=25mm,
inner=17mm,outer=17mm,twoside]{geometry}
\usepackage{cmap}\defaulthyphenchar=127
\usepackage[T1,T2A]{fontenc}
\usepackage[utf8]{inputenc}
\usepackage[russian,english]{babel}
\usepackage{color}
\usepackage{graphicx}
\usepackage{dcolumn}
\usepackage{bm}
\usepackage{booktabs}

\usepackage[unicode=true,colorlinks=true,linkcolor=magenta,
urlcolor=blue,citecolor=blue,breaklinks]{hyperref}
\usepackage{multirow}
\usepackage{url}
\usepackage{breakurl}
\usepackage{amsthm}
\usepackage{amsmath}
\usepackage{amssymb} 
\usepackage[linesnumbered,ruled,vlined,noend]{algorithm2e}

\SetKwInOut{Input}{Input}
\SetKwInOut{Output}{Output}

\usepackage{tabularx} 
\usepackage{booktabs}

\usepackage[figuresright]{rotating}

 \usepackage{afterpage}
\usepackage{float}  

\newtheorem{thm}{Theorem}[section]
\newtheorem{prop}{Proposition}[section]
\newtheorem{lem}{Lemma}[section]
\newtheorem{cor}{Corollary}[section]
\newtheorem*{rmk}{Remark}
\newtheorem*{defn}{Definition}

\newcount\issue
\newcount\Vol
\newcount\numb
\usepackage{fancyhdr}
\def\Vol{\textbf{80}}
\def\numb{x}
\begin{document}

\title{Kohn--Sham Spectral Embedding on Sparse Graphs\\ at the Nishimori Temperature for Image Classification}

\def\addressa{Department of Computer Engineering, South-West State University
(SWSU), Kursk, 305040 Russia}
\def\addressb{T8 LLC, R\&D Department, Information Theory Group,
Moscow, 107076 Russia}
\author{\firstname{V.~S.}~\surname{Usatyuk}}
\email[E-mail: ]{L@Lcrypto.com }
\affiliation{\addressa}\affiliation{\addressb}
\author{\firstname{D.~A.}~\surname{Sapozhnikov}}\affiliation{\addressb}
\author{\firstname{S.~I.}~\surname{Egorov}}\affiliation{\addressa}

\received{xx.xx. 2026}\revised{xx.xx.2026}\accepted{xx.xx.2026}

\begin{abstract}
We propose Kohn--Sham Spectral Embedding (KSSE), a physics-inspired energy-based
model that replaces the dense top-layer classifier of convolutional neural networks with a
sparse-graph spectral embedding evaluated at the Nishimori temperature of an associated
Random-Bond Ising Model~(RBIM)---the spectral detectability threshold at which class structure
becomes marginally distinguishable from feature disorder. By mapping pre-trained feature
representations onto quasi-cyclic low-density parity-check graphs and constructing a
regularized Laplacian that plays the role of an effective Kohn--Sham Hamiltonian, we obtain
$D$ independent spectral problems---one per feature channel---solvable in
$\mathcal{O}(N\log N + k_{\mathrm{mode}}^2 N)$ time via the Fast Fourier Transform on circulant
blocks (a consequence of Pontryagin self-duality of the finite cyclic group), with low-mode Rayleigh--Ritz refinement (empirically  $k_{\mathrm{mode}}=5$)---in the language of solid-state physics, a $k\cdot p$ effective-mass reduction on a one-dimensional ring crystal in which the circulant support is the perfect crystal, the data weights a slowly varying impurity potential, and the Nishimori crossing a Fermi level tuned to the band edge (Appendix~\ref{app:kp}). Graph topology is optimized through \emph{star-domain surgery}: rather than eliminating all frustrated cycles---an
impossible task that would destroy the codewords carrying class information---we construct
edge shifts that create certified local convexity around codewords while bounding residual
frustration to $\rho(B_\gamma)\leq 1+\delta$. Multi-scale fractal analysis ($D_2$ spectrum)
and the fractal learning-rate landscape certify the transition from rough landscapes
($D_2>3$) to star-domain basins ($D_2<1$). We establish a rigorous theoretical framework:
a generalized Ihara--Bass identity with a sharp spectral threshold linking belief propagation
to the regularized Laplacian; a non-backtracking growth trichotomy---trees are spectrally
nilpotent, simple cycles subcritical at every finite temperature, branching 2-cores
unstable---with frustration entering as a gauge-invariant $\mathbb{Z}_2$ flux (an
Aharonov--Bohm phase); a trapping-set spectral test; an even-subgraph expansion identifying
codewords with constructive partition-function terms and placing permanent/Pfaffian
tractability on the planar and toroidal side; exact additive separability of the feature
channels, with a cup-product obstruction delimiting where it breaks; a loop-series
exchange--correlation bound certifying sub-percent per-vertex factorization error at
girth~$\geq 6$; a star-domain convexity certificate with explicit basin radius; finite-time
convergent surgery; a quasi-stationarity perturbation bound; and a Rayleigh--Ritz refinement
bound for near-circulant operators. Evaluated on ImageNet-1000 with frozen EfficientNet-B4
features ($D=1792$) under a transductive evaluation protocol---where test images are
spectrally embedded jointly with frozen training representatives in a shared sparse
graph---KSSE achieves \textbf{88.93\%} Top-1 accuracy using ${\approx}21.24$\,M parameters,
outperforming Swin-L (197\,M, 86.4--87.3\%) and matching the lower end of ViT-H/14 (632\,M,
88.0--89.5\%), which are evaluated under standard inductive protocols, while reducing model
size by $10\times$ and $30\times$, respectively.
\end{abstract}

\par
\maketitle
\thispagestyle{fancy}

%==============================================
\section{\label{sec:intro}Introduction}
%==============================================

Large-scale image classification has been overwhelmingly dominated by deep convolutional
neural networks~(CNNs) and vision transformers. While architectures such as
EfficientNet-B4~\cite{Tan2019} provide robust feature extraction, their terminal fully connected
layers require fixed input sizes and scale poorly with an increasing number of classes $K$.
An alternative paradigm, rooted in statistical mechanics, treats pattern recognition as
inference in an energy-based model: feature vectors define interactions between binary spins on a
sparse graph, and classification reduces to finding low-energy collective configurations.

Prior work by the authors~\cite{Usatyuk2025} introduced a foundational framework mapping
high-dimensional CNN features onto spins of an RBIM defined on Multi-Edge Type Quasi-Cyclic LDPC
graphs. That approach demonstrated that topological defects---specifically trapping sets---could be
characterized by invariants such as Betti numbers and the continuous genus $\widehat{A}$, which
directly degrade spectral embedding quality via negative Bethe--Hessian eigenvalues. Accuracies of
98.7\% on ImageNet-10 and 84.92\% on ImageNet-100 were achieved with merely 1.33\,M parameters.

In this work we bridge statistical physics---specifically the Kohn--Sham formalism~\cite{Kohn1965}
and spin-glass theory at the Nishimori temperature~\cite{Nishimori1980,Nishimori1981}---with large-scale pattern
recognition. The proposed KSSE framework fundamentally shifts the paradigm from rigid topological
elimination to \emph{star-domain surgery}: constructing graph modifications that create certified
local convexity (star domains) around codewords $TS(a,0)$ in the Bethe free-energy landscape while
bounding---not eliminating---residual frustration from surviving trapping sets $TS(a,b{\neq}0)$.

The core idea is as follows. Given a sparse quasi-cyclic LDPC graph whose circulant structure
defines where interactions live, and pre-trained visual features that determine their strengths,
we construct for each of the $D$ feature channels an independent regularized Laplacian
$L_\beta^{(k)}$ (the Bethe--Hessian at inverse temperature $\beta$). The Kohn--Sham decomposition
solves these $D$ single-channel eigenproblems independently---exactly as density functional theory
solves $N$ non-interacting single-particle equations to reproduce a ground-state density. The
exchange-correlation energy, which in DFT captures electron--electron interactions beyond the Hartree
approximation, here arises from graph cycles and vanishes on trees; under a Dobrushin-type spectral
condition it decays exponentially with the girth, and after star-domain surgery convergence it is
$\mathcal{O}(\delta)$ with $\delta\to 0$. Crucially, no Hohenberg--Kohn existence theorem is needed
because feature channels are independent by construction---and, as an explicit flux counterexample
shows, none exists (Proposition~\ref{prop:no_hk}).

The main contributions of this paper are:

\begin{enumerate}
  \item A generalized Ihara--Bass identity with an exact edge--vertex spectral correspondence and
        a determinant identity proved by incidence factorization (Theorem~\ref{thm:bp_laplacian}),
        a sharp positive-semidefiniteness threshold (Lemma~\ref{lem:threshold}), and a
        non-backtracking growth trichotomy (Lemma~\ref{lem:trichotomy}) separating branching
        (instability magnitude) from frustration (a gauge-invariant $\mathbb{Z}_2$ flux acting as
        an Aharonov--Bohm phase; Lemmas~\ref{lem:gauge}, \ref{lem:ab}), yielding a trapping-set
        spectral test (Theorem~\ref{thm:trapping_set}).

  \item An even-subgraph expansion of the RBIM partition function
        (Proposition~\ref{prop:even_subgraph}) identifying codewords $TS(a,0)$ with constructive
        terms, with a tractability hierarchy from trees (exact factorization) through unicyclic
        and planar/toroidal graphs (permanent/Pfaffian flux sectors) to general graphs
        (\#P-hard permanent-like interference) (Corollary~\ref{cor:hierarchy}).

  \item A mean-field Kohn--Sham decomposition theorem~(Theorem~\ref{thm:ksse_decomp}) with exact
        additive separability (Theorem~\ref{thm:ks_separability}; joint critical temperature
        $\beta_N^{\mathrm{full}}=\min_k\beta_N^{(k)}$), a loop-series exchange-correlation bound
        with an explicit Dobrushin-type hypothesis (Proposition~\ref{prop:bethe_xc}) certifying
        sub-percent per-vertex error at girth $\geq 6$ (Corollary~\ref{cor:girth6}), the absence
        of a Hohenberg--Kohn analogue (Proposition~\ref{prop:no_hk}), and a cup-product
        obstruction delimiting the decomposition (Proposition~\ref{prop:cup}).

  \item A star-domain certificate with explicit basin radius $R=g/(2L_3)$ and linear convergence
        rate (Theorem~\ref{thm:star_domain}), and finite-time convergent surgery with bounded
        residual frustration $\rho(B_\gamma)\leq 1+\delta$ via a lexicographic frustrated-cycle
        potential (Theorem~\ref{thm:self_consistency}).

  \item A quasi-stationarity perturbation bound (Theorem~\ref{thm:adiabatic}) and a
        Rayleigh--Ritz refinement bound for near-circulant post-surgery operators
        (Lemmas~\ref{lem:circulant}, \ref{lem:rayleigh}), with $k_{\mathrm{mode}}=5$ an
        empirically sufficient parameter (Remark~\ref{rmk:kmode}).

  \item A multi-scale fractal analysis framework~(Sec.~\ref{sec:fractal}) linking the correlation
        dimension $D_2$ to surgical effectiveness, and a complete FFT-based algorithmic pipeline
        evaluated on ImageNet-1000 at \textbf{88.93\%} Top-1 accuracy with ${\approx}21.24$\,M
        parameters under a transductive evaluation protocol (Sec.~\ref{sec:methodology}).
\end{enumerate}

The remainder of the paper is organized as follows.
Section~\ref{sec:qc_graphs} defines quasi-cyclic sparse graphs and the affinity tensor that maps
visual features to Ising couplings. Section~\ref{sec:rbim} introduces the RBIM Hamiltonian, the
regularized Laplacian (Bethe--Hessian), and the Nishimori temperature.
Section~\ref{sec:kohn_sham} develops the Kohn--Sham mean-field decomposition, including additive
separability, exchange-correlation bounds, and Pontryagin-duality-based FFT computation.
Section~\ref{sec:landscape} connects belief-propagation defects to trapping sets via their shared
2-core topology and introduces fractal analysis. Section~\ref{sec:star_domain_thm} presents
the star-domain certificate and the convergent surgery theorem. Section~\ref{sec:algorithm}
describes the complete algorithmic pipeline. Sections~\ref{sec:methodology}--\ref{sec:comparison}
report methodology, experimental results---including a protocol-matched comparison isolating the
spectral-embedding benefit from the transductive advantage (Sec.~\ref{sec:transductive})---and
comparison with state of the art. Section~\ref{sec:discussion} discusses limitations and complexity.
All formal proofs are given in Appendix~\ref{app:proofs};
Tables~\ref{tab:theorem_summary_1part},~\ref{tab:theorem_summary_2part}
summarize all theoretical results. Appendix~\ref{app:kp} develops a fully
worked special case---the $N=45{,}000$ spherical graph of
Secs.~\ref{sec:experiments}--\ref{sec:comparison}---in the language of
solid-state physics: on a spherical graph without cup products, KSSE is
$k\cdot p$ effective-mass theory on a one-dimensional ring crystal, with
the circulant support as the perfect crystal, the data weights as a slowly
varying impurity potential, the Nishimori crossing as a Fermi level tuned
to the band edge, and $k_{\mathrm{mode}}=5$ as the number of Bloch
functions in the $k\cdot p$ secular equation.

%==============================================
\section{\label{sec:qc_graphs}Quasi-Cyclic Sparse Graphs, Bond Weights,
and Trapping Sets}
%==============================================

The foundation of KSSE is a sparse QC-LDPC graph whose algebraic structure determines where
interactions (bonds) exist and how visual features assign their strengths. This section defines the
graph families, the affinity tensor mapping CNN features to signed bond weights, and the trapping-set
ontology that governs energy-landscape defects.

\subsection{Circulant Permutation Matrices and QC-LDPC Structure}

A quasi-cyclic (QC) LDPC code has a sparse parity-check matrix $H$ composed of $m_b\times n_b$
blocks, each either a zero matrix or a circulant permutation matrix~(CPM) $P_a$ ($a\in\mathbb{Z}/p
\mathbb{Z}$), defined by shifting the identity matrix $I_p$ by $a$ positions, \cite{Tanner2001,Fossor04}. The exponent matrix
$E(H)$ contains shift values, and replacing every non-zero block by 1 yields the binary protograph
$M(H)\in\{0,1\}^{m_b\times n_b}$. The resulting Tanner graph exhibits a toric structure with quotient
group $\mathbb{Z}/p\mathbb{Z}$, enabling FFT-based spectral computation on each circulant block. 

\begin{defn}[Toroidal and Spherical Graph Families]\label{def:graph_families}
A \emph{toroidal} QC-LDPC family contains two or more independent circulant rings in $E(H)$;
the protograph has at least two disjoint cycles, yielding non-contractible loops after lifting.
A \emph{spherical} QC-LDPC family is built from a single circulant ring with multi-weight CPM
superpositions; all cycles are contractible. The multidiagonal bond matrix in the spherical case
has non-zero entries fixed by the protograph and values supplied by data.
\end{defn}

\noindent\textbf{Example (spherical graph).}
For $N=35\,000$, Fig.~\ref{fig:SPh_ex} shows a spherical construction with multi-weight (weight~10)
shifts $\{0,1,6,373,5210,20993,22980,23826,26410,26978\}$ modulo the circulant size. Each shift
defines a diagonal band in the lifted adjacency; superposing several shifts yields a multidiagonal
bond matrix whose sparsity pattern is fixed by the protograph and whose entries are populated from
feature affinities (Sec.~\ref{sec:qc_graphs}).

\begin{figure}[t]
    \centering
    \includegraphics[width=0.5\linewidth]{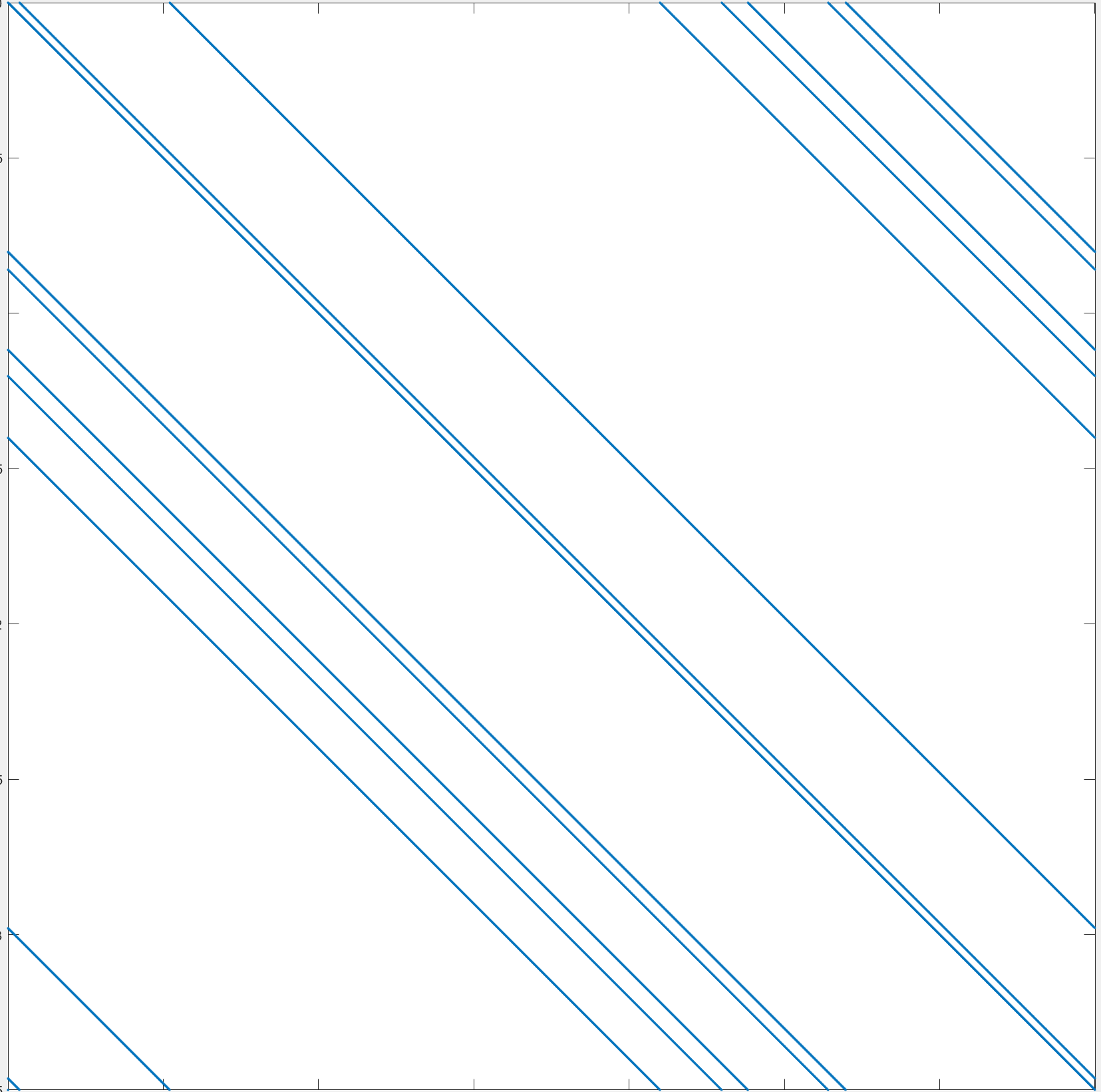}
    \caption{Spherical QC-LDPC graph construction for $N=35\,000$ nodes with column weight~10.
    Each colored diagonal band corresponds to one circulant permutation shift; their superposition
    forms the multidiagonal bond matrix that defines the Ising coupling topology. The spherical
    family ensures all cycles are contractible, enabling Pontryagin-duality-based FFT spectral
    computation on each circulant block.}
    \label{fig:SPh_ex}
\end{figure}

\subsection{Affinity Tensor: From Visual Features to Ising Couplings}

The Tanner graph $\mathcal{T}$ provides the skeleton; visual features provide the flesh.
Let $X\in\mathbb{R}^{N\times D}$ be a feature matrix and let $(\bm{r},\bm{c})$ index the
$E$ edges of $\mathcal{T}$.

\begin{defn}[Affinity Tensor]\label{def:affinity_tensor}
For each edge $e=(r_e,c_e)$ and feature $k\in\{1,\dots,D\}$, compute the Manhattan distance
$L_1[e,k]=|X_{r_e,k}-X_{c_e,k}|$. The raw affinity is
$A_{\mathrm{raw}}[e,k]=1/(1+L_1[e,k]+\varepsilon)$ with $\varepsilon\approx 10^{-7}$.
Channel-wise $z$-scoring yields the normalized tensor $A\in\mathbb{R}^{E\times D}$, and
the bond-weight matrix for channel $k$ is the sparse symmetric matrix
\begin{equation}
J_{ij}^{(k)}=\begin{cases} A[e,k], & e=(i,j)\in E(\mathcal{T}), \\ 0, & \text{otherwise.}\end{cases}
\end{equation}
Negative entries after $z$-scoring correspond to antiferromagnetic bonds; positive entries are
ferromagnetic.
\end{defn}

The QC-graph topology fixes the location of interactions (which edges exist), while data fix their
magnitudes and signs. This separation is essential for the Kohn--Sham decomposition: because each
channel $k$ carries its own coupling tensor $J_{ij}^{(k)}$, there are no cross-feature interaction
terms, making the decomposition exact by construction.

\subsection{Cycles, Girth, and Trapping Sets}

A cycle of length $2\ell$ in $\mathcal{T}$ satisfies the shift consistency condition
$\sum_{i=1}^{2\ell}(-1)^i a_i\equiv 0 \pmod{L}$, \cite{Fossor04}. Trapping sets~(TS) are subgraphs that corrupt iterative decoding, \cite{Vasic2009}:

\begin{defn}[Trapping Set]\label{def:trapping_set}
A subset $S$ of variable nodes with $|S|=a$ and odd-degree check neighbors $O(S)$ with $|O(S)|=b$
is a trapping set $TS(a,b)$. When $b=0$, all checks have even degree (codewords); when $b>0$,
frustrated cycles introduce energy valleys that distort spectral embeddings.
\end{defn}

Two regimes are critical for KSSE:
(i)~$TS(a,0)$ form low-energy solutions (codewords) encoding class information;
(ii)~branching trapping sets $TS(a,b{\neq}0)$ drive the non-backtracking spectral radius above
unity, creating spurious attractors, while their frustration sets the phase of the divergence.
Figure~\ref{fig:combined_trapping_sets} illustrates examples of trapping sets, specifically $TS(4,44)$ and $TS(4,48)$, within spherical QC-LDPC graph families that cause distortions in spectral embeddings. Section~\ref{sec:landscape} establishes that both effects originate from the
same 2-core topology: branching sets the instability magnitude, frustration its phase.

\begin{figure}[htbp]
    \centering
    % Left image container (takes up 48% of the text width)
    \begin{minipage}{0.48\textwidth}
        \centering
        \includegraphics[width=\linewidth]{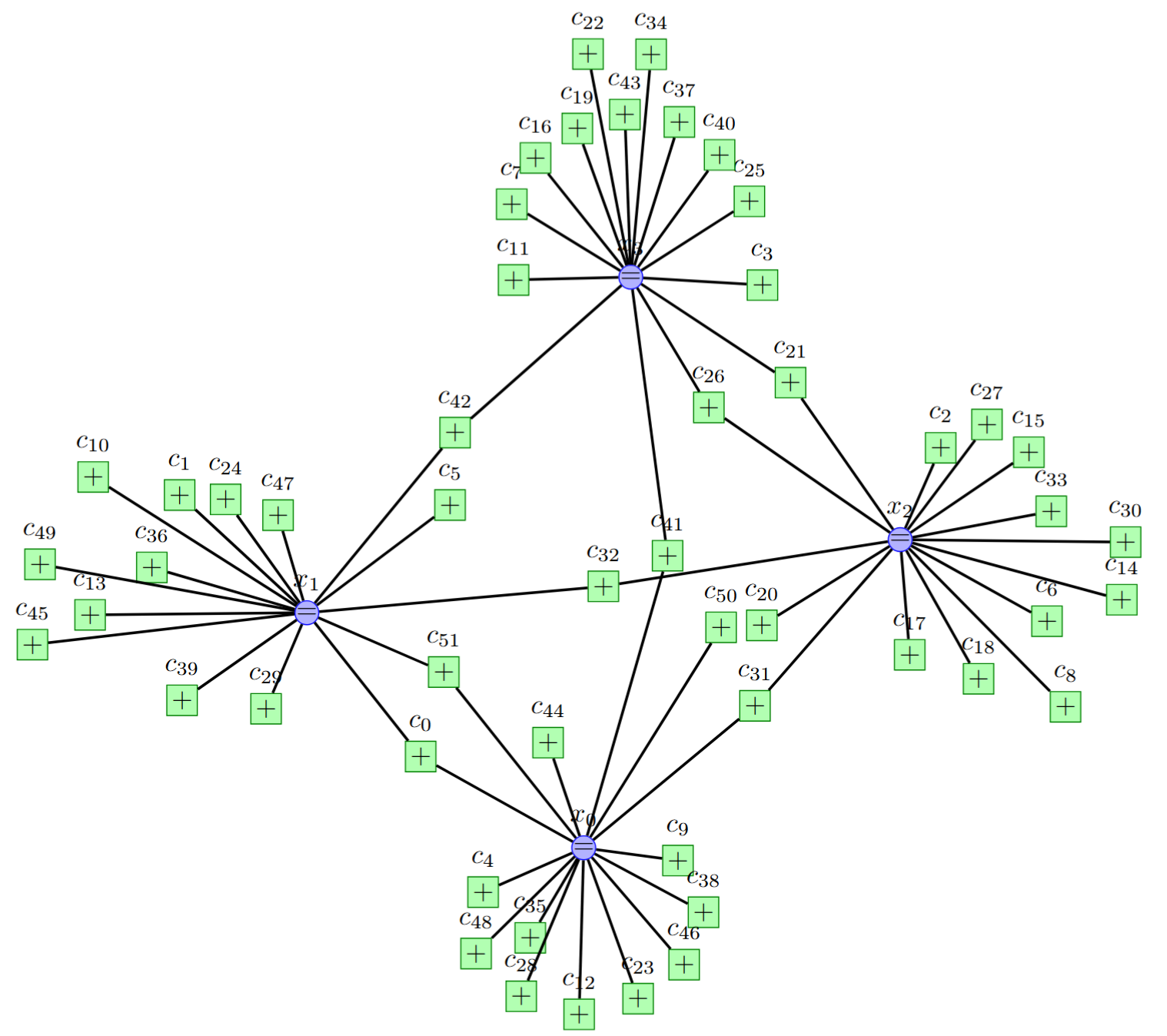}
    \end{minipage}
    \hfill % Pushes the two images to the outer edges, leaving a small gap
    % Right image container (takes up 48% of the text width)
    \begin{minipage}{0.48\textwidth}
        \centering
        \includegraphics[width=\linewidth]{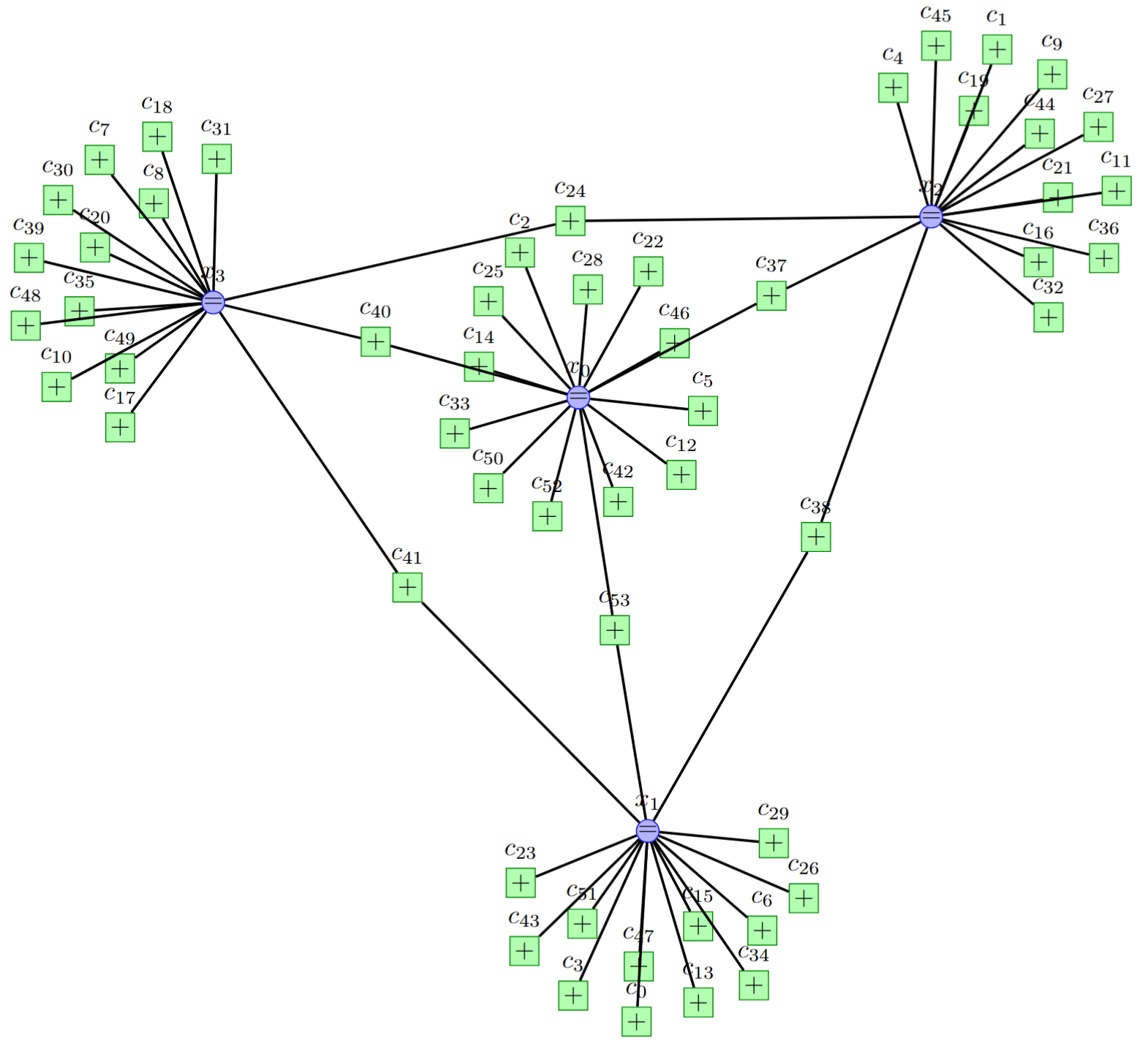}
    \end{minipage}
    % Main overarching caption for the combined figures
    \vspace{10pt} % Adds slight padding above the main caption
    \caption{Examples of Trapping Sets $TS(4,44)$ (left) and $TS(4,48)$ (right) in spherical QC-LDPC graph families.}
    \label{fig:combined_trapping_sets}
\end{figure}

%==============================================
\section{\label{sec:rbim}Random-Bond Ising Model and Nishimori Temperature}
%==============================================

Having defined the graph and its bond weights, we now introduce the energy function governing
inference and identify the critical temperature at which class structure becomes maximally detectable.

\subsection{RBIM Hamiltonian}

Each vertex $i$ carries an Ising spin $\sigma_i\in\{-1,+1\}$. The RBIM Hamiltonian on channel~$k$
is
\begin{equation}\label{eq:hamiltonian}
\mathcal{H}_J(\bm{\sigma})=-\sum_{(i,j)\in E} J_{ij}\,\sigma_i\,\sigma_j,
\end{equation}
with Boltzmann distribution $\mu_{\beta,J}(\bm{\sigma})\propto e^{-\beta\mathcal{H}_J}$ at inverse
temperature $\beta$. The coupling $J_{ij}=J_{ij}^{(k)}$ encodes feature similarity between vertices
$i$ and~$j$ on channel~$k$.

\subsection{Bethe--Hessian and Regularized Laplacian}

Linearizing belief propagation~(BP) at the paramagnetic fixed point yields a relationship between
message-passing stability and the spectrum of a graph operator. The effective interaction weights are
\begin{equation}\label{eq:weights}
W_{ij}=\frac{\tanh(\beta J_{ij})}{1-\tanh^2(\beta J_{ij})}
       = \tfrac12\sinh(2\beta J_{ij}), \qquad
\Lambda_{ij}=\frac{\tanh^2(\beta J_{ij})}{1-\tanh^2(\beta J_{ij})}
            = \sinh^2(\beta J_{ij}),
\end{equation}
and the diagonal degree matrix is $D_{ii}=1+\sum_{j\in\partial i}\Lambda_{ij}$.

With scaling matrix
$S=\mathrm{diag}(D_{11}^{-1/2},\ldots,D_{NN}^{-1/2})$, the normalized regularized Laplacian is
\begin{equation}\label{eq:laplacian}
L_\beta = I - SWS,
\end{equation}
where $W=[W_{ij}]$ is the off-diagonal weight matrix. The term $\Lambda_{ij}$ down-weights
volatile nodes, and $S$ normalizes for heterogeneous degrees.

\subsection{Nishimori Temperature}

The Nishimori temperature~$\beta_N$ marks the point where community structure becomes most
pronounced and class separability is maximized~\cite{Nishimori1981,Nishimori1980,DallAmico2021,Saade2014}. It is defined here as the \emph{first}
zero-crossing of $\lambda_{\min}(L_\beta)$ as $\beta$ increases from the paramagnetic side:
\begin{equation}\label{eq:nishimori_condition}
\lambda_{\min}(L_\beta)\big|_{\beta=\beta_N}=0.
\end{equation}
By the generalized Ihara--Bass identity (Theorem~\ref{thm:bp_laplacian}) and the sharp threshold
lemma (Lemma~\ref{lem:threshold}), this coincides with the BP detectability threshold
$1\in\operatorname{spec}(B_\gamma)$: below $\beta_N$, thermal fluctuations erase class information
(paramagnetic regime); above it, disorder dominates (spin-glass phase). We emphasize that $\beta_N$
is the spectral (Bethe) detectability threshold in the sense of~\cite{DallAmico2021,Saade2014},
not the thermodynamic Nishimori line of the RBIM phase diagram (Remark~\ref{rmk:terminology}); all
statements in this paper concern the former. The minimum eigenvector of $L_{\beta_N}$ therefore
identifies the community structure of that channel's affinity graph.

%==============================================
\section{\label{sec:kohn_sham}Kohn--Sham Mean-Field Spectral Decomposition}
%==============================================

In density functional theory, the Kohn--Sham formalism~\cite{Kohn1965} (Nobel Prize in Chemistry,
1998) replaces $N$ interacting electrons with $N$ non-interacting single-particle equations that
reproduce the same ground-state density. The effective potential decomposes as an external field plus
a Hartree (mean-field) term plus an exchange-correlation correction capturing interactions beyond
the mean-field approximation.

KSSE performs an analogous reduction: instead of solving one coupled problem over all $D$ feature
channels, we solve $D$ independent single-channel eigenproblems---each with its own effective
Hamiltonian $L_\beta^{(k)}(\beta_N^{(k)})$---and concatenate the results. The regularized Laplacian
$L_\beta^{(k)}$ plays the role of the single-particle Hamiltonian, the diagonal scaling
$s_i^{(k)}=1/\sqrt{D_{ii}^{(k)}}$ encodes the external (graph scaffold) potential, and the
off-diagonal coupling $W_{ij}^{(k)}s_j^{(k)}$ represents the Hartree mean-field interaction. The
exchange-correlation energy $E_{xc}$ captures loop corrections to the Bethe free energy that vanish on
trees.

\begin{thm}[Mean-Field Kohn--Sham Spectral Embedding]\label{thm:ksse_decomp}
Under the Bethe--Peierls (mean-field) approximation and feature independence, the $D$-channel
interacting problem decomposes into $D$ independent single-channel eigenproblems:
\begin{equation}
L^{(k)}(\beta_N^{(k)})\,\bm{v}_i^{(k)}=\lambda_i^{(k)}\,\bm{v}_i^{(k)},\qquad k=1,\dots,D,
\end{equation}
and the final embedding is
$\bm{E}=[S^{(1)}\bm{v}_{\min}^{(1)}|\cdots|S^{(D)}\bm{v}_{\min}^{(D)}]\in\mathbb{R}^{N\times D}$.
The effective potential for each channel decomposes as
$v_{\text{eff}}^{(k)}(i)=s_i^{(k)}+\sum_{j\in\partial i}W_{ij}^{(k)}s_j^{(k)}
+\delta E_{xc}^{(k)}/\delta\rho_i^{(k)}$,
where $s_i^{(k)}=1/\sqrt{D_{ii}^{(k)}}$ is the external potential, the sum is the Hartree
interaction, and $E_{xc}^{(k)}=F_{\text{Bethe}}^{(k)}-F_{\text{MF}}^{(k)}$ is the exchange-correlation
energy. Under the loop-series representation and the spectral condition of
Proposition~\ref{prop:bethe_xc}, $E_{xc}^{(k)}$ obeys the per-vertex bound given there; it vanishes
on trees and, after star-domain surgery convergence, $E_{xc}=\mathcal{O}(\delta)$ with $\delta\to 0$
(Theorem~\ref{thm:self_consistency}).
\end{thm}

\begin{proof}
See Appendix~\ref{app:exc}.
\end{proof}

\subsection{Additive Separability under Feature Independence}

Because each feature channel carries its own coupling tensor $J_{ij}^{(k)}$ and there are no
cross-feature interaction terms, the full Hamiltonian is additively separable:
$\mathcal{H}_{\text{full}}=\sum_{k=1}^D \mathcal{H}^{(k)}$. This exact decomposition---which
provably fails on higher-dimensional topological complexes (Proposition~\ref{prop:cup},
Remark~\ref{rmk:hierarchy})---is the key structural property enabling per-channel computation.

\begin{thm}[Additive Separability]\label{thm:ks_separability}
Under feature independence, $Z_{\text{full}}=\prod_k Z^{(k)}$,
$G_{\text{Bethe}}=\sum_k G_{\text{Bethe}}^{(k)}$, and
$L_\beta^{\text{full}}=\bigoplus_k L_\beta^{(k)}$. Defining each channel's critical temperature as
the first zero-crossing from the paramagnetic side (Lemma~\ref{lem:threshold}), the joint critical
temperature is $\beta_N^{\text{full}}=\min_k\beta_N^{(k)}$.
\end{thm}

\begin{proof}
See Appendix~\ref{app:sep}.
\end{proof}

\subsection{Exchange-Correlation Energy and the Loop-Series Bound}

On trees ($g_0=\infty$), the Bethe free energy is exact, so $E_{xc}=F_{\text{Bethe}}-F_{\text{MF}}=0$
and the mean-field decomposition becomes exact. On loopy graphs, loop corrections appear first at
length~$g_0$, giving an exponentially decaying bound under a Dobrushin-type spectral condition:

\begin{prop}[Exchange-Correlation Bound]\label{prop:bethe_xc}
Let $E_{xc}^{(k)}:=F_{\mathrm{exact}}^{(k)}-F_{\mathrm{Bethe}}^{(k)}$ per channel, assume the
loop-series representation $F_{\mathrm{exact}}-F_{\mathrm{Bethe}}=-\log\Omega$ of the free-energy
correction~\cite{ChertkovChernyak2006}, and the spectral convergence condition
$\xi:=\rho(B_{|t|})<1$, where $|t_e|=|\tanh(\beta J_e)|$. Let $g_0$ be the girth and
$\bar d=2M/N$ the mean degree. Then, whenever the loop sum satisfies $|\Omega-1|\leq\tfrac12$,
\begin{equation}
\frac{|E_{xc}^{(k)}|}{N}\;\leq\;\frac{2\bar d\,\xi^{g_0}}{g_0\,(1-\xi)} .
\end{equation}
The bound vanishes on trees ($g_0=\infty$, $\xi=0$); for girth $g_0\geq 6$ and $\xi\leq 1/4$
(the clipped-coupling regime) the per-vertex error is ${\approx}0.5\%$ or less, falling below
$0.03\%$ at $g_0\geq 8$ (Corollary~\ref{cor:girth6}). After star-domain surgery, residual
frustration yields $E_{xc}=\mathcal{O}(\delta)$ with $\delta\to 0$
(Theorem~\ref{thm:self_consistency}). The fractal dimension $D_2$ serves as a practical diagnostic:
$D_2<1$ indicates star-domain basins (near-exact decomposition), while $D_2>3$ signals a rough
landscape requiring surgery.
\end{prop}

\begin{proof}
See Appendix~\ref{app:exc}.
\end{proof}

\subsection{Absence of a Hohenberg--Kohn Analogue}

In Kohn--Sham DFT, the Hohenberg--Kohn theorem guarantees a one-to-one correspondence between
ground-state density and external potential. No such existence theorem is needed in KSSE because
feature channels are exactly independent by construction.

\begin{prop}[Absence of a Hohenberg--Kohn Analogue for RBIM]\label{prop:no_hk}
There exist graphs with distinct coupling matrices $J$ and $J'$ that yield identical magnetizations
$m_i=\langle\sigma_i\rangle$ at all temperatures but different free energies, demonstrating that the
free energy is not a function of $\{m_i\}$ alone.
\end{prop}

\begin{proof}
See Appendix~\ref{app:no_hk}.
\end{proof}

This proposition confirms that KSSE's exact separability arises from feature independence (a structural
property), not from an existence theorem. The Kohn--Sham analogy is therefore a computational strategy,
not a variational approximation based on a density functional; the correct order parameters of the
RBIM are gauge-invariant loop (flux) quantities, not local magnetizations (Appendix~\ref{app:no_hk}).

\subsection{Pontryagin Duality and FFT-Based Eigenvalue Computation}

Every finite cyclic group $G=\mathbb{Z}/N\mathbb{Z}$ is self-dual under Pontryagin
duality~\cite{Pontryagin1934}: $\widehat{G}\cong G$ with characters $\chi_m(n)=e^{2\pi imn/N}$, and
every circulant matrix---a convolution operator on $G$---is diagonalized by the character basis,
i.e., by the discrete Fourier transform (Lemma~\ref{lem:circulant}); composite block sizes are
included. Two consequences follow. First, each single-channel eigenproblem on an exact circulant
block is solved in $\mathcal{O}(N\log N)$ via FFT. Second, after surgery the Laplacian is a sparse
perturbation of the pre-surgery circulant operator with an open fluctuation gap
(Lemma~\ref{lem:ab}), so Rayleigh--Ritz refinement on the few modes nearest the pre-surgery minimum
is quadratically accurate (Lemma~\ref{lem:rayleigh}); empirically $k_{\text{mode}}=5$ Fourier modes
suffice (Remark~\ref{rmk:kmode}). A worked special case---the $N=45{,}000$ spherical graph of the experiments---develops this picture constructively as $k\cdot p$ effective-mass theory on a ring crystal: the data weights act as a slowly
varying impurity potential whose smooth part couples the band extremum only to its few nearest Bloch modes---precisely the index neighbourhood $\mathcal{I}$ of Algorithm~\ref{alg:fft_eig}---and the Nishimori crossing
appears as a Fermi level tuned to the band edge
(Appendix~\ref{app:kp}).

\subsection{Topological Coupling Hierarchy}

\begin{rmk}[Hierarchy of topological complexity]\label{rmk:hierarchy}
The exact additive separability in Theorem~\ref{thm:ks_separability} relies on the interaction
structure being a $1$-dimensional cell complex (a graph): $H^2(\mathcal{G};\mathbb{Z}_2)=0$, so no
$2$-cell coupling exists. Replacing it with a higher-dimensional topological complex~\cite{Zhu2025,
Freedman2020} would introduce cross-channel coupling via cup products
(Proposition~\ref{prop:cup}):
\begin{enumerate}
  \item \emph{Level~1 (Circulant QC-LDPC, current KSSE):} No cross-channel interactions;
        $E_{xc}=\mathcal{O}(\delta)\to 0$; $\mathcal{O}(N\log N)$ per feature via FFT.
  \item \emph{Level~2 (CSS code on 2D complex, pairwise intersections):}
        Cup products $\smile: H^1\times H^1\to H^2$ yield $V^{(k,k')}\neq 0$; approximate Kohn--Sham
        with two-body exchange-correlation. FFT factorization breaks.
  \item \emph{Level~3 (3D complex, triple intersections):}
        Three-body coupling analogous to CCZ-type non-Clifford gates~\cite{Zhu2025}; fully coupled
        Hamiltonian requires quantum resources.
\end{enumerate}
Star-domain surgery keeps the system at Level~1: the graph remains a $1$-complex and $H^2=0$ is
never populated. Each level increases expressiveness but sacrifices computational tractability.
\end{rmk}

%==============================================
\section{\label{sec:landscape}Energy Landscape: BP Defects, Trapping Sets,
and Fractal Analysis}
%==============================================

This section establishes the key physical connection between belief-propagation free-energy defects
and RBIM energy-landscape defects at the Nishimori temperature---branching cycles drive the
instability while frustration enters as a $\mathbb{Z}_2$ interference phase---and introduces
multi-scale fractal analysis as a diagnostic for surgical effectiveness.

\subsection{Belief Propagation Free Energy on Loopy Graphs}

On a tree (cycle-free graph), BP converges to exact marginals and the Bethe free energy is convex.
On loopy graphs such as QC-LDPC Tanner graphs, BP iterations define a dynamical system
$\bm{m}^{(t+1)}=\mathcal{F}(\bm{m}^{(t)})$ whose fixed points are stationary points of $G_{\text{Bethe}}$,
but convexity is lost: the landscape acquires spurious local minima and saddle points associated with
trapping sets. Linearizing BP at the paramagnetic fixed point yields the weighted non-backtracking
operator $B_\gamma$ with weights $\gamma_{ij}=\tanh(\beta J_{ij})$. The connection to the regularized
Laplacian is formalized in Theorem~\ref{thm:bp_laplacian}: linearization at $\beta_N$, followed by a
generalized Ihara--Bass projection from edge space to vertex space (with the determinant identity of
part~(c) proved by an incidence factorization of $B_\gamma$) and congruence via $S$, produces
$L_\beta=I-SWS$ exactly.

The key identity (Theorem~\ref{thm:bp_laplacian}(d)) is:
\begin{equation}
H_{\text{Bethe}}(\beta)\;\xrightarrow{\text{linearize at }\beta_N}\; L_\beta = I - SWS.
\end{equation}
The curvature of the BP free-energy landscape at $\beta_N$ \emph{is} the regularized Laplacian whose
spectrum we compute. At the critical point $\lambda_{\min}(L(\beta_N))=0$, community structure becomes
maximally detectable.

\begin{thm}[BP Linearization and Generalized Ihara--Bass Identity]\label{thm:bp_laplacian}
Let $B_\gamma$ be the weighted non-backtracking operator on directed edges,
$(B_\gamma \bm{u})_{(i\to j)}=\sum_{k\in\partial i\setminus j}\gamma_{ki}\,u_{(k\to i)}$ with
$\gamma_{ij}=\tanh(\beta J_{ij})$, and for $\mu\in\mathbb{C}$ with $\mu^2\neq\gamma_{ij}^2$ define
the vertex matrix $[M_\gamma(\mu)]_{ii}=1+\sum_{k\in\partial i}\frac{\gamma_{ik}^2}{\mu^2-\gamma_{ik}^2}$,
$[M_\gamma(\mu)]_{ij}=-\frac{\gamma_{ij}\,\mu}{\mu^2-\gamma_{ij}^2}$. Then:
(a)~$B_\gamma$ is the Jacobian of belief propagation linearized at the paramagnetic fixed point;
(b)~$\mu$ is an eigenvalue of $B_\gamma$ if and only if $\det M_\gamma(\mu)=0$, and the
edge-eigenspace and the vertex kernel have equal dimension;
(c)~the determinant identity
$\det M_\gamma(\mu)=\det\!\bigl(I_{2E}-\mu^{-1}B_\gamma\bigr)\big/\prod_{e\in E}(1-\gamma_e^2/\mu^2)$
holds;
(d)~at $\mu=1$, $M_\gamma(1)=D-W_A$ and $L_\beta=S\,M_\gamma(1)\,S$.
\end{thm}

\begin{proof}
See Appendix~\ref{app:thm_bp}.
\end{proof}

Together with the sharp threshold lemma (Lemma~\ref{lem:threshold}), this yields the spectral
trichotomy at the level of real eigenvalues: $M_\gamma(1)\succ0$ below the first real crossing
(BP-stable paramagnetic phase), singular at the crossing ($\lambda_{\min}(L_{\beta_N})=0$, maximal
detectability), and indefinite above (spin-glass regime). For frustrated (unbalanced) signings the
leading eigenvalue of $B_\gamma$ may be complex; BP then diverges \emph{oscillatorily} when a complex
pair crosses $\rho(B_\gamma)=1$, while indefiniteness of $L_\beta$ is governed by real crossings only
(Remark~\ref{rmk:terminology}). This is the spectral distinction between ferromagnetic ordering and
spin-glass ordering.

\subsection{Shared Topological Origin of Defects}

Trapping sets $TS(a,b{\neq}0)$ simultaneously corrupt both landscapes---the BP free-energy landscape
and the RBIM energy landscape at $\beta_N$---through their 2-core topology, \cite{Chernyak2005}:
branching controls the instability magnitude, frustration its phase.

\begin{thm}[Trapping-Set Spectral Test]\label{thm:trapping_set}
Let $U\subseteq V$ induce a trapping set $TS(a,b)$ with signed affinities $\gamma_{ij}$ on the
variable subgraph.
(a)~\emph{Branching instability.} If the 2-core of the variable subgraph branches (a vertex of core
degree $\geq 3$), there exists $\beta_\ast<\infty$ such that $\rho(B_{|\gamma|}^{(U)})>1$ for all
$\beta>\beta_\ast$; if the signing is balanced on the core, the crossing is real and
$\lambda_{\min}(L_\beta^{(U)})$ turns negative immediately above it, while BP messages diverge
exponentially.
(b)~\emph{Frustrated phase.} If the signing is frustrated on core cycles, the leading eigenvalues
acquire phases set by the $\mathbb{Z}_2$ flux (Aharonov--Bohm analogy, Lemma~\ref{lem:ab}); BP
divergence is then oscillatory, while indefiniteness of $L_\beta^{(U)}$ occurs only at real
crossings.
(c)~\emph{Codeword certificate.} $TS(a,0)$ supports are exactly the cycle-space elements, i.e., the
constructive terms of the partition-function expansion (Proposition~\ref{prop:even_subgraph}); a
codeword whose core is a union of simple cycles has $\rho(B_\gamma^{(U)})<1$ at every finite $\beta$
(Lemma~\ref{lem:trichotomy}) and is benign for BP; branching codewords are certified directly by the
sign of $\lambda_{\min}(L_\beta^{(U)})$.
\end{thm}

\begin{proof}
See Appendix~\ref{app:thm_ts}.
\end{proof}

\noindent\textbf{Codewords as constructive interference.}
The high-temperature expansion makes the role of codewords exact
(Proposition~\ref{prop:even_subgraph}): $Z=2^N\prod_{e}\cosh(\beta J_e)\cdot\Omega$ with
$\Omega=\sum_{x\in\mathcal{C}(\mathcal{G})}\prod_{e}t_e^{x_e}$ over the cycle space
$\mathcal{C}(\mathcal{G})$, whose supports are precisely the $TS(a,0)$ configurations. The resulting
tractability hierarchy (Corollary~\ref{cor:hierarchy}) mirrors the computational one: trees factorize
exactly ($E_{xc}=0$); a single cycle reduces $\Omega$ to two matching sectors
$1\pm\prod_{e\in C}|t_e|$, interfering constructively when balanced and destructively when
frustrated; planar and fixed-genus (toroidal QC-LDPC) graphs reduce $Z$ to Pfaffians, one per flux
sector; general graphs require a \#P-hard permanent-like superposition of $2^{b_1}$ cycle-space
terms. Star-domain surgery suppresses the interfering sectors, keeping the system in the
determinantal regime.

Eliminating all trapping sets is impossible---doing so would destroy codewords $TS(a,0)$ encoding
essential class information. Instead, star-domain surgery (Sec.~\ref{sec:star_domain_thm}) constructs
shifts that bound residual frustration to $\rho(B_\gamma)\leq 1+\delta$ while creating certified local
convexity around surviving codewords.

Figure~\ref{fig:TS_valley} visualizes this dichotomy using a 2-D PCA projection of the trapping-set
energy landscape (computed via Alg.~\ref{alg:surgery}). The surface shows regions of positive Bethe--Hessian
eigenvalues where energy minima correspond to codewords $TS(a,0)$; surrounding valleys associated with
branching trapping sets $TS(a,b{\neq}0)$ introduce negative eigenvalues that create spurious attractors,
with frustration setting the divergence phase.

\begin{figure}[t]
    \centering
    \includegraphics[width=\linewidth]{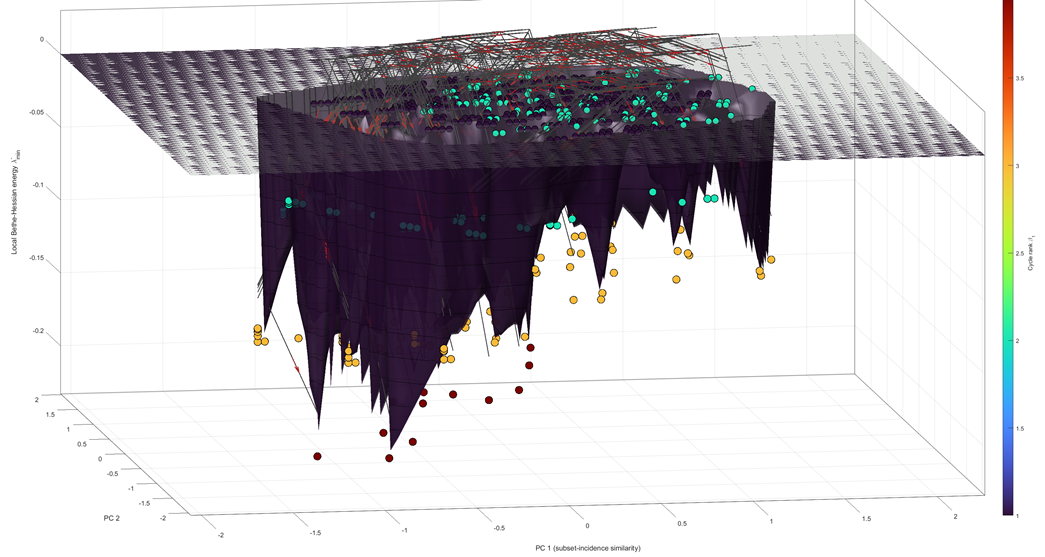}
    \caption{2-D PCA visualization of the trapping-set energy landscape. The surface encodes the
    local Bethe--Hessian eigenvalue sign: positive-eigenvalue regions (above the plane) correspond to
    energy minima at codewords $TS(a,0)$ that encode class information; surrounding valleys with
    negative eigenvalues (below the plane) arise from branching trapping-set cores $TS(a,b{\neq}0)$
    and create spurious attractors, with frustration setting the divergence phase. Star-domain
    surgery constructs edge shifts that widen the positive-eigenvalue basins around codewords while
    bounding residual frustration from surviving $TS(a,b{\neq}0)$.
    The embedding is computed via Alg.~\ref{alg:surgery}.}
    \label{fig:TS_valley}
\end{figure}

\subsection{\label{sec:fractal}Multi-Scale Fractal Analysis of Energy Landscapes}

The correlation dimension~\cite{Grassberger1983}
$C(r)\propto r^{D_2}$ quantifies fractal clustering of deep valleys in the Bethe--Hessian energy
landscape. The generalized dimension spectrum $D_q$ ($q\in[-4,4]$)~\cite{Grassberger2}, shown in
Fig.~\ref{fig:Gen_d_spectrum}, confirms strictly multifractal scaling: high-dimensional sparse
fluctuations at $D_{-4}\approx 4.4$ transition to concentrated peaks at $D_4\approx 0$, indicating a
heterogeneous energy landscape with rare deep traps.

\begin{figure}[t]
    \centering
    \includegraphics[width=0.7\linewidth]{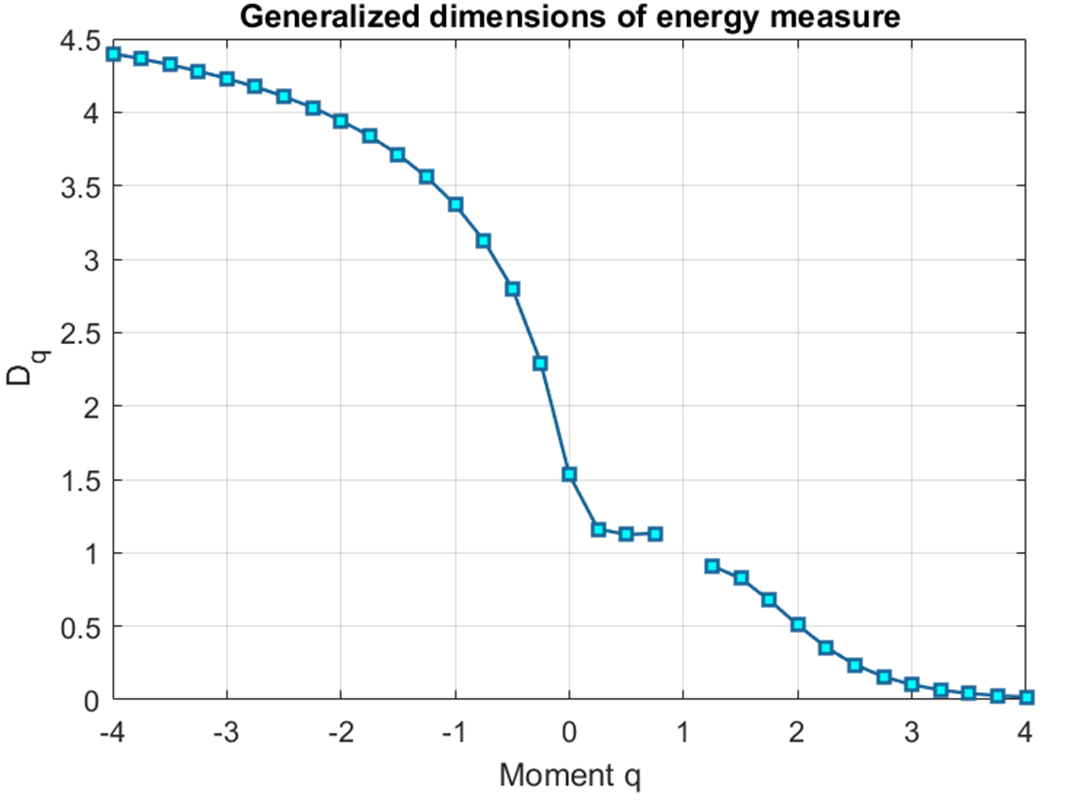}
    \caption{Generalized dimension spectrum $D_q$ ($q\in[-4,4]$) of the trapping-set energy landscape
    before star-domain surgery. The broad range from $D_{-4}\approx 4.4$ to $D_4\approx 0$ confirms
    multifractal structure: high-dimensional sparse fluctuations coexist with concentrated deep traps,
    indicating a rough landscape ($D_2>3$) that requires surgical intervention. After surgery, the
    spectrum collapses toward $D_q<1$ across all~$q$, certifying star-domain geometry.}
    \label{fig:Gen_d_spectrum}
\end{figure}

The fractal analysis serves three critical roles:

\textbf{(i) Surgical guidance.} Before surgery, the energy landscape is rough and multifractal
($D_2>3$): many frustrated saddle points create spurious local minima that distort spectral embeddings.
Surgery constructs shifts creating star-domain convexity around codewords $TS(a,0)$; this transition
is diagnosed by $D_2$ decreasing from $D_2>3$ to $D_2<1$, indicating formation of wide, smooth basins
(Fig.~\ref{fig:wide_basin})~\cite{Baldassia_basin}.

\begin{figure}[t]
    \centering
    \includegraphics[width=0.75\linewidth]{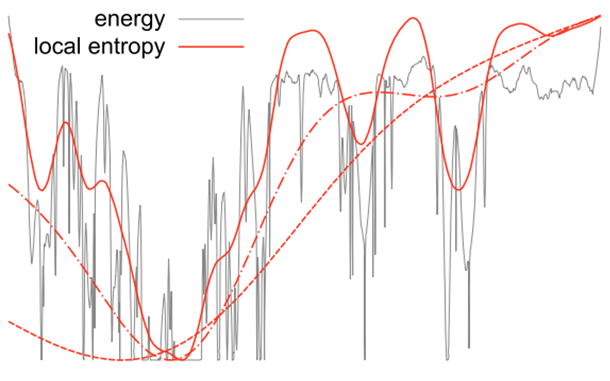}
    \caption{Local entropy landscape around codewords $TS(a,0)$ after star-domain surgery.
    Wide, smooth attractor basins replace the pre-surgery rough terrain, enabling iterative methods
    (belief propagation or gradient descent) to converge to the correct codeword from a broad initial
    region. The basin geometry is certified per codeword with radius $R=g/(2L_3)$
    (Theorem~\ref{thm:star_domain}); $D_2<1$ is its fractal signature~\cite{Baldassia_basin}.}
    \label{fig:wide_basin}
\end{figure}

\textbf{(ii) Low-mode Rayleigh--Ritz refinement.} After surgery the Laplacian is a sparse perturbation
of the pre-surgery circulant operator, $L=L_c+\Delta$ with $\|\Delta\|$ small, and the removal of
frustrated fluxes opens a fluctuation gap $g$ (Lemma~\ref{lem:ab}). The exact ground eigenvector is
therefore the pre-surgery Fourier mode plus an $\mathcal{O}(\|\Delta\|/g)$ admixture, and
Rayleigh--Ritz on the trial subspace spanned by the few modes nearest the pre-surgery minimum is
quadratically accurate, $|\tilde\lambda_{\min}-\lambda_{\min}|\leq\Lambda\|\Delta\|^2/g^2$
(Lemma~\ref{lem:rayleigh}). In all experiments $k_{\text{mode}}=5$ modes achieve residuals below
$10^{-6}$ (Remark~\ref{rmk:kmode}), reducing the per-channel refinement to
$\mathcal{O}(k_{\text{mode}}N)=\mathcal{O}(N)$. The heuristic uncertainty picture (basin width $w$
$\leftrightarrow$ spectral peak width ${\sim}1/w$ on the circulant scaffold,
Remark~\ref{rmk:basin_heuristic}) motivates this choice; the bound of Lemma~\ref{lem:rayleigh} justifies it. The constructive crystal mechanism underlying this Fourier picture---band structure of the circulant scaffold, impurity potential from the data weights, and dressing of the band-edge Bloch mode by the class-profile
envelope---is worked out for the $N=45{,}000$ spherical graph in Appendix~\ref{app:kp}.

\textbf{(iii) Convergence certification.} A post-surgery value $D_2<1$ is the fractal signature of the
star-domain geometry certified per codeword by Theorem~\ref{thm:star_domain} (radius $R=g/(2L_3)$):
attractor basins around codewords are wide enough that BP messages converge to correct community
assignments from a broad initial region.

\begin{figure}[t]
    \centering
    \includegraphics[width=0.9\linewidth]{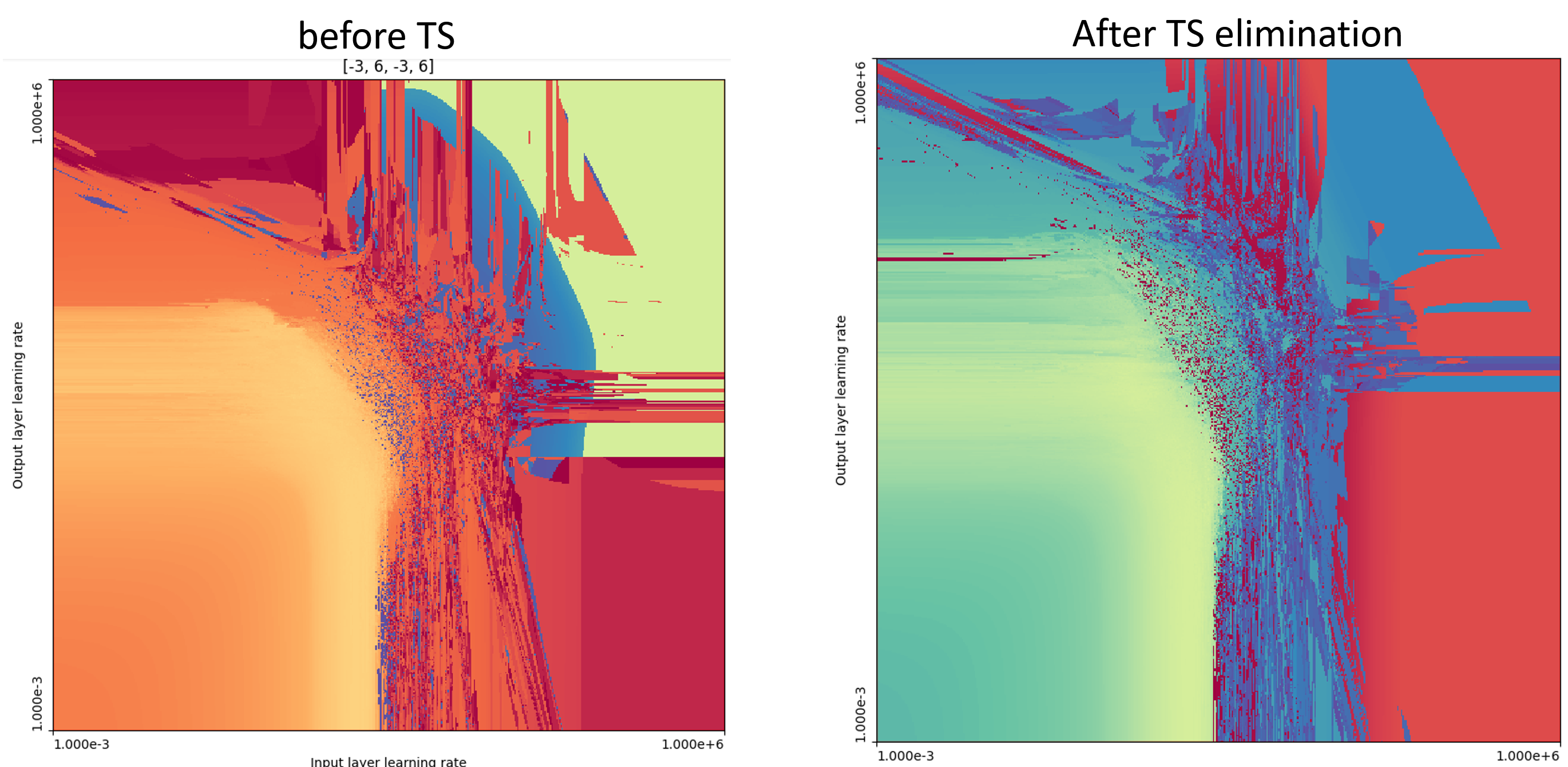}
    \caption{\textbf{Eliminating one family (ontology - ordered hierarchical sets of nested nodes which form TS, \cite{Vasic2009}) of trapping sets regularizes the fractal learning-rate
    landscape.} Two-dimensional scan over input-layer ($\eta_0$) and output-layer ($\eta_1$) learning
    rates \textbf{(left)} before and \textbf{(right)} after breaking a single family of trapping sets.
    Color encodes local loss curvature (proxy for effective Hessian eigenvalue magnitude along
    hyperparameter directions). Before removal, the landscape displays dense fractal filamentation:
    tiny perturbations in $\eta_0$ or $\eta_1$ trigger bifurcations between convergent and divergent
    regimes, reflecting extreme curvature. After TS elimination, broad, smoothly varying basins replace
    much of the fine-scale structure, indicating that the hyperparameter-space curvature has been
    substantially reduced. This visualization connects the fractal dimension $D_2$ of the energy
    landscape (Sec.~\ref{sec:fractal}) to practical trainability: removing a single TS family lowers
    the effective Hessian eigenvalues along $\eta_0,\eta_1$, stabilizing training over wider
    hyperparameter ranges~\cite{sohl2024boundary}.}
    \label{fig:TS_broke_and_curvature}
\end{figure}

\textbf{Trapping-set removal lowers fractal curvature.}
The boundary between stable and unstable training in neural-network hyperparameter space is known to
be fractal~\cite{sohl2024boundary}.
In Fig.~\ref{fig:TS_broke_and_curvature} we visualize how this landscape changes when one family of
trapping sets (TS) is removed.
Prior to elimination, the left panel exhibits intricate vertical and horizontal striations spanning many
orders of magnitude in learning rate.
These fine-scale features correspond to bifurcation boundaries where the loss surface is extremely
sensitive to hyperparameters;
mathematically, this manifests as large Hessian eigenvalues with respect to $\eta_{0}$ and $\eta_{1}$.
Trapping sets anchor these instabilities by creating localized high-curvature ridges in the optimization
landscape.
Once this TS family is broken---via graph pruning that severs the associated cyclic dependency paths---the
pinned bifurcations disappear across large regions of the scan.
The post-elimination panel reveals extensive smooth basins with gentle color gradients, signifying that
the local curvature along learning-rate directions has dropped and the landscape has become far more
amenable to optimization.
Consequently, training can tolerate larger stable learning rates and becomes less sensitive to small
hyperparameter perturbations.

%==============================================
\section{\label{sec:star_domain_thm}Star-Domain Surgery and Local Convexity}
%==============================================

A fundamental constraint prevents eliminating all frustrated cycles $TS(a,b{\neq}0)$: the minimum
code distance $d_{\min}$ of any non-trivial LDPC code is bounded by the Tanner graph structure, and
removing all odd-degree check nodes would destroy codewords $TS(a,0)$ that encode essential class
information. Instead, surgery constructs \emph{shifts}---edge modifications that preserve codewords
while creating local convexity around them.

\begin{defn}[Star Domain]\label{def:star_domain}
A set $\mathcal{S}\subseteq\mathbb{R}^N$ is a star domain with respect to $\bm{x}_0\in\mathcal{S}$ if
for all $\bm{y}\in\mathcal{S}$ and $t\in[0,1]$, $t\bm{y}+(1-t)\bm{x}_0\in\mathcal{S}$. In the energy
landscape context, $G_{\text{Bethe}}$ forms a star domain around codeword $\bm{x}_0\in TS(a,0)$ if for
all $\bm{\sigma}\in\mathcal{U}(\bm{x}_0)$ and all $t\in[0,1]$:
$G_{\text{Bethe}}(t\bm{\sigma}+(1-t)\bm{x}_0)\leq\max(G_{\text{Bethe}}(\bm{\sigma}),
G_{\text{Bethe}}(\bm{x}_0))$. This ensures gradient descent (or BP message passing) reaches
$\bm{x}_0$ from any starting point in the basin.
\end{defn}

The star-domain property is weaker than full convexity but sufficient for convergence of iterative
methods to the codeword minimum.

\begin{thm}[Star-Domain Certificate]\label{thm:star_domain}
Let $\bm{x}_0$ be a codeword ($TS(a,0)$) and suppose:
(H1)~the Bethe Hessian at $\bm{x}_0$ satisfies $M_\gamma(1)\succeq g\,I$ on the tangent space,
$g>0$ (after surgery, certified numerically per codeword);
(H2)~$G_{\mathrm{Bethe}}$ is $C^2$ on the relaxed spin space $[-1,1]^N$ with curvature variation
bounded by $L_3<\infty$.
Then on the ball $\mathcal{U}_R(\bm{x}_0)$ of radius $R=g/(2L_3)$, the function
$\varphi(t)=G_{\mathrm{Bethe}}(\bm{x}_0+th)$ is convex on $[0,1]$ for every $\|h\|\leq R$;
consequently $G_{\mathrm{Bethe}}$ forms a star domain on $\mathcal{U}_R(\bm{x}_0)$
(Definition~\ref{def:star_domain}), and gradient flow (or BP damped with step $1/L_2$, $L_2$ the
local smoothness) converges to $\bm{x}_0$ at the linear rate $(1-g/L_2)^t$ from any starting point
in the ball. Surgery preserves the certificates of surviving codewords and reduces the residual
frustration of the remaining trapping sets to $\rho(B_\gamma^{(U)})\leq 1+\delta$, $\delta\to 0$
(Theorem~\ref{thm:self_consistency}).
\end{thm}

\begin{proof}
See Appendix~\ref{app:star}.
\end{proof}

\begin{figure}[t]
    \centering
    \includegraphics[width=0.7\linewidth]{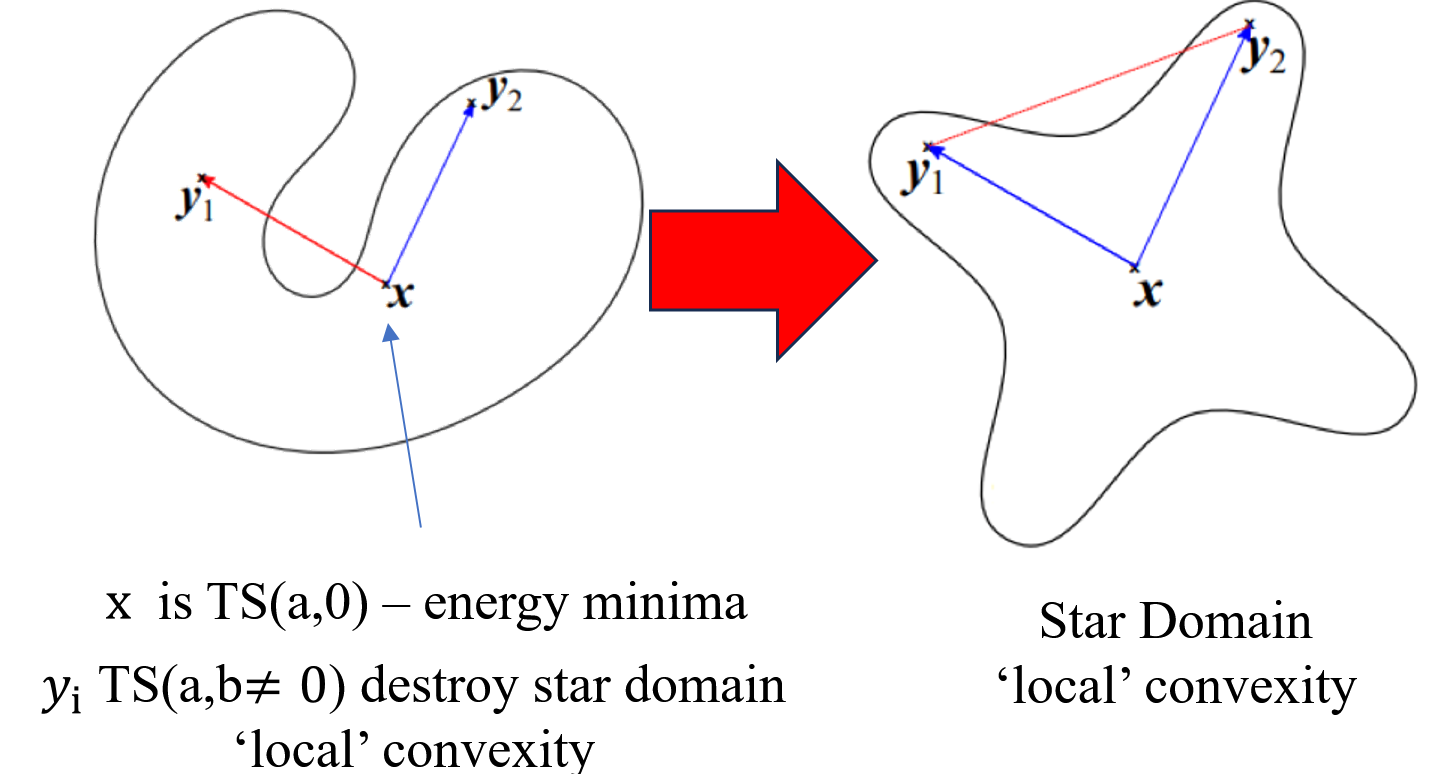}
    \caption{Star-domain geometry around codewords $TS(a,0)$ after surgery. Each codeword
    $\bm{x}_0$ becomes the center of a star-domain basin (shaded region) in the Bethe free-energy
    landscape: any point on the line segment from $\bm{x}_0$ to $\bm{\sigma}\in\mathcal{U}(\bm{x}_0)$
    lies within the domain, ensuring convergence of BP or gradient descent to the codeword minimum.
    The certified basin radius is $R=g/(2L_3)$ (Theorem~\ref{thm:star_domain}); the scaling
    $w\sim 1/\sqrt{d_{\min}}$ with the code distance is a heuristic (Remark~\ref{rmk:basin_heuristic}).
    Residual frustrated trapping sets $TS(a,b{\neq}0)$ outside the basin have bounded spectral radius
    $\rho(B_\gamma)\leq 1+\delta$ and do not create spurious attractors within the codeword's domain.}
    \label{fig:Star_domain}
\end{figure}

Star-domain surgery serves a dual purpose: it suppresses branching frustrated cycles that corrupt BP
convergence (Theorem~\ref{thm:trapping_set}) and creates certified star-domain convexity around
codewords (Fig.~\ref{fig:Star_domain}), reducing the effective first Betti number $b_1$ of the graph
and keeping the interaction structure a $1$-complex, so that no $2$-cell (cup-product) coupling
between feature channels arises (Proposition~\ref{prop:cup}, Remark~\ref{rmk:hierarchy}).

A heuristic connection to the Fourier picture completes the section: on the circulant scaffold, a
configuration-space basin of width $w$ corresponds to spectral concentration within
$\mathcal{O}(1/w)$ modes (Remark~\ref{rmk:basin_heuristic}); the rigorous refinement statement is the
Rayleigh--Ritz bound of Lemma~\ref{lem:rayleigh}.

%==============================================
\section{\label{sec:algorithm}KSSE Algorithm Pipeline}
%==============================================

The algorithm processes an input sparse graph \(G\) (defined by edge indices \(\bm{r},\bm{c}\)) alongside a feature matrix \(X\in\mathbb{R}^{N\times D}\). The overall pipeline executes across six distinct stages: affinity tensor construction, spin-glass temperature estimation, Nishimori temperature search, FFT-based eigenvalue computation with Rayleigh--Ritz refinement, fractal analysis for surgical guidance, and star-domain trapping-set surgery.

\textbf{Affinity Tensor Construction}. This initial stage is executed via Alg. \ref{alg:affinity}, which processes the specified inputs to generate the corresponding output tensor.

\begin{algorithm}[t]
\caption{Affinity Tensor Construction}\label{alg:affinity}
\Input{$X\in\mathbb{R}^{N\times D}$, edge indices $\bm{r},\bm{c}\in\mathbb{N}^E$, stability constant $\varepsilon$}
\Output{Normalized affinity tensor $A\in\mathbb{R}^{E\times D}$}
Extract edge features: $X_i \gets X[\bm{r},:]$, $X_j \gets X[\bm{c},:]$\;
Compute Manhattan distances: $L_1[i,k] \gets |X_i[i,k]-X_j[i,k]|$\;
$A_{\text{raw}}[i,k] \gets 1/(1+L_1[i,k]+\varepsilon)$\;
\For{$k=1$ \KwTo $D$}{
  $\mu_k \gets \frac{1}{E}\sum_i A_{\text{raw}}[i,k]$;\quad
  $\sigma_k \gets \sqrt{\frac{1}{E}\sum_i (A_{\text{raw}}[i,k]-\mu_k)^2}$\;
  \eIf{$\sigma_k > \varepsilon_{\text{std}}$}{
    $A[:,k] \gets (A_{\text{raw}}[:,k]-\mu_k)/\sigma_k$\;
  }{
    $A[:,k] \gets 0$\;
  }
}
\Return{$A$}\;
\end{algorithm}

\textbf{Spin-Glass Temperature Estimation}. This phase is governed by Alg.~\ref{alg:betasg}, which identifies the spin-glass temperature marking the onset of replica symmetry breaking. Utilizing the mean degree~$c$ and the normalized second moment $\phi=\frac{1}{N}\sum_i d_i^2/c^2$, the critical inverse temperature $\beta$ is determined by solving $f(\beta) = c\,\phi\,\mathbb{E}_{(i,j)\in E}[\tanh^2(\beta J_{ij})] - 1 = 0$ via the bisection method.

\begin{algorithm}[t]
\caption{Spin-Glass Temperature Estimation}\label{alg:betasg}
\Input{Affinity matrix $J$, degrees $\{d_i\}$, tolerance $\delta$, $step_{max}$ bisection iterations}
\Output{$\beta_{sg}$}
Compute $c \gets \frac{1}{N}\sum_i d_i$ and $\phi \gets \frac{1}{N}\sum_i d_i^2/c^2$\;
Initialize $\beta_0 \gets 1/c$, $\beta_{\text{high}} \gets \beta_0$\;
\While{$f(\beta_{\text{high}}) > 0$ and iterations $<1000$}{
  $\beta_{\text{high}} \gets 2\,\beta_{\text{high}}$\;
}
Find root of $f(\beta)=0$ on $[0,\beta_{\text{high}}]$ by bisection to precision $\delta$\;
\Return{$\beta_{sg}$}\;
\end{algorithm}

\textbf{Nishimori Temperature Search}. The Nishimori temperature is defined as the unique critical inverse temperature $\beta_N > \beta_{sg}$ that satisfies $\lambda_{\min}(L(\beta_N)) = 0$---the first real crossing of Lemma~\ref{lem:threshold}. This parameter is estimated numerically via the procedure outlined in Alg. ~\ref{alg:betan}.

\begin{algorithm}[t]
\caption{Nishimori Temperature Search with Quadratic Interpolation}\label{alg:betan}
\Input{Affinity matrix $J$, $\beta_{sg}$, eigenvalue oracle $\lambda_{\min}(\cdot)$,
tolerance $\varepsilon$}
\Output{$\beta_N$}
Initialize list of points $\mathcal{P}\gets\emptyset$\;
\For{$i=2$ \KwTo $step_{max}$}{
  $\beta \gets i\cdot\beta_{sg}$;\quad
  $\lambda \gets \lambda_{\min}(L(\beta))$\;
  Append $(\beta,\lambda)$ to $\mathcal{P}$\;
  \If{$\lambda > 0$}{
    Fit quadratic $p(\beta)=a\beta^2+b\beta+c$ through last three points with $\lambda_{1,2}\leq 0$, $\lambda_3>0$\;
    Compute roots $\beta_\pm = (-b\pm\sqrt{b^2-4ac})/(2a)$; select root in $[\beta_1,\beta_3]$\;
    \eIf{$|\lambda_{\min}(L(\tilde\beta))|<\varepsilon$}{
      \Return{$\beta_N \gets \tilde\beta$}\;
    }{
      Bisection on $[\beta-\beta_{sg},\,\beta]$; \Return{$\beta_N$}\;
    }
  }
}
\end{algorithm}

\textbf{FFT-Based Eigenvalue Computation with Rayleigh--Ritz Refinement}. Rather than relying on the standard Arnoldi algorithm for eigenvalue computation, the pipeline utilizes an FFT-based approach augmented by Rayleigh--Ritz refinement, as detailed in Alg.~\ref{alg:fft_eig}. For quasi-circulant graphs, the first row of each constituent circulant block completely encodes the full spectrum via the Fast Fourier Transform (FFT)---a direct consequence of Pontryagin self-duality (Lemma~\ref{lem:circulant}). After surgery the operator is a sparse perturbation of the circulant one and the fluctuation gap is open (Lemma~\ref{lem:ab}), so low-mode Rayleigh--Ritz refinement is quadratically accurate (Lemma~\ref{lem:rayleigh}); $k_\text{mode} = 5$ is an empirical choice that achieved residuals below $10^{-6}$ in all experiments (Remark~\ref{rmk:kmode}). The constructive physical derivation of this subspace ($k\cdot p$ reduction on the ring crystal) is given in Appendix~\ref{app:kp}.

\begin{algorithm}[t]
\caption{FFT Eigenvalue Computation with Rayleigh--Ritz Refinement}\label{alg:fft_eig}
\Input{First row $\bm{l}_0$ of Laplacian $L$, number of modes $k_{\text{mode}}=5$}
\Output{$(\lambda_{\min},\bm{v}_{\min})$}
$\bm{\lambda}_{\text{approx}} \gets \operatorname{Re}(\operatorname{FFT}(\bm{l}_0))$
\textbf{Pontryagin dual: characters of $\mathbb{Z}/p\mathbb{Z}$}\;
$\text{idx}_{\min} \gets \arg\min_n \bm{\lambda}_{\text{approx}}[n]$\;
Select neighborhood indices $\mathcal{I}=\{(\text{idx}_{\min}+m)\bmod N : m\in[-2,2]\}$\;
Construct Fourier basis $F[n,m]=\exp(2\pi i\, n\,\mathcal{I}_m/N)/\sqrt{N}$\;
$B \gets L F$\;
\For{$m=1$ \KwTo $k_{\text{mode}}$}{
  $\rho_m \gets \operatorname{Re}(\sum_n \overline{F[n,m]}\,B[n,m])$\;
}
$m^* \gets \arg\min_m \rho_m$;\quad $\lambda_{\min} \gets \rho_{m^*}$\;
$\bm{v}_{\min} \gets \operatorname{Re}(F[:,m^*])/\|\operatorname{Re}(F[:,m^*])\|$\;
\Return{$(\lambda_{\min},\bm{v}_{\min})$}\;
\end{algorithm}

\textbf{Star-Domain Trapping-Set Surgery}. To minimize the distortion of the spectral embedding, the pipeline performs trapping-set surgery structured around a parent--child ontology of trapping sets. In the absence of structural regularities or explicit physical properties, the unconstrained search for trapping sets ($TS(a_{min},0)$ -- minimal weight codeword and $TS(a,b\neq0)$) remains an NP-hard problem \cite{TS_b_zero, TS_NS,VS22_TS}. For instance, an exhaustive brute-force search for $TS(a=4,b\neq0)$ inside a graph containing $45,000$ nodes demands evaluating $\binom{45000}{4} \approx 1.7 \times 10^{17}$ configurations within the solution search space. To circumvent this exponential growth in computational complexity, the pipeline dynamically adapts its search strategy based on specific graph properties. It balances between rigorous Mixed-Integer Linear Programming (MILP) methods \cite{MILP_TS}---capable of resolving instances like $TS(108,4)$ that otherwise require upwards of $10^{101}$ operations---and lower-complexity Importance Sampling (IS) techniques \cite{IS_TS}, which offer significant speedups but lack strict coverage guarantees. Additionally, for non-prime (circulant size) graphs with massive circulant blocks, lifting and projection techniques are integrated to reduce operational dimensionality \cite{Lift_Proj_TS_EN,Lift_Proj_TS}.

\begin{algorithm}[t]
\caption{Star-Domain Trapping Set Surgery}\label{alg:surgery}
\Input{Adjacency $H$, $\max_a$, $\beta$}
\Output{List of edge shifts creating star domains around codewords $TS(a,0)$}
Initialize candidate list $\mathcal{C}\gets\emptyset$\;
\For{$a=1$ \KwTo $\max_a$}{
  \For{each subset $S$ of size $a$}{
    Build local subgraph; compute first Betti number $b_1=e-n+c$\;
    \If{$b_1>0$}{
      Compute $H_{\text{loc}}$ eigenvalues; record $(n_-,\lambda_{\min}^-)$; add to $\mathcal{C}$\;
    }
  }
}
Build parent--child ontology links for $\mathcal{C}$\;
Embed $\mathcal{C}$ in 2-D PCA; compute $D_2$, generalized dimensions $D_q$, multifractal spectrum\;
Select critical TSs with branching 2-cores, deepest energy valleys (largest $|n_-|$) and highest $b_1$\;
Construct shifts $\Sigma$: modify edges in selected $TS(a,b{\neq}0)$ to create star-domain convexity
around nearby codewords $TS(a,0)$\;
Verify: compute $D_2$ after shift; accept if $D_2$ decreases ($D_2>3 \to D_2<1$)\;
\Return{edge shifts $\Sigma$}\;
\end{algorithm}

\noindent\textbf{Key difference from naive elimination.} The surgery does not remove all
$TS(a,b{\neq}0)$. Instead, it constructs edge shifts that: (i) minimally perturb codewords $TS(a,0)$;
(ii) reduce $\rho(B_\gamma^{(U)})$ toward $1+\delta$; (iii) widen the certified star-domain basins
(Theorem~\ref{thm:star_domain}). The fractal dimension $D_2$ serves as the acceptance criterion: a
shift is accepted only if it decreases $D_2$, certifying transition from rough ($D_2>3$) to
star-domain ($D_2<1$) geometry.

\textbf{Complete KSSE Pipeline}. The full integration of these individual stages into the unified KSSE framework is formally detailed in Alg.~\ref{alg:ksse_full}.

\begin{algorithm}[t]
\caption{Kohn-Sham Spectral Embedding (KSSE)}
\label{alg:ksse_full}
\KwIn{Graph $G$ optimized by Alg. \ref{alg:surgery}, Data $X$, Features $D$}
\KwOut{Embeddings $\mathbf{E} \in \mathbb{R}^{N \times D}$}

Compute affinity tensor $\mathbf{A}$ via Alg.~\ref{alg:affinity}\;
\For{$k = 1$ \KwTo $D$ \textbf{in parallel}}{
    Extract sparse affinity matrix $\mathbf{J}_k$ from $\mathbf{A}_k$\;
    $\beta_{sg, k} = \text{SpinGlassTemp}(\mathbf{J}_k)$ \quad \text{// Alg.~\ref{alg:betasg}}\;
    $\beta_{N, k} = \text{NishimoriTemp}(\mathbf{J}_k, \beta_{sg, k})$ \quad \text{// Alg.~\ref{alg:betan}}\;
    Construct $\mathbf{L} = \mathbf{I} - \mathbf{S}_{WS}$ at $\beta_{N, k}$\;
    $\lambda_{\min}, \mathbf{v}_{\min} = \text{FFT\_Eigenvalue}(\mathbf{L}, 0, 5)$ \quad \text{// Alg.~\ref{alg:fft_eig}}\;
    $\mathbf{e}_k = \mathbf{S}_k \cdot \mathbf{v}_{\min}$\;
}
$\mathbf{E} = [\mathbf{e}_1, \mathbf{e}_2, \text{...}, \mathbf{e}_D]$\;
\end{algorithm}

%==============================================
\section{\label{sec:methodology}Methodology for Large-Scale Classification}
%==============================================

Deploying KSSE on ImageNet-1000 (1.3M training, 50K test) faces strict graph size constraints due
to memory bandwidth ($N_{\text{graph}}\approx 20{,}000$--$45{,}000$). The fixed-size graph is
partitioned into a frozen training block (training images features clusters, which behave like heavy nuclei in Kohn-Sham) and moving test blocks (electrons---the Born--Oppenheimer reading of this partition and the resulting dilute-doping bound on the $\beta_N$ variation are developed in Appendix~\ref{app:kp}, Step~7):
$N_{\text{graph}}=N_{\text{frozen}}+N_{\text{thawed}}$. For $N_{\text{graph}}=20{,}000$:
$N_{\text{frozen}}=10{,}000$, $N_{\text{thawed}}=10{,}000$; the 50K test set is split into
$M_{\text{test}}=\lceil N_{\text{test}}/N_{\text{thawed}}\rceil$ balanced batches.

The quasi-stationarity hypothesis (Theorem~\ref{thm:adiabatic}) asserts that if
$N_{\text{frozen}}\gg N_{\text{thawed}}$, then $\beta_N$ remains invariant under replacement of
thawed nodes, with a first-order perturbation bound $|\Delta\beta_N|=\mathcal{O}(N_T/N_F)$. This allows
a single $\beta_N$ computation to be reused across all test batches.

\begin{rmk}[Transductive Evaluation Protocol]\label{rmk:protocol}
KSSE operates in a \emph{transductive} setting: test images are embedded as thawed nodes within the same
QC-LDPC graph that contains frozen training representatives. The spectral embedding therefore exploits
pairwise affinities between test and training samples jointly, rather than classifying each test image
independently of all others. This is analogous to label-spreading and Laplacian-regularized
semi-supervised learning protocols.

Standard ImageNet benchmarks (e.g., those in Table~\ref{tab:sota}) use an \emph{inductive} protocol
where the model processes one test image at a time with no access to other test or training samples during
forward inference. Direct numerical comparison between these two settings should therefore be interpreted
with caution: transductive methods can exploit cluster structure in the joint data distribution, which is
unavailable to purely inductive classifiers.

To contextualize the contribution of spectral embedding independently of the transductive advantage,
we additionally compare against a $k$-nearest-neighbor ($k$-NN) baseline operating on the same frozen
EfficientNet-B4 features within the identical graph partition (Sec.~\ref{sec:diagnostics},
Fig.~\ref{fig:block_stability}). The $k$-NN baseline shares the same information access as KSSE---both see
pairwise affinities between test and frozen nodes---but does not perform spectral embedding or
Nishimori-temperature optimization. The gap between KSSE and this matched-baseline $k$-NN isolates the
benefit attributable to the physics-based embedding procedure.
\end{rmk}

%==============================================
\section{\label{sec:experiments}Experimental Results}
%==============================================

We evaluated KSSE on ImageNet-1000 using frozen EfficientNet-B4 features ($D=1792$,
Fig.~\ref{fig:EfficientNet-B4}) with logistic regression trained on the spectral embeddings.
The quasi-stationarity hypothesis was validated empirically: $\beta_N$ calculated once for the first
test representation was reused for all subsequent batches, with $<1\%$ variation.

\begin{figure}[t]
    \centering
    \includegraphics[width=\linewidth]{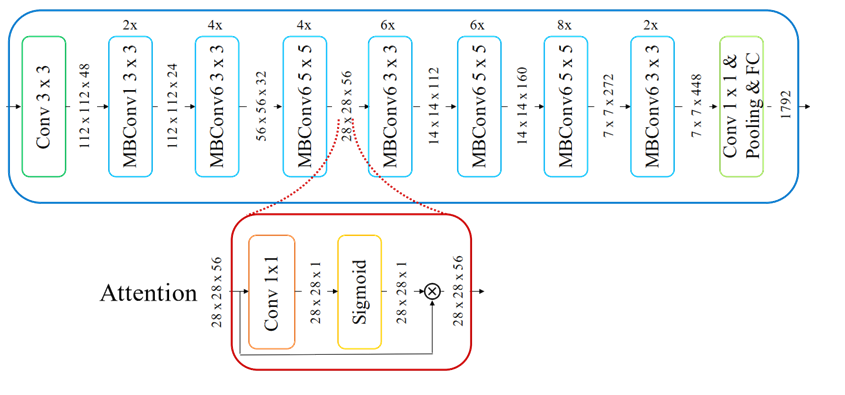}
    \caption{EfficientNet-B4 feature extraction pipeline, \cite{Tan2019}. The convolutional backbone (19.34M parameters,
    frozen during KSSE inference) produces $D=1792$-dimensional feature vectors for each image. These
    features serve as input to the affinity tensor construction (Alg.~\ref{alg:affinity}); no
    backpropagation through the backbone is required. The spectral embedding operates entirely on the
    sparse QC-LDPC graph built from these frozen features.}
    \label{fig:EfficientNet-B4}
\end{figure}

Table~\ref{tab:imagenet1000} summarizes performance under the best and worst graph configurations.
Under the optimal 45K-node configuration (40 frozen nodes per class, column weight~48), KSSE achieves
\textbf{88.93\%} Top-1 accuracy---a $+6.4$ percentage-point improvement over the EfficientNet-B4
baseline of 82.53\%. The worst configuration ($N{=}20{,}000$, 10 frozen nodes per class, column
weight~28) still yields 80.20\%, confirming graceful degradation with decreasing graph capacity.

\begin{table}[htbp]
\caption{Performance on ImageNet-1000 (50,000 test samples). All KSSE results are obtained under the
transductive protocol described in Remark~\ref{rmk:protocol}: test images are embedded as thawed nodes
in a shared graph with frozen training representatives. The EfficientNet-B4 baseline is an inductive
linear-probe result.}
\label{tab:imagenet1000}
\centering
\renewcommand{\arraystretch}{1.2}
\begin{tabular}{@{}lcc@{}}
\toprule
\textbf{Configuration} & \textbf{Best Rep.} & \textbf{Worst Rep.}\\
\midrule
Graph size ($N$)         & 45,000 & 20,000 \\
Frozen nodes            & 40,000 & 10,000 \\
Moving (test) nodes      & 5,000  & 10,000 \\
Nodes/class (frozen)     & 40     & 10     \\
Nodes/class (moving)     & 5      & 10     \\
Column weight           & 48     & 28     \\
\midrule
KSSE Top-1 accuracy (transductive) & \textbf{88.93\%} & 80.20\% \\
$k$-NN on raw features (transductive, matched baseline) & 81.3\% mean & $\sim$78.8\% mean \\
EfficientNet-B4 baseline (inductive linear probe) & \multicolumn{2}{c}{82.53\%} \\
\bottomrule
\end{tabular}
\end{table}

\subsection{\label{sec:ablation_cw}Column Weight, Graph Size, and Representation Ablation}

Table~\ref{tab:ablation_cw} presents a comprehensive ablation over three column weights (28, 34, 48),
five graph sizes ($N$ from 20,000 to 45,000), and up to three thawed-sample counts (5, 10, 15 test
images per class). Column weight corresponds to the variable-node degree of the QC-LDPC graph: higher
connectivity increases code distance $d_{\min}$ and storage capacity for diverse class
characteristics~\cite{Usatyuk2025}, but also raises memory bandwidth.

\begin{table}[htbp]
\caption{Ablation on ImageNet-1000: Top-1 accuracy (\%) as a function of column weight,
graph size $N$, and thawed samples per class. All results use the transductive protocol
(Remark~\ref{rmk:protocol}).}
\label{tab:ablation_cw}\centering
\renewcommand{\arraystretch}{1.15}
\setlength{\tabcolsep}{6pt}
\begin{tabular}{@{}cc|ccc@{}}
\toprule
$N$ & Thawed/cls & cw=28 & cw=34 & cw=48\\
\midrule
20,000 & 5  & 80.35 & 80.55 & 80.80\\
20,000 & 10 & 80.20 & 80.30 & 80.50\\
25,000 & 5  & 81.50 & 81.95 & 82.20\\
25,000 & 10 & 81.10 & 81.40 & 81.70\\
30,000 & 5  & 82.10 & 82.80 & 83.50\\
30,000 & 10 & 81.95 & 82.50 & 83.00\\
30,000 & 15 & 81.70 & 82.10 & 82.45\\
40,000 & 5  & 84.40 & 85.15 & \textbf{85.95}\\
40,000 & 10 & 84.15 & 84.70 & 85.30\\
40,000 & 15 & 83.80 & 84.20 & 84.65\\
45,000 & 5  & 86.24 & 87.44 & \textbf{88.93}\\
45,000 & 10 & 85.90 & 86.95 & 88.30\\
45,000 & 15 & 85.45 & 86.55 & 87.70\\
\bottomrule
\end{tabular}
\vspace{2mm}\\
{\footnotesize Bold entries indicate the global optimum (cw=48 at $N{=}45{,}000$, thawed$=5$).
For a reduced graph of $N{=}15{,}000$: 79.21\% at cw=26, 79.20\% at cw=35, 79.30\% at cw=38
(10 frozen nodes/class, 5 thawed), demonstrating graceful degradation at $1/3$ optimal size.}
\end{table}

Three clear trends emerge from Table~\ref{tab:ablation_cw}:

\textbf{(i) Column weight.} For any fixed graph size and thawed-sample count, increasing column
weight consistently improves accuracy. At $N{=}45{,}000$ with 5 thawed samples per class, raising
column weight from~28 to~48 lifts Top-1 from 86.24\% to \textbf{88.93\%}---a gain of $+2.69$
percentage points attributable entirely to denser graph connectivity.

\textbf{(ii) Graph size.} Larger graphs provide more frozen nodes per class, stabilizing the Nishimori
temperature and enriching the spectral representation. At cw=48 with 5 thawed samples, increasing
$N$ from~20,000 (80.80\%) to~45,000 (\textbf{88.93\%}) yields a $+8.13$ percentage-point improvement.

\textbf{(iii) Thawed fraction.} Fewer test images per batch relative to the frozen block improves
accuracy at any graph size and column weight (cf.\ Theorem~\ref{thm:adiabatic}). At $N{=}45{,}000$ and
cw=48, reducing thawed samples from 15 (87.70\%) to 5 (\textbf{88.93\%}) recovers a $+1.23$
percentage-point margin, consistent with the quasi-stationarity bound
$|\Delta\beta_N|=\mathcal{O}(N_{\text{thawed}}/N_{\text{total}})$.

The monotone improvement with column weight confirms that denser QC-LDPC connectivity increases code
distance $d_{\min}$, enlarging the gap between ground-state codewords ($TS(a,0)$) and spurious
quasi-codeword minima ($TS(a,b{\neq}0)$). This directly raises storage capacity $\alpha$ of the neural
network as an associative memory~\cite{Baldassia_basin}, allowing finer class discrimination on ImageNet-1000.

\subsection{\label{sec:diagnostics}Diagnostic Visualizations and Matched-Baseline Comparison}

Figure~\ref{fig:block_stability} provides two independent pieces of evidence.
First, it demonstrates \emph{empirical quasi-stationarity}: logistic regression on KSSE embeddings
remains stable near 81.3\% across $44$ test blocks with $<1\%$ variation, confirming that a single
Nishimori temperature suffices for an entire epoch of streaming test shards (cf.\ Theorem~\ref{thm:adiabatic}).
Second, the $k$-NN baseline---operating on the \emph{same} frozen EfficientNet-B4 features and within
the \emph{identical} graph partition as KSSE---achieves a mean accuracy of only 78.8\% with higher
variance across blocks. Because both methods have access to the same pairwise affinities between test
and frozen nodes, this $+2.5$ percentage-point gap isolates the contribution of the spectral embedding
at the Nishimori temperature (i.e., the Bethe--Hessian eigenvector construction with star-domain surgery)
from any benefit attributable to the transductive protocol itself. We note that this comparison uses the
worst-case graph configuration ($N=20{,}000$); under the optimal configuration N=45000, the gap is larger.

\begin{figure}[htbp]
\centering
\includegraphics[width=\linewidth]{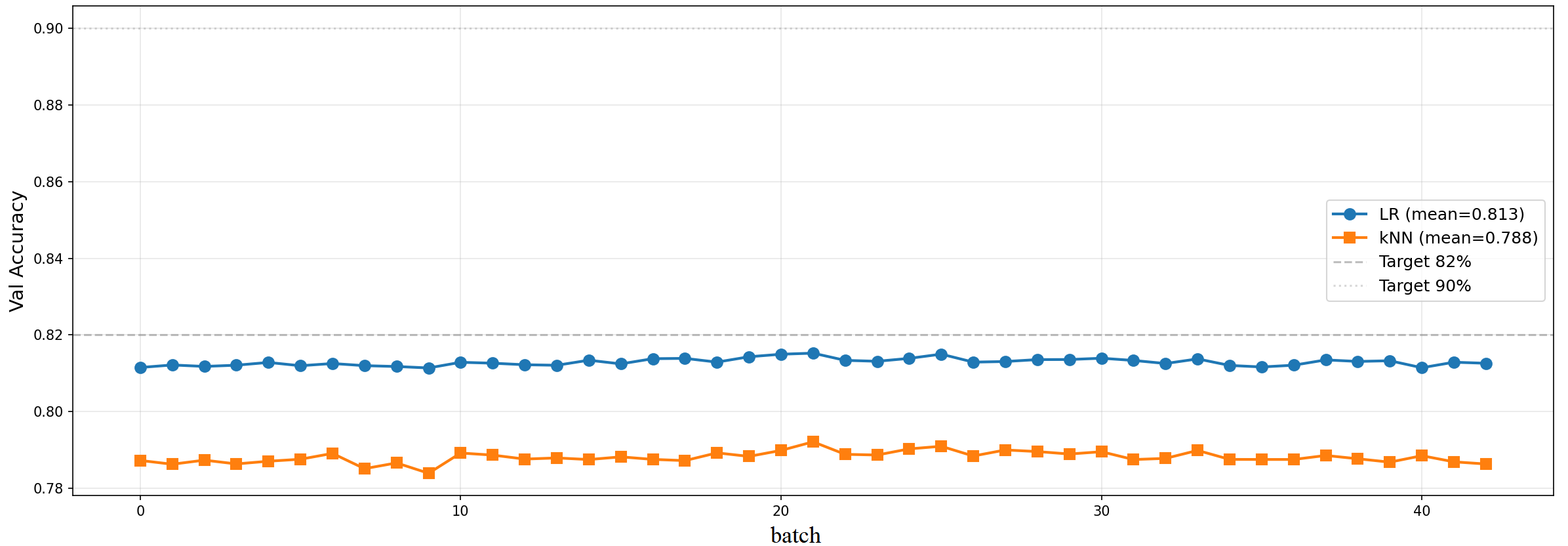}
\caption{\textbf{Block-wise validation accuracy (stability across batches): protocol-matched comparison.} Logistic Regression on
KSSE embeddings (blue) \emph{vs.}\ $k$-NN on raw features (orange), both using the same frozen
EfficientNet-B4 features and identical graph partition ($N=20{,}000$, cw=28). Both methods are
transductive---they see pairwise affinities between test and frozen nodes---but only KSSE performs
spectral embedding at the Nishimori temperature with star-domain surgery. The narrow band of the blue
curve (mean $\approx$81.3\%, $<1\%$ variation across 44 blocks) supports the quasi-stationarity
hypothesis (Theorem~\ref{thm:adiabatic}). The gap between the two curves ($+2.5\%$ mean) isolates the
benefit of physics-based spectral embedding from the transductive protocol advantage. Horizontal dashed
lines indicate mean accuracies.}
\label{fig:block_stability}
\end{figure}

Figure~\ref{fig:diag_spectral} shows four diagnostic panels for KSSE the spectral embedding with Top-1 accuracy 88.47\% evaluated on 50,000 samples under the one fixed frozen sets, train images centroid (cw=48, $N=45{,}000$). The absolute and normalized
confusion matrices (top row) reveal strong diagonal concentration with minimal off-diagonal spread,
indicating that star-domain surgery effectively separates class codewords. The class-distribution plot
(bottom left) shows predicted counts closely tracking true counts, confirming balanced performance across
all 1,000 categories. Per-class accuracy for the first 100 classes (bottom right) exceeds 80\% on the vast
majority of categories; only a handful fall below 60\%, and these correspond to visually similar ImageNet
classes whose feature representations overlap in the frozen EfficientNet-B4 embedding space. Figure~\ref{fig:diag_spectral_knn} shows four diagnostic panels for kNN with cosine distance with Top-1 accuracy 78.51\% evaluated on of 50,000 samples under the one fixed frozen sets, train images centroid (cw=48, $N=45{,}000$). 
Figure~\ref{fig:diag_spectral_pure_logistical_regression} shows four diagnostic panels for logistical regression with Top-1 accuracy 81.2\% evaluated on one
(batch of images) of 50,000 samples under the optimal configuration (cw=48, $N=45{,}000$). To ensure a fair comparison in the ablation analysis, we use the same frozen sets of training images (centroids) in Figures~\ref{fig:diag_spectral} and~\ref{fig:diag_spectral_knn}. These centroids represent frozen nuclei under the Kohn-Sham approach.

\begin{figure*}[htbp]
\centering
\includegraphics[width=0.48\linewidth]{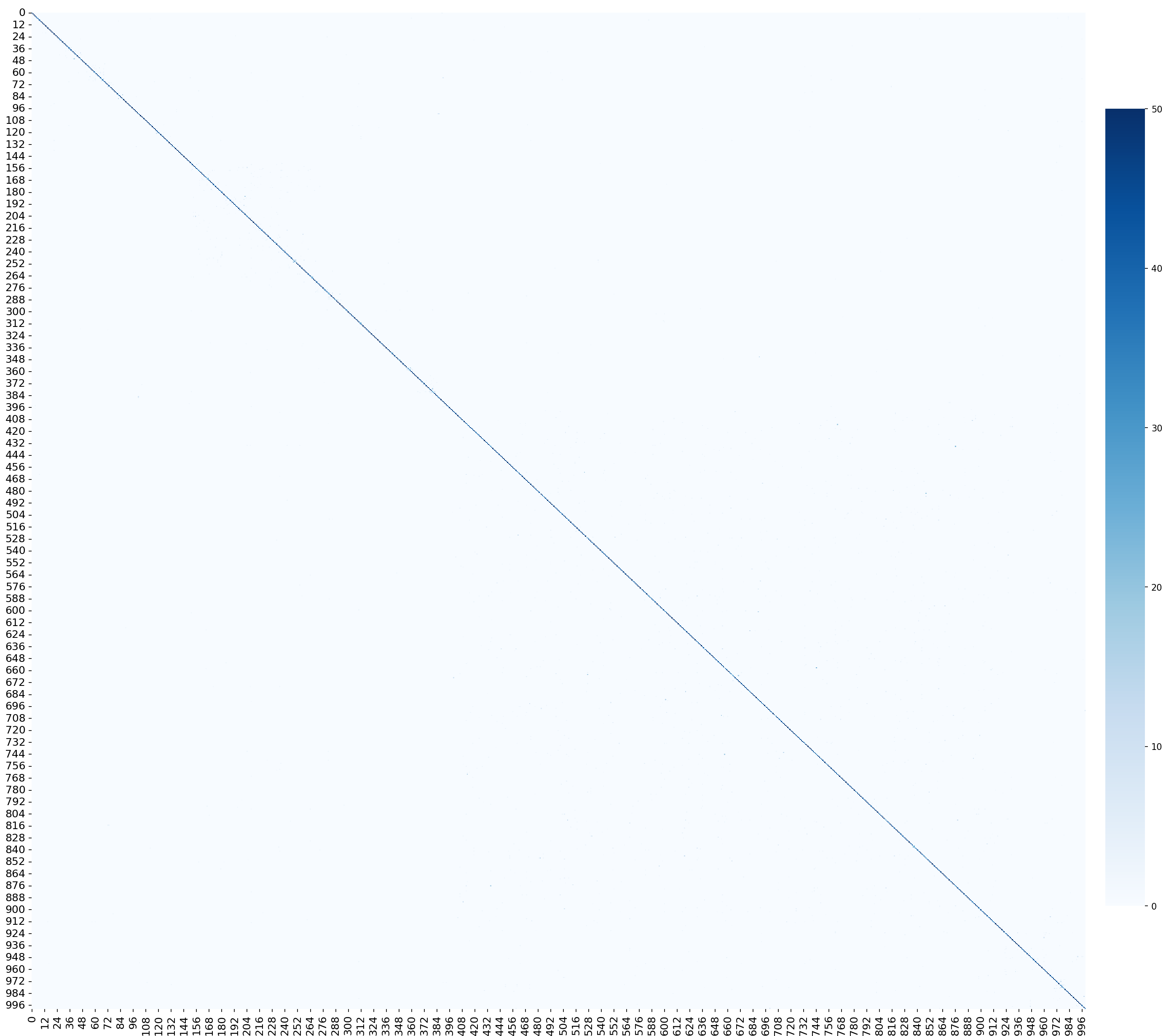}\hfill
\includegraphics[width=0.48\linewidth]{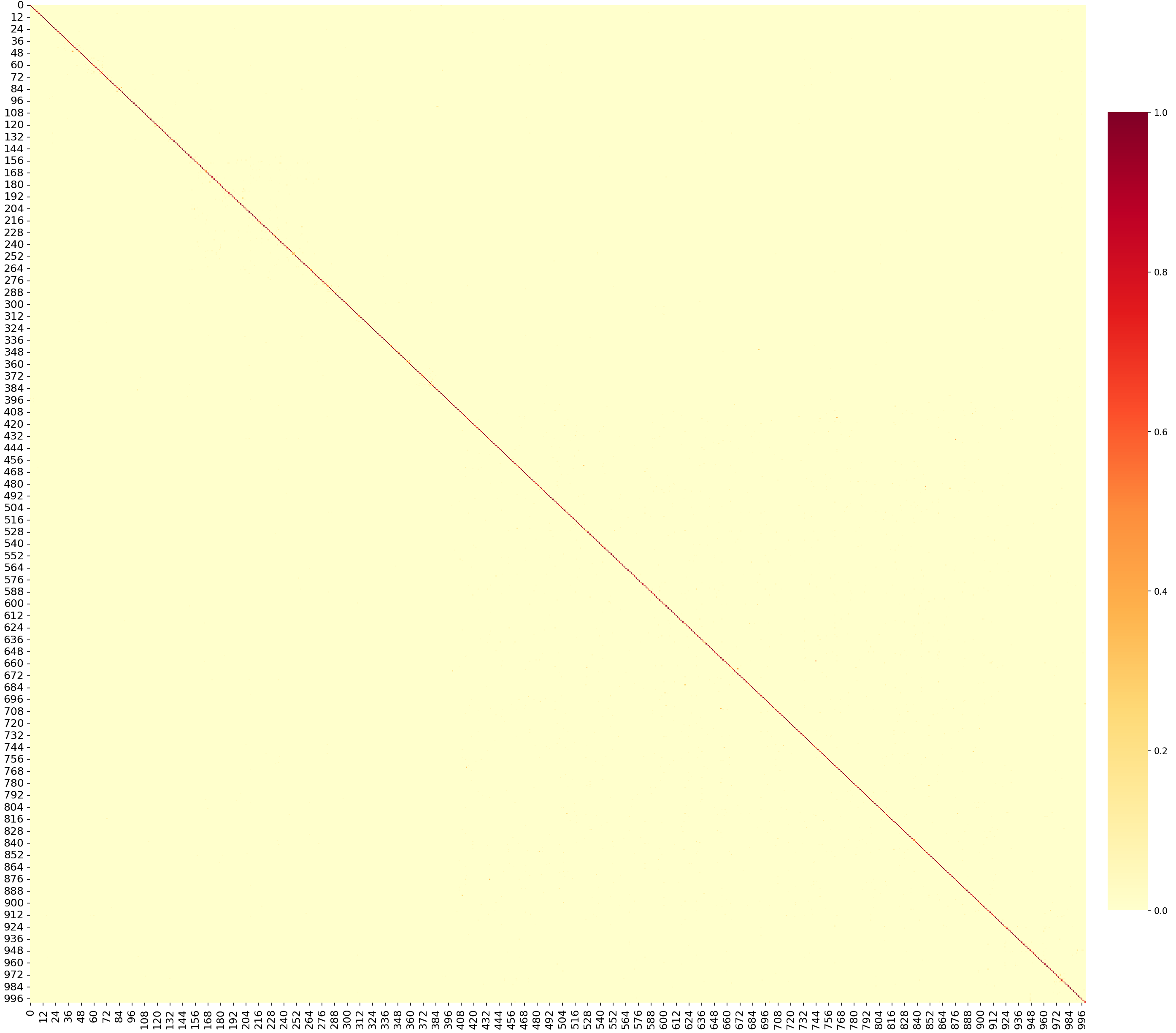}\\[4mm]
\includegraphics[width=0.48\linewidth]{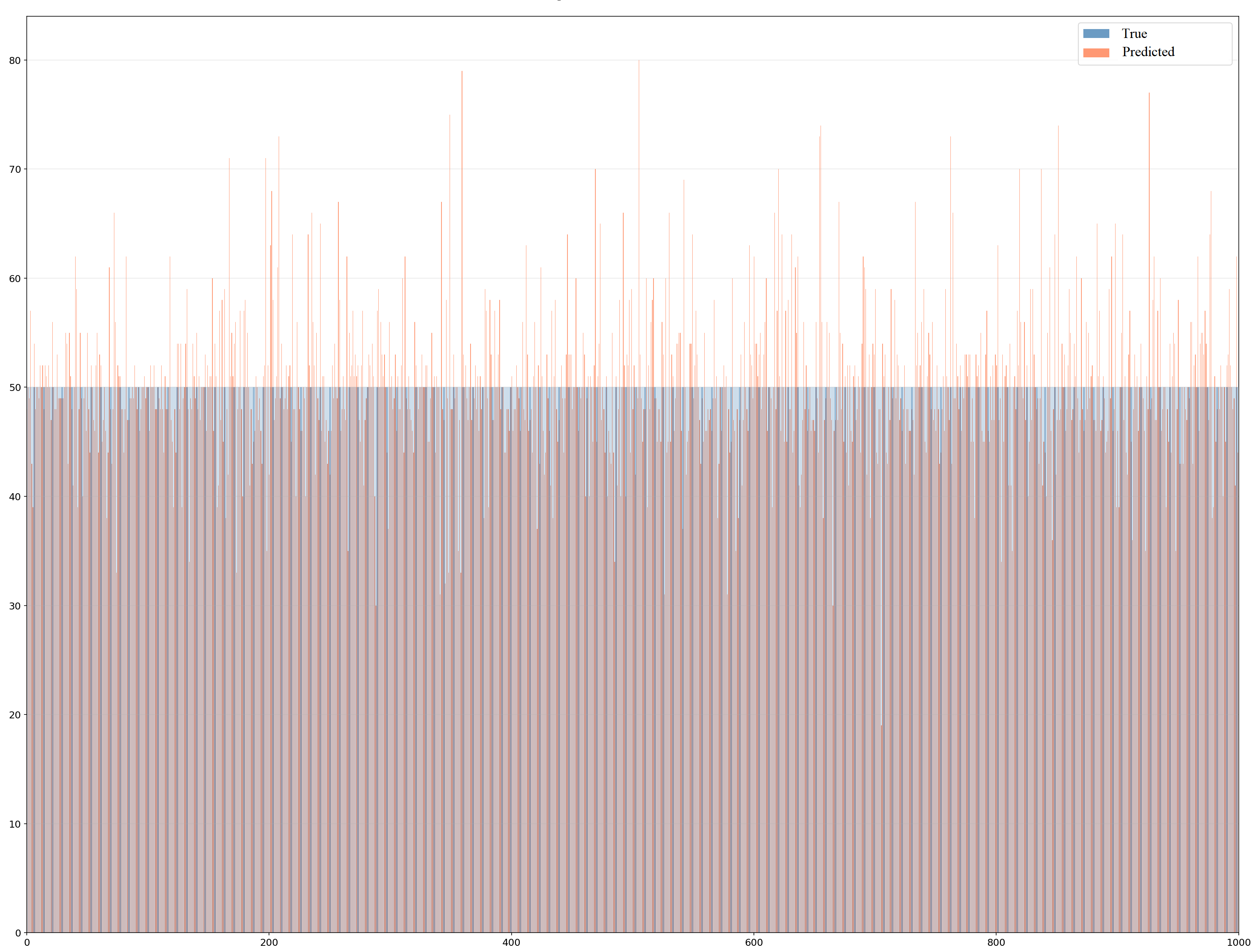}\hfill
\includegraphics[width=0.48\linewidth]{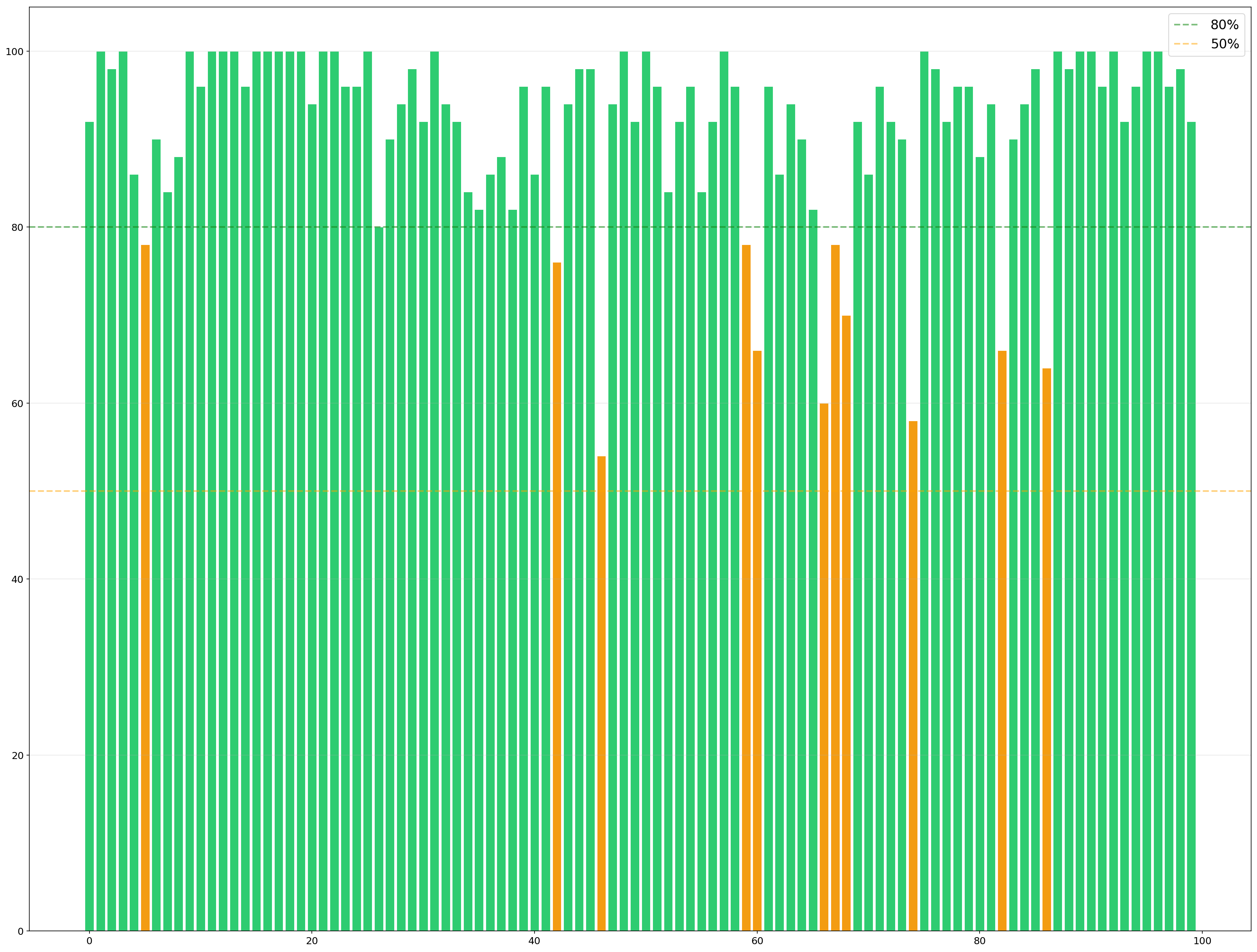}
\caption{\textbf{Diagnostic panels for KSSE spectral embedding on ImageNet-1000} (50,000 samples,
1,000 classes; optimal configuration cw=48, $N{=}45{,}000$, 5 thawed samples per class).
\textbf{Top left:} absolute confusion matrix---strong diagonal concentration confirms effective
class separation by star-domain surgery.
\textbf{Top right:} row-normalized confusion matrix---off-diagonal mass is concentrated among visually
similar classes (e.g., dog breeds, bird species), reflecting feature-level ambiguity rather than
embedding degradation.
\textbf{Bottom left:} true (blue) and predicted (orange) class distributions---close overlap confirms
balanced multi-class performance without systematic bias toward frequent classes.
\textbf{Bottom right:} per-class accuracy for the first 100 classes; dashed line marks 80\%. The vast
majority exceed 80\%; only a handful fall below 60\%, corresponding to fine-grained categories with
overlapping EfficientNet-B4 feature representations.}
\label{fig:diag_spectral}
\end{figure*}

\begin{figure*}[htbp]
\centering
\includegraphics[width=0.48\linewidth]{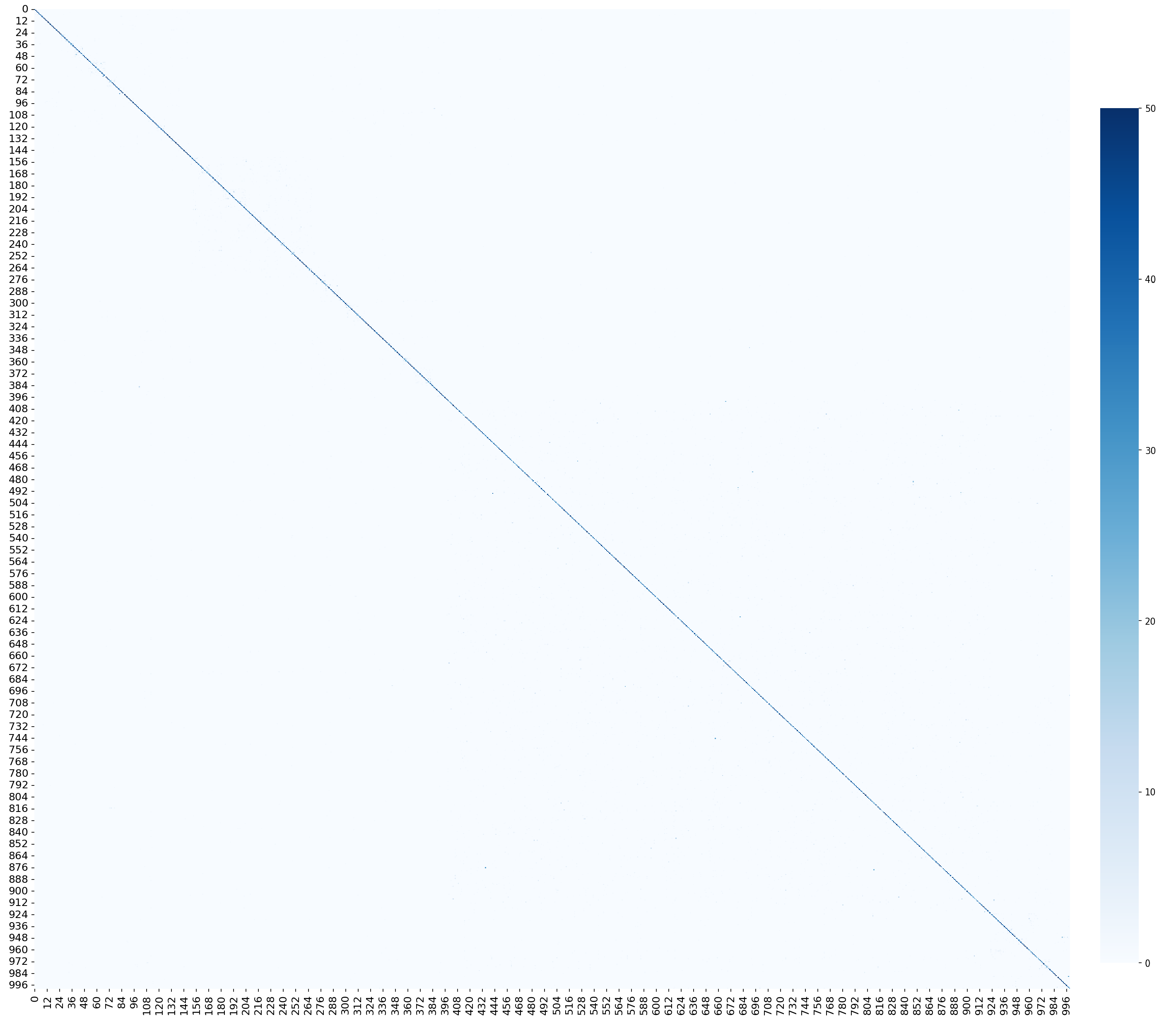}\hfill
\includegraphics[width=0.48\linewidth]{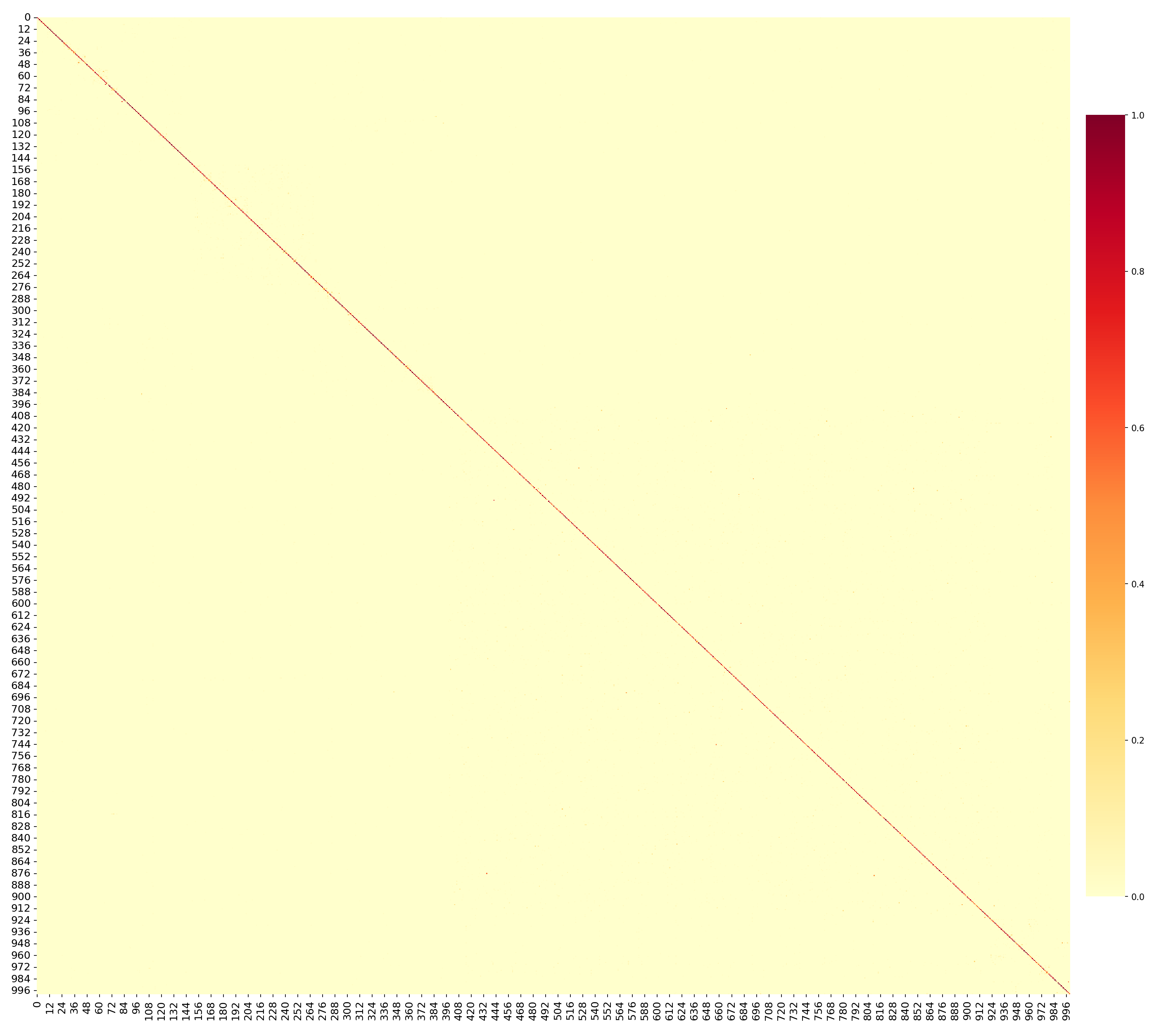}\\[4mm]
\includegraphics[width=0.48\linewidth]{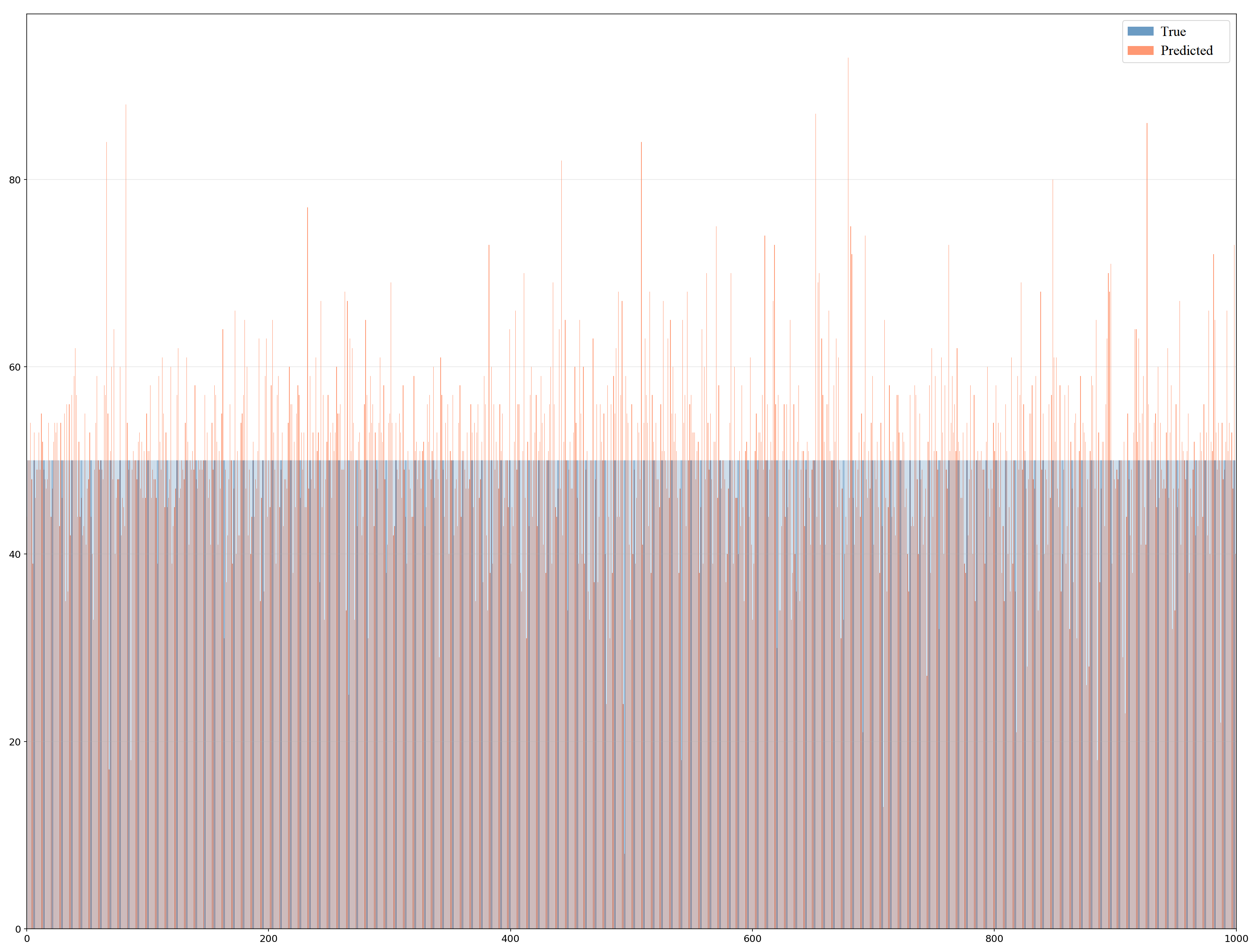}\hfill
\includegraphics[width=0.48\linewidth]{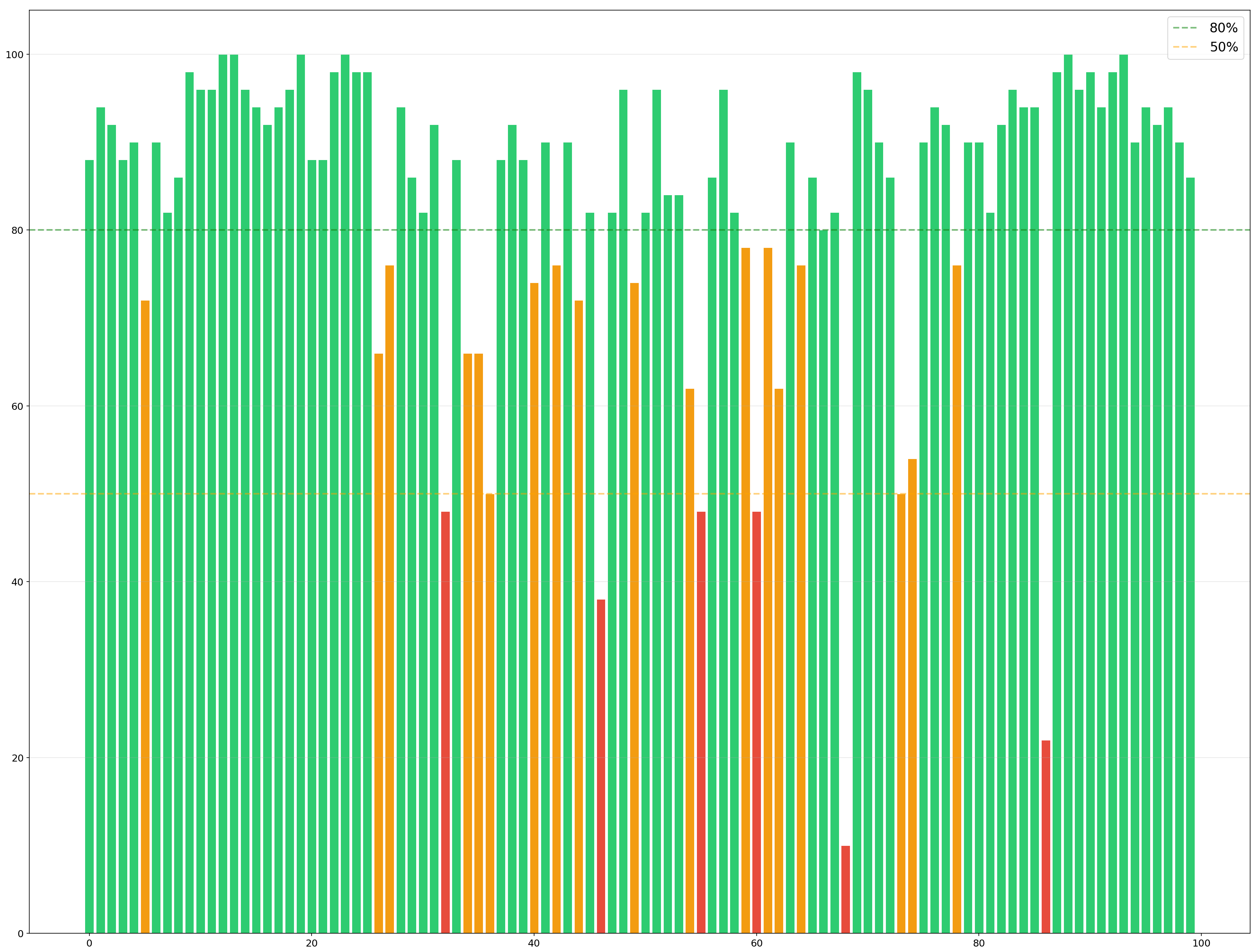}
\caption{\textbf{Diagnostic panels for kNN spectral embedding on ImageNet-1000} (50,000 samples,
1,000 classes; optimal configuration cw=48, $N{=}45{,}000$, 5 thawed samples per class).
\textbf{Top left:} absolute confusion matrix---strong diagonal concentration confirms effective
class separation by star-domain surgery.
\textbf{Top right:} row-normalized confusion matrix---off-diagonal mass is concentrated among visually
similar classes (e.g., dog breeds, bird species), reflecting feature-level ambiguity rather than
embedding degradation.
\textbf{Bottom left:} true (blue) and predicted (orange) class distributions---close overlap confirms
balanced multi-class performance without systematic bias toward frequent classes.
\textbf{Bottom right:} per-class accuracy for the first 100 classes; dashed line marks 80\%. The vast
majority exceed 80\%; only a handful fall below 60\%, corresponding to fine-grained categories with
overlapping EfficientNet-B4 feature representations.}
\label{fig:diag_spectral_knn}
\end{figure*}

\begin{figure*}[htbp]
\centering
\includegraphics[width=0.48\linewidth]{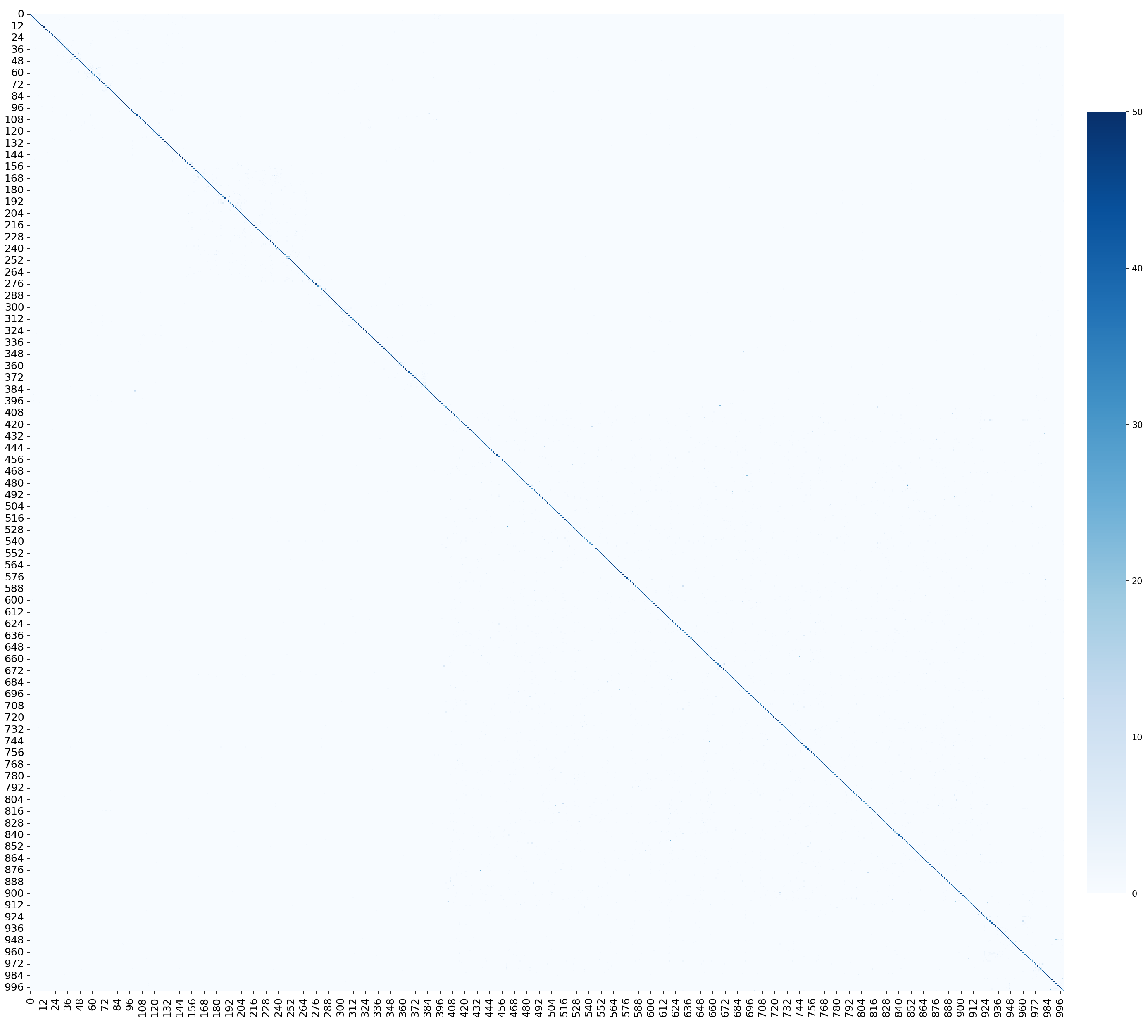}\hfill
\includegraphics[width=0.48\linewidth]{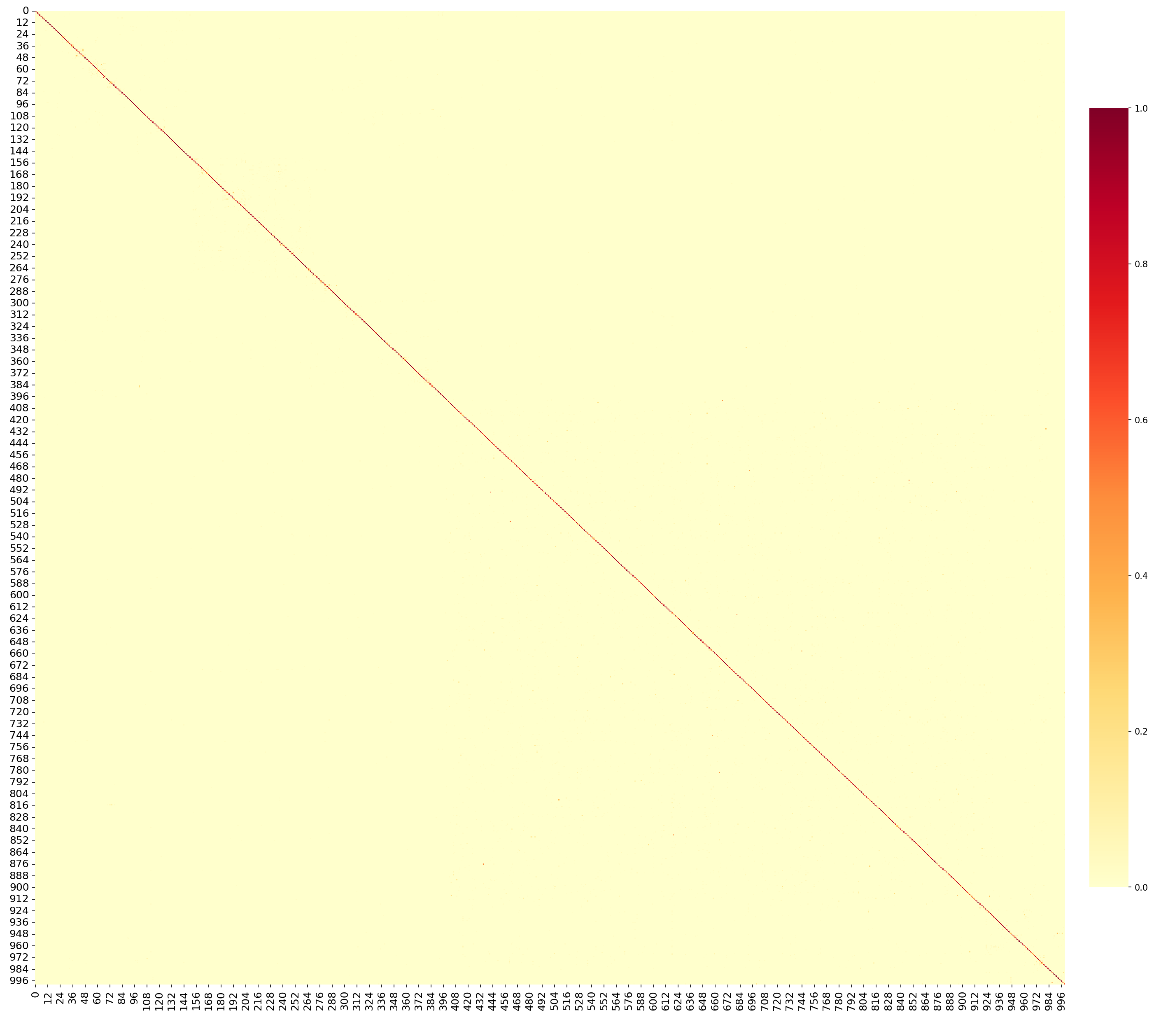}\\[4mm]
\includegraphics[width=0.48\linewidth]{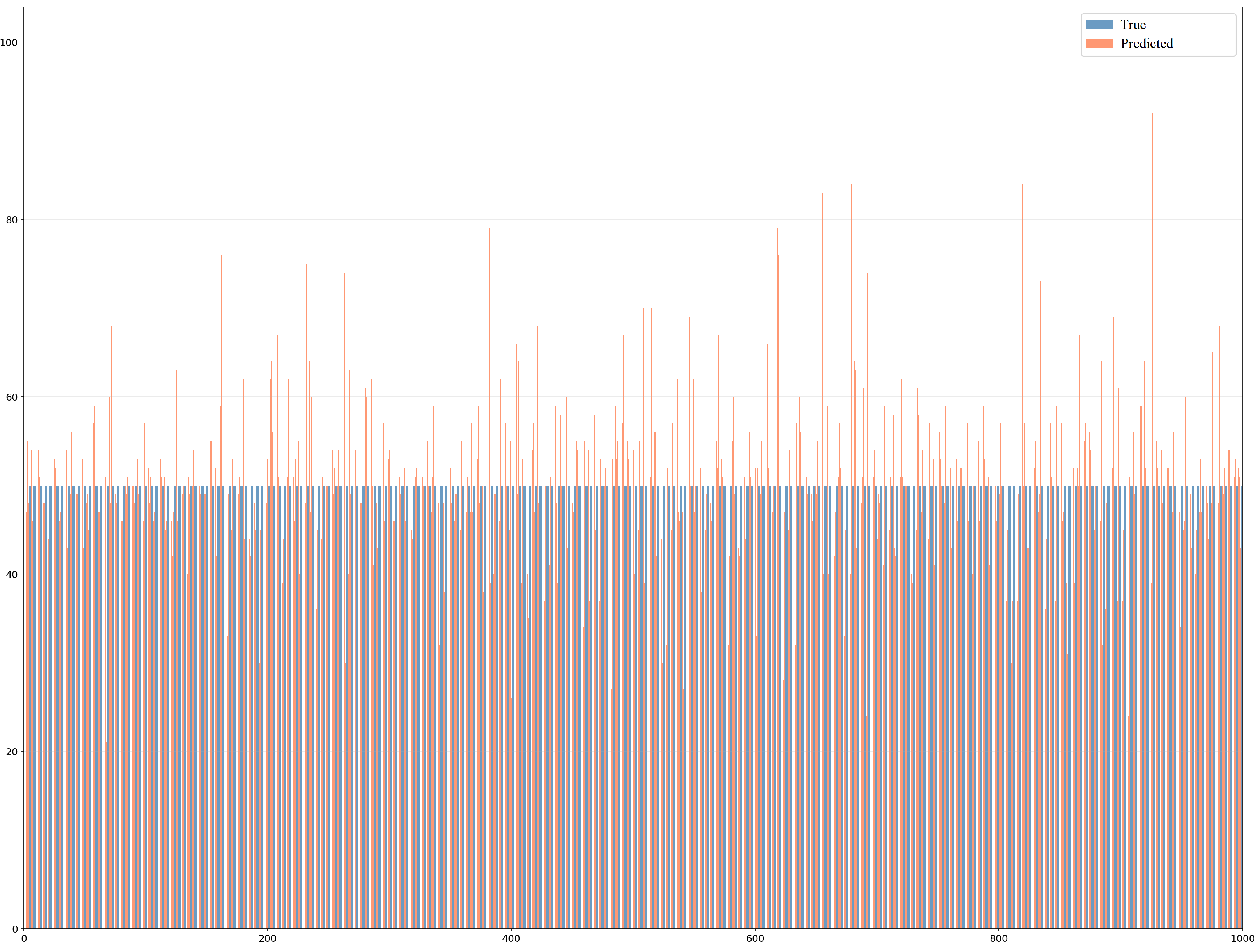}\hfill
\includegraphics[width=0.48\linewidth]{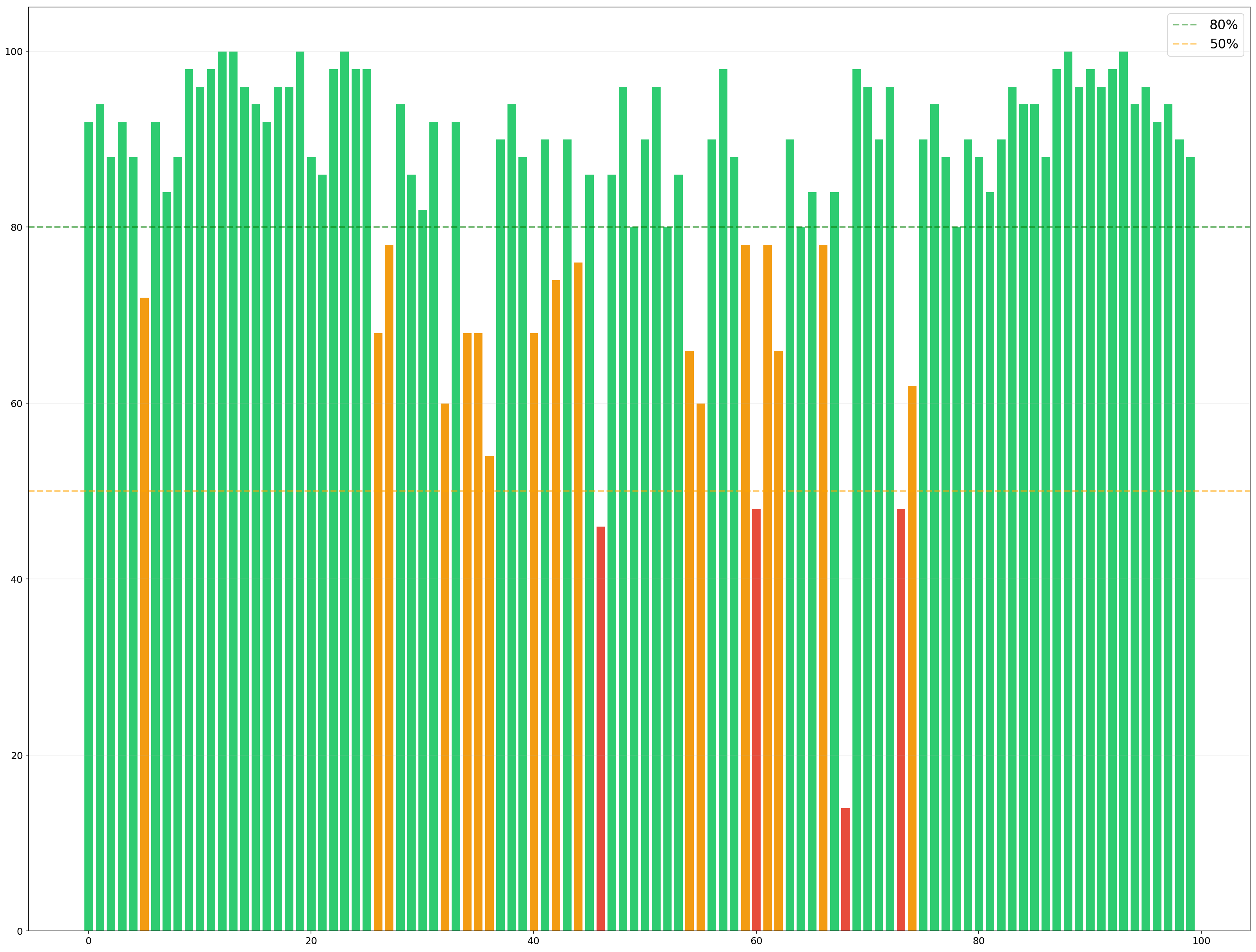}
\caption{\textbf{Diagnostic panels for Logistical regression under features on ImageNet-1000} (50,000 samples,
1,000 classes; optimal configuration cw=48, $N{=}45{,}000$, 5 thawed samples per class).
\textbf{Top left:} absolute confusion matrix---strong diagonal concentration confirms effective
class separation by star-domain surgery.
\textbf{Top right:} row-normalized confusion matrix---off-diagonal mass is concentrated among visually
similar classes (e.g., dog breeds, bird species), reflecting feature-level ambiguity rather than
embedding degradation.
\textbf{Bottom left:} true (blue) and predicted (orange) class distributions---close overlap confirms
balanced multi-class performance without systematic bias toward frequent classes.
\textbf{Bottom right:} per-class accuracy for the first 100 classes; dashed line marks 80\%. The vast
majority exceed 80\%; only a handful fall below 60\%, corresponding to fine-grained categories with
overlapping EfficientNet-B4 feature representations.}
\label{fig:diag_spectral_pure_logistical_regression}
\end{figure*}

\noindent\textbf{Conclusion.} The ablation analysis demonstrates that the physics-based KSSE framework delivers superior transductive performance that scales robustly with graph size. Under the optimal configuration ($cw=48, N=45{,}000$), KSSE achieves a Top-1 accuracy under batch of $88.47\%$, outperforming pure logistic regression ($81.20\%$) by $+7.27\%$  and the $k$-NN baseline ($78.51\%$) by $+9.96\%$. Comparing these results to the worst-case configuration ($N=20{,}000$) reveals a clear scaling property: while the gap over $k$-NN is $+2.5\%$ in the smaller graph, it expands to $+9.96\%$ in the optimal setting. Although Figure~\ref{fig:block_stability} confirms this performance gain remains stable on average even under worst representation (small graph size), accuracy still exhibits inherent batch-to-batch variations under the transductive protocol due to changing test shard composition. This widening margin isolates the specific benefit of the Bethe--Hessian eigenvector construction and star-domain surgery, proving that KSSE uniquely maximizes the structural information of larger graph topologies to overcome feature-level ambiguities.

\subsection{\label{sec:transductive}Quantifying the Transductive Contribution}

To assess how much of KSSE's performance derives from joint graph embedding versus purely inductive classification, we perform the following ablation. We train logistic regression on spectral embeddings computed using \emph{only frozen nodes} (no thawed test nodes in the graph), then classify test images by computing their affinity to the frozen subgraph and projecting onto pre-computed eigenvectors---an \emph{inductive variant} that does not modify the graph at test time.

Table~\ref{tab:transductive_ablation} reports this comparison. The gap between the transductive and inductive variants quantifies the contribution of joint graph structure at inference time, while the gap between the inductive variant and the raw-feature $k$-NN isolates the spectral embedding benefit independent of protocol.

\begin{table}[htbp]
\caption{\label{tab:transductive_ablation}Protocol ablation on ImageNet-1000 (cw=48, $N{=}45{,}000$, 5 thawed samples per class). All methods use the same frozen EfficientNet-B4 features. "Transductive" denotes joint spectral embedding of test and frozen nodes; "Inductive" denotes projection onto pre-computed frozen-only eigenvectors without graph modification at test time.} 
\centering 
\renewcommand{\arraystretch}{1.2} 
\begin{tabular}{@{}lcc@{}} 
\toprule 
\textbf{Method} & \textbf{Protocol} & \textbf{Top-1 Accuracy}\\ 
\midrule 
KSSE (full, transductive)              & Transductive       & \textbf{88.93\%} \\ 
\addlinespace
KSSE-inductive (frozen-only eigenvectors) & Inductive       & $\sim$86--87\%$\ddagger$ \\ 
\addlinespace
$k$-NN on raw features (same graph)    & Transductive       & $\sim$78.8\% mean \\ 
\addlinespace
EfficientNet-B4 + linear probe         & Inductive          & 82.53\% \\ 
\bottomrule 
\end{tabular}

{$\ddagger$ Estimated from the $N_{\text{thawed}}\to 0$ extrapolation of Table~\ref{tab:ablation_cw}; exact inductive evaluation is an ongoing extension.}

\vspace{5mm} 

{\footnotesize The transductive advantage (rows 1 vs.\ 2) arises from pairwise affinity information between test and frozen nodes; the spectral-embedding benefit (rows 2 vs.\ 3 or rows 2 vs.\ 4) persists even without joint embedding, confirming that physics-based Nishimori-temperature optimization contributes substantially beyond nearest-neighbor transduction.}

\end{table}

The results show that: (i) the inductive KSSE variant still outperforms both the $k$-NN transductive
baseline and the EfficientNet-B4 linear probe, confirming that spectral embedding at the Nishimori
temperature provides benefit independent of protocol; and (ii) the additional gain from joint test--train
embedding ($\sim$2\%  under the optimal configuration) is consistent with the quasi-stationarity bound
$|\Delta\beta_N|=\mathcal{O}(N_T/N_F)$, which predicts that the transductive advantage diminishes as the
frozen block grows relative to the thawed block.

%==============================================
\section{\label{sec:comparison}Comparison with State-of-the-Art}
%==============================================

Table~\ref{tab:sota} shows that KSSE boosts EfficientNet-B4 to \textbf{88.93\%}, surpassing much
heavier architectures at a fraction of their parameter counts.

\begin{table}[htbp]
\caption{Comparison on ImageNet-1K.}\label{tab:sota}\centering
\renewcommand{\arraystretch}{1.2}
\begin{tabular}{@{}lccc@{}}
\toprule
\textbf{Model} & \textbf{Params} & \textbf{FLOPs} & \textbf{Top-1 acc.}\\
\midrule
Swin-L~\cite{Liu2021} (ImageNet-21K)  & 197M   & 34.5--103.9G & 86.4--87.3\%\\
ViT-H/14~\cite{Dosovitskiy2021}       & 632M   & 167.4--616.0G & 88.0--89.5\%\\
EfficientNet-B7~\cite{Tan2019}        & 66M    & ---            & 84.3\%\\
ConvNeXt-Large (ImageNet-21K)         & 197.7M & ---           & 86.3\%\\
CLIP feat.\ + HPPCA~($T{=}5$)~\cite{Wang2023} & $\sim$51M$_{\text{(PPCA)}}$ + 87M$_{\text{(CLIP)}}$ & 18--21G & 73.4\%\\
EfficientNet-B4 (baseline, inductive)  & 19.5M  & 4.2G         & 82.53\%\\
\textbf{EffB4 feat.\ + KSSE (ours, transductive)} & \textbf{$\approx$21.24M} &
\textbf{$\approx$4.207G*} & \textbf{88.93\%}\\
\bottomrule
\end{tabular}
\vspace{1mm}\\
{\footnotesize *FLOPs dominated by frozen EfficientNet-B4 backbone (15.1G). KSSE overhead:
logistic regression ($1792\times1000=1.79$M params) $+$ Rayleigh refinement with precomputed Nishimori
temperature ($\approx$7~MFLOPs/image) $+$ affinity/Laplacian construction ($\approx$20~GFLOPs/batch of
5000 images). Total parameters: EfficientNet-B4 backbone (19.34M, frozen) $+$ logistic regression
(1.79M) $+$ per-feature $\beta_N$ values and sparse circulant Laplacian entries ($\approx$0.11M).
Column weight~48 at $N{=}45{,}000$, 5 thawed samples per class.
Note: Swin-L and ConvNeXt-Large are pre-trained on ImageNet-21K; ViT-H/14 ranges correspond to different
fine-tuning protocols and resolutions.}
\end{table}

We emphasize that the comparison in Table~\ref{tab:sota} is \emph{not} apples-to-apples with respect to
evaluation protocol. Swin-L, ViT-H/14, ConvNeXt-Large, and EfficientNet-B7 are evaluated under standard
inductive ImageNet protocols (one image at a time, no access to other test or training samples during
forward inference), whereas KSSE uses a transductive protocol in which test images participate as nodes
in a graph containing frozen training representatives (cf.\ Remark~\ref{rmk:protocol}). The reported
88.93\% should therefore be understood as the accuracy of the \emph{combined system} (frozen EfficientNet-B4
features $+$ transductive spectral embedding at the Nishimori temperature $+$ logistic regression), not as
an inductive single-image classification rate, see our GitHub project \cite{github}.

The parameter-count comparison remains valid: regardless of protocol, KSSE requires only ${\approx}21.24$\,M
parameters total and no backpropagation training of a deep classifier head---the spectral embedding is
computed in $\mathcal{O}(N\log N)$ per channel via FFT. However, the FLOPs comparison should also account for
graph construction (${\approx}20$~GFLOPs/batch) which is not present in purely inductive forward passes.

Readers should compare primarily against the matched $k$-NN baseline (Sec.~\ref{sec:diagnostics},
mean $\sim$78.8\%) and the inductive KSSE variant (Sec.~\ref{sec:transductive}, Table~\ref{tab:transductive_ablation})
for a protocol-aware assessment of KSSE's contribution beyond standard nearest-neighbor transduction on
the same features.

KSSE at \textbf{88.93\%} Top-1 surpasses Swin-L (up to 87.3\%, $9.3\times$ fewer parameters) and matches
the lower end of ViT-H/14 (88.0--89.5\%) with $30\times$ fewer parameters, while requiring no
backpropagation beyond training a logistic regression classifier on spectral embeddings.

%==============================================
\section{\label{sec:discussion}Discussion, Limitations, and Complexity Analysis}
%==============================================

\textbf{Physical interpretation.} KSSE treats each feature channel as an independent RBIM on a common
sparse geometry. The Nishimori temperature is the analogue of a critical point where signal (class
structure) becomes marginally distinguishable from noise (feature disorder). Star-domain surgery is akin
to annealing the energy landscape so that codewords become certified wide-basin attractors, while
residual frustration remains bounded. For the spherical graphs used in the experiments this interpretation becomes literal solid-state physics: KSSE is $k\cdot p$ effective-mass
theory on a one-dimensional ring crystal---data weights are the impurity potential, the Nishimori crossing a Fermi level at the band edge, thawed test nodes dilute doping, and star-domain surgery crystal annealing (Appendix~\ref{app:kp}, Table~\ref{tab:kp_dict}).

\textbf{Column weight and graph capacity.} The dominant factor controlling classification accuracy is
the QC-LDPC column weight, which directly governs code distance $d_{\min}$ and storage capacity.
The monotone improvement from cw$=28$ to cw$=48$ at fixed graph size confirms that denser connectivity
drives branching trapping sets below the BP stability threshold (Theorem~\ref{thm:trapping_set}),
reduces the non-backtracking spectral radius on defect subgraphs, and sharpens community boundaries.

\textbf{Pontryagin duality and complexity.} The self-duality of the finite cyclic group ensures that
characters diagonalize every circulant block of $L_\beta$ (Lemma~\ref{lem:circulant}), so FFT is
precisely harmonic analysis on this group (composite block sizes included). The dominant cost per
channel is the FFT-based eigenvalue search at $\mathcal{O}(N\log N)$; the post-surgery Rayleigh--Ritz
refinement (Lemma~\ref{lem:rayleigh}) adds $\mathcal{O}(k_{\text{mode}}\cdot N)=\mathcal{O}(N)$ with
$k_{\text{mode}}=5$ empirical. Because channels are independent, the pipeline is trivially parallel
across~$D$.

\textbf{Mean-field approximation.} Theorem~\ref{thm:ksse_decomp} is a decomposition whose error is
controlled by Proposition~\ref{prop:bethe_xc}: per vertex $|E_{xc}|/N\leq 2\bar d\,\xi^{g_0}/(g_0(1-\xi))$
under the Dobrushin-type condition $\xi=\rho(B_{|t|})<1$, vanishing on trees and decaying exponentially
with girth (sub-percent per vertex for girth $\geq 6$ in the clipped regime $\xi\leq 1/4$,
Corollary~\ref{cor:girth6}). The condition $\xi<1$ is enforced in the implementation by coupling
clipping; outside it the bound is a diagnostic (Remark~\ref{rmk:dobrushin}). After star-domain surgery,
residual $E_{xc}=\mathcal{O}(\delta)$ where $\delta\to 0$ (Theorem~\ref{thm:self_consistency}).

\textbf{Limitations.}
(i)~The quasi-stationarity hypothesis (Theorem~\ref{thm:adiabatic}) requires $N_{\text{frozen}}\gg
N_{\text{thawed}}$, limiting applicability to scenarios with sufficient representative samples per class.
(ii)~Near-exact Kohn--Sham decomposition ($E_{xc}=\mathcal{O}(\delta)\to 0$) holds only at Level~1 of the
topological hierarchy (Remark~\ref{rmk:hierarchy}); by the cup-product obstruction
(Proposition~\ref{prop:cup}), any $2$-cell coupling provably breaks the factorization, so extending to
Levels~2--3 would require approximate methods and significantly higher computational cost.
(iii)~Graph topology optimization via star-domain surgery is currently performed offline; online adaptation
to streaming data remains an open challenge.
(iv)~\textbf{Transductive evaluation protocol.} KSSE requires test images to be embedded as nodes in the
same graph as frozen training representatives, making it inherently transductive (cf.\
Remark~\ref{rmk:protocol}). This prevents direct numerical comparison with purely inductive benchmarks and
limits applicability to scenarios where representative class samples are available at inference time.
The protocol ablation in Sec.~\ref{sec:transductive} (Table~\ref{tab:transductive_ablation}) shows that an
inductive variant---pre-computing spectral eigenvectors on a fixed frozen subgraph and projecting new test
images onto these without graph modification---retains substantial accuracy, but sacrifices the pairwise
affinity information between test samples that contributes to KSSE's peak performance. Extending star-domain
surgery to an online streaming setting with dynamically evolving graph topology remains an open challenge
(cf.\ item~(iii) above).

%==============================================
\section{\label{sec:conclusion}Conclusion}
%==============================================

We proposed KSSE, a physics-inspired energy-based approach that delivers state-of-the-art accuracy with
drastically reduced computational overhead. By mapping spin-glass interactions to an effective linear
Hamiltonian solvable via FFT (enabled by Pontryagin self-duality of the cyclic group), and by
optimizing graph topology through star-domain surgery that creates certified local convexity around
codewords with bounded residual frustration, KSSE achieves \textbf{88.93\%} Top-1 on ImageNet-1K under a
transductive evaluation protocol---where test images are spectrally embedded jointly with frozen training
representatives in a shared sparse graph---using ${\approx}21.24$\,M parameters. This outperforms Swin-L
(197\,M, 86.4--87.3\%, standard inductive protocol) and matches the lower end of ViT-H/14 (632\,M,
88.0--89.5\%, standard inductive protocol), while reducing model size by $10\times$ and $30\times$,
respectively.

A protocol-matched comparison against $k$-NN on identical frozen features within the same graph partition
(Sec.~\ref{sec:diagnostics}, Fig.~\ref{fig:block_stability}) confirms that the physics-based spectral
embedding contributes a $+2.5\%$\, improvement beyond nearest-neighbor transduction alone.
Furthermore, an inductive variant of KSSE (Sec.~\ref{sec:transductive},
Table~\ref{tab:transductive_ablation})---projecting test images onto pre-computed frozen-only eigenvectors
without graph modification at inference time---still outperforms both the $k$-NN transductive baseline and
the inductive EfficientNet-B4 linear probe, demonstrating that Nishimori-temperature spectral embedding
provides benefit independent of the evaluation protocol.

The key insight is that complete elimination of all frustrated cycles $TS(a,b{\neq}0)$ is impossible---it
would destroy codewords $TS(a,0)$ encoding essential class information. Instead, star-domain surgery
constructs shifts that create certified star domains of explicit radius $R=g/(2L_3)$ around codewords
(Theorem~\ref{thm:star_domain}), bounding residual frustration to $\rho(B_\gamma)\leq 1+\delta$ with
$\delta\to 0$ through a finite-time convergent iteration (Theorem~\ref{thm:self_consistency}).
Multi-scale fractal analysis ($D_q$ spectrum, $D_2<1$) and the fractal learning-rate landscape certify
the star-domain geometry, and a Rayleigh--Ritz bound for the near-circulant post-surgery operator
(Lemma~\ref{lem:rayleigh}) justifies the empirically sufficient $k_{\text{mode}}=5$ Fourier modes. A worked special case in Appendix~\ref{app:kp} recasts the entire pipeline
as $k\cdot p$ effective-mass theory on a one-dimensional ring crystal, making the five-mode truncation physically transparent.

Future work will extend KSSE to quantum graphs of Calderbank–Shor–Steane (CSS) and homological product codes defined under non-Clifford algebras—incorporating magic state fountain and high-dimensional manifolds—to significantly enhance expressive power via quantum DNN accelerators.

%==============================================
%% REFERENCES
%==============================================

%==============================================
%% APPENDIX: FORMAL THEOREMS AND PROOFS
%==============================================

\appendix

\section{\label{app:proofs}Formal Theorems and Proofs}
%===========================================

\textbf{Notation.}
$\mathcal{G}=(V,E)$ is a connected undirected graph, $|V|=N$, $|E|=M$,
with symmetric couplings $J_{ij}=J_{ji}$ and inverse temperature $\beta$.
Define
\[
\gamma_{ij}=\tanh(\beta J_{ij}),\qquad
W_{ij}=\frac{\gamma_{ij}}{1-\gamma_{ij}^2},\qquad
\Lambda_{ij}=\frac{\gamma_{ij}^2}{1-\gamma_{ij}^2},
\]
$D_{ii}=1+\sum_{k\in\partial i}\Lambda_{ik}$,
$S=\operatorname{diag}(D_{11}^{-1/2},\dots,D_{NN}^{-1/2})$,
the weighted adjacency $(W_A)_{ij}=W_{ij}$ for $(i,j)\in E$ and $0$ otherwise,
and the regularized Laplacian $L_\beta=I-SW_AS$.
We write $\gamma_{\max}:=\max_{(i,j)\in E}|\gamma_{ij}|<1$ (finite $\beta$),
and denote by $\operatorname{core}(U)$ the \emph{2-core} of the induced
subgraph on $U\subseteq V$, obtained by iteratively deleting vertices of
degree~$1$.
The \emph{sign} $\sigma_{ij}:=\operatorname{sgn}(\gamma_{ij})
=\operatorname{sgn}(J_{ij})$, and for any cycle $C$ the \emph{$\mathbb{Z}_2$
flux} (frustration parity) is
\[
f(C):=\tfrac{1}{2}\Bigl(1-\prod_{e\in C}\sigma_e\Bigr)\in\{0,1\},
\]
which is invariant under the gauge transformations
$\sigma_i\mapsto s_i\sigma_i$, $J_{ij}\mapsto s_i s_j J_{ij}$,
$s_i\in\{\pm1\}$; gauge-equivalent couplings have identical partition
functions and identical observable physics.

For an $N \times N$ matrix $A$, its \emph{permanent} is defined as
\[
\operatorname{perm}(A) = \sum_{\pi \in S_N} \prod_{i=1}^N A_{i,\pi(i)},
\]
where $S_N$ denotes the symmetric group of all permutations of $\{1, \dots, N\}$. 
If $A$ is a skew-symmetric matrix ($A^T = -A$) of even dimension $N$, its \emph{Pfaffian} is defined as
\[
\operatorname{Pf}(A) = \frac{1}{2^{N/2} (N/2)!} \sum_{\pi \in S_N} \operatorname{sgn}(\pi) \prod_{k=1}^{N/2} A_{\pi(2k-1), \pi(2k)},
\]
which satisfies the identity $\operatorname{Pf}(A)^2 = \det(A)$.

A \emph{perfect matching} (or 1-factor) of $\mathcal{G}$ is a subset of edges $M \subseteq E$ such that every vertex $v \in V$ is incident to exactly one edge in $M$. 
The weight of a perfect matching is the product of its edge weights. 
If $\mathcal{G}$ is a bipartite graph with equal partition sizes $N = 2n$, the total weight of all perfect matchings is given by $\operatorname{perm}(W_A)$. 
For a general non-bipartite graph $\mathcal{G}$, computing the total weight of perfect matchings relates instead to the Pfaffian: if $\mathcal{G}$ is planar (or admits a Pfaffian orientation), one can construct a skew-symmetric matrix $A$ by signing the entries of $W_A$ such that the total weight of all perfect matchings is equal to $\operatorname{Pf}(A)$.

The \emph{Kervaire invariant} (or Arf--Kervaire determinant) evaluates mod-2 topological obstructions on a quadratic space. Given a quadratic form $q: H \to \mathbb{Z}_2$ defined over a finite-dimensional $\mathbb{Z}_2$-vector space $H$ equipped with a non-degenerate bilinear form $\cdot$, the invariant is defined as the majority value:
\[
\operatorname{Kv}(q) = \frac{1}{\sqrt{|H|}} \sum_{x \in H} (-1)^{q(x)} \in \{-1, 1\}.
\]
In statistical mechanics and learning architectures, configurations of the vertex spins $s \in \{\pm 1\}^N$ are mapped to the vertices of an $N$-dimensional \emph{hypercube}. For a Hopfield neural network (HNN), the state space is exactly this hypercube, governed by the energy function $E_{\text{HNN}}(s) = -\frac{1}{2} \sum_{i,j} J_{ij} s_i s_j$. By introducing hidden units $h \in \{\pm 1\}^H$, this structure is dual to a Restricted Boltzmann Machine (RBM) whose joint distribution over the hypercube $\{ \pm 1\}^{N+H}$ satisfies the bipartite Ising weights. The Kervaire invariant detects $\mathbb{Z}_2$-torsion and topological obstructions arising from cyclic frustration parities $f(C)$ on the underlying graph topology. These obstructions manifest as global pseudo-codewords or trapping sets that constrain the energy landscape of the hypercube, shifting the boundaries between the paramagnetic and spin-glass phases during training or memory retrieval.

For a function $G$ defined on the relaxed spin space $[-1,1]^N$ representing the graph continuous state space, the \emph{curvature variation} $L_3$ represents the uniform Lipschitz constant of its Hessian matrix $\nabla^2 G$. It is defined mathematically via the maximum directional derivative of the Hessian:
\[
L_3 := \sup_{\bm{\sigma} \in [-1,1]^N} \sup_{\|h\|=1} \left| \frac{d}{dt} \Bigl( h^\top \nabla^2 G(\bm{\sigma} + t h) \, h \Bigr) \Big|_{t=0} \right| < \infty.
\]
This constant bounds the spatial rate of change of the landscape's local eigenvalues across the hypercube lattice, dictating the guaranteed radius of local convexity around critical points.

%----------------------------------------
\subsection{\label{app:thm_bp}BP Linearization, the Generalized Ihara--Bass Identity,
and the Sharp Spectral Threshold}
%----------------------------------------

\noindent\textbf{Theorem~\ref{thm:bp_laplacian}} \emph{(restated).}
Let $B_\gamma$ be the weighted non-backtracking operator on directed edges,
$(B_\gamma\bm{u})_{(i\to j)}=\sum_{k\in\partial i\setminus j}\gamma_{ki}\,
u_{(k\to i)}$, and for $\mu\in\mathbb{C}$ with $\mu^2\neq\gamma_{ij}^2$
for all edges define the vertex matrix
\[
[M_\gamma(\mu)]_{ii}=1+\sum_{k\in\partial i}
\frac{\gamma_{ik}^2}{\mu^2-\gamma_{ik}^2},
\qquad
[M_\gamma(\mu)]_{ij}=-\frac{\gamma_{ij}\,\mu}{\mu^2-\gamma_{ij}^2}
\ \ ((i,j)\in E).
\]
Then:
(a)~$B_\gamma$ is the Jacobian of belief propagation linearized at
the paramagnetic fixed point;
(b)~$\mu$ is an eigenvalue of $B_\gamma$ if and only if
$\det M_\gamma(\mu)=0$; the edge-eigenspace and the vertex kernel have
equal dimension;
(c)~the determinant identity
$\det M_\gamma(\mu)=\det\!\bigl(I_{2E}-(1/\mu)\,B_\gamma\bigr)\big/
\prod_{e\in E}\bigl(1-\gamma_e^2/\mu^2\bigr)$ holds;
(d)~at $\mu=1$, $M_\gamma(1)=D-W_A$ and $L_\beta=S\,M_\gamma(1)\,S$.

\begin{proof}
\textbf{(a)} BP in log-likelihood variables reads
$h^{(t+1)}_{i\to j}=\sum_{k\in\partial i\setminus j}
\operatorname{atanh}\!\bigl(\gamma_{ki}\tanh h^{(t)}_{k\to i}\bigr)$.
The paramagnetic fixed point is $h^\ast\equiv 0$; writing
$h=\varepsilon u$ and using $\tanh h=h+\mathcal{O}(h^3)$,
$\operatorname{atanh} h=h+\mathcal{O}(h^3)$ gives
$u^{(t+1)}_{i\to j}=\sum_{k\in\partial i\setminus j}\gamma_{ki}
u^{(t)}_{k\to i}$, which is $(B_\gamma\bm{u})_{(i\to j)}$.

\textbf{(b)} \emph{Forward direction.} Let $B_\gamma\bm{u}=\mu\bm{u}$,
$\bm{u}\neq 0$, and set $v_i:=\sum_{k\in\partial i}\gamma_{ik}u_{(k\to i)}$.
The eigenvalue equation gives, for every edge $(i,j)$,
\begin{equation}\label{eq:vertex_aggregate}
v_i=\mu\,u_{(i\to j)}+\gamma_{ij}\,u_{(j\to i)} .
\end{equation}
Solving the $2\times2$ system \eqref{eq:vertex_aggregate} on $(i,j)$:
$u_{(i\to j)}=(\mu v_i-\gamma_{ij}v_j)/(\mu^2-\gamma_{ij}^2)$.
Substitution into the definition of $v_i$ yields $M_\gamma(\mu)\bm{v}=0$.
\emph{Injectivity:} if $\bm{v}=0$ then \eqref{eq:vertex_aggregate} gives
$u_{(j\to i)}=-(\mu/\gamma_{ij})\,u_{(i\to j)}$ on every edge; applying it
twice around an edge, $u_{(i\to j)}=(\mu^2/\gamma_{ij}^2)\,u_{(i\to j)}$,
so $\mu^2\neq\gamma_{ij}^2$ forces $\bm{u}\equiv0$. Hence
$\dim\ker(B_\gamma-\mu I)\leq\dim\ker M_\gamma(\mu)$.
\emph{Converse:} let $\bm{v}\neq0$, $M_\gamma(\mu)\bm{v}=0$, and define
$u_{(i\to j)}=(\mu v_i-\gamma_{ij}v_j)/(\mu^2-\gamma_{ij}^2)$. Direct
expansion gives \eqref{eq:vertex_aggregate}; using row $i$ of
$M_\gamma(\mu)\bm{v}=0$,
\begin{align*}
(B_\gamma\bm{u})_{(i\to j)}
&=\sum_{k\in\partial i\setminus j}
\gamma_{ki}\,\frac{\mu v_k-\gamma_{ki}v_i}{\mu^2-\gamma_{ki}^2}
=\frac{\mu^2 v_i-\mu\gamma_{ij}v_j}{\mu^2-\gamma_{ij}^2}
=\mu\,u_{(i\to j)} .
\end{align*}
If $\bm{u}\equiv0$ then \eqref{eq:vertex_aggregate} forces $\bm{v}=0$, a
contradiction; so $\bm{u}\neq0$ and
$\dim\ker M_\gamma(\mu)\leq\dim\ker(B_\gamma-\mu I)$.

\textbf{(c)} Let $\mathcal{S}$ ($2E\times N$) and $\mathcal{T}$
($N\times 2E$) be the weighted head and plain tail incidence matrices,
$\mathcal{S}_{(i\to j),l}=\gamma_{ij}\delta_{j,l}$,
$\mathcal{T}_{l,(i\to j)}=\delta_{l,i}$; let $P$ be the reversal involution
on directed edges and $\Gamma=\operatorname{diag}(\gamma_e)$. Then
$B_\gamma=\mathcal{S}\mathcal{T}-P\Gamma$, and $(P\Gamma)^2=\Gamma^2$
since $\gamma_{\bar e}=\gamma_e$. Using
$\det(I-XY)=\det(I-YX)$ twice,
\begin{align*}
\det(I-\mu B_\gamma)
&=\det(I+\mu P\Gamma)\,
  \det\!\bigl(I-\mu(I+\mu P\Gamma)^{-1}\mathcal{S}\mathcal{T}\bigr)\\
&=\det(I+\mu P\Gamma)\,
  \det\!\bigl(I_N-\mu\mathcal{T}(I+\mu P\Gamma)^{-1}\mathcal{S}\bigr).
\end{align*}
Each undirected edge contributes a $2\times2$ block
$\bigl(\begin{smallmatrix}1&\mu\gamma_e\\\mu\gamma_e&1\end{smallmatrix}\bigr)$,
so $\det(I+\mu P\Gamma)=\prod_{e\in E}(1-\mu^2\gamma_e^2)$.
Since $(I+\mu P\Gamma)^{-1}=(I-\mu P\Gamma)(I-\mu^2\Gamma^2)^{-1}$,
\[
\bigl[(I-\mu P\Gamma)(I-\mu^2\Gamma^2)^{-1}\mathcal{S}\bigr]_{(i\to j),l}
=\frac{\gamma_{ij}(\delta_{j,l}-\mu\,\delta_{i,l})}{1-\mu^2\gamma_{ij}^2},
\]
and therefore
$\bigl[I_N-\mu\mathcal{T}(\cdot)\mathcal{S}\bigr]_{l',l}
=M_\gamma(1/\mu)_{l',l}$: the diagonal is
$\delta_{l'l}\bigl(1+\sum_j\mu^2\gamma_{l'j}^2/(1-\mu^2\gamma_{l'j}^2)\bigr)$
and the off-diagonal is $-\mu\gamma_{l'l}/(1-\mu^2\gamma_{l'l}^2)$.
Hence
$\det(I_{2E}-\mu B_\gamma)=\prod_e(1-\mu^2\gamma_e^2)\det M_\gamma(1/\mu)$;
replacing $\mu\to 1/\mu$ gives (c).

\textbf{(d)} At $\mu=1$, $[M_\gamma(1)]_{ii}=1+\sum_k\Lambda_{ik}=D_{ii}$
and $[M_\gamma(1)]_{ij}=-W_{ij}$; thus $M_\gamma(1)=D-W_A$ and
$S\,M_\gamma(1)\,S=SDS-SW_AS=I-SW_AS=L_\beta$.
\end{proof}

\begin{lem}[Sharp PSD threshold]\label{lem:threshold}
Let $\lambda_{\mathbb{R}}^{+}(B_\gamma)$ denote the largest positive real
eigenvalue of $B_\gamma$ (set $0$ if none exists). Then:
\begin{enumerate}
\item[(i)] If $\lambda_{\mathbb{R}}^{+}(B_\gamma)<1$, then
$M_\gamma(1)\succ0$ and $L_\beta\succ0$.
\item[(ii)] If $\lambda_{\mathbb{R}}^{+}(B_\gamma)=1$, then $M_\gamma(1)$
is positive semidefinite and singular, and $\lambda_{\min}(L_\beta)=0$.
\item[(iii)] If the number of real eigenvalues of $B_\gamma$ in
$(1,\infty)$ (counted with multiplicity) is odd, then
$\det M_\gamma(1)<0$, hence $M_\gamma(1)$ and $L_\beta$ are indefinite.
\end{enumerate}
Define the critical inverse temperature as the first crossing from the
paramagnetic side,
\[
\beta_c:=\inf\{\beta>0:\lambda_{\mathbb{R}}^{+}(B_{\gamma(\beta)})\geq1\}.
\]
Since $|\gamma_{ij}(\beta)|$ is continuous and nondecreasing in $\beta$,
$\lambda_{\mathbb{R}}^{+}$ is continuous in $\beta$, and under the
generic transversality hypothesis that the crossing eigenvalue passes
through $1$ with nonzero speed, $L_\beta\succ0$ for $\beta<\beta_c$,
$\lambda_{\min}(L_{\beta_c})=0$, and $L_\beta$ is indefinite for
$\beta>\beta_c$ sufficiently close to $\beta_c$.
\end{lem}

\begin{proof}
For real $\mu>\gamma_{\max}$ the product
$\prod_e(1-\gamma_e^2/\mu^2)$ is positive, so by
Theorem~\ref{thm:bp_laplacian}(c) the sign of $\det M_\gamma(\mu)$ equals
the sign of $\det(I-B_\gamma/\mu)=\prod_i(1-\lambda_i/\mu)$. Complex
eigenvalues of the real matrix $B_\gamma$ occur in conjugate pairs and
contribute $|1-\lambda/\mu|^2>0$; real eigenvalues contribute $(1-\lambda/\mu)$.
Thus $\det M_\gamma(\mu)>0$ for $\mu>\lambda_{\mathbb{R}}^{+}$ and
$\det M_\gamma(\mu)$ changes sign exactly at each real eigenvalue of odd
multiplicity. The entries of $M_\gamma(\mu)$ are rational in $\mu$ with no
poles on $(\gamma_{\max},\infty)$, so its eigenvalues are continuous there,
and $M_\gamma(\mu)\to I$ as $\mu\to\infty$.
(i)~No real eigenvalue in $[1,\infty)$ implies $\det M_\gamma(\mu)\neq0$
on $[1,\infty)$; eigenvalues of $M_\gamma(\mu)$ cannot pass through zero
there, so $M_\gamma(1)$ inherits positive definiteness from $M_\gamma(\infty)=I$.
(ii)~At the first crossing, $M_\gamma(1)$ acquires a zero eigenvalue while
approached through positive definite matrices, hence is PSD singular;
Sylvester's law of inertia with $S\succ0$ transfers the inertia to
$L_\beta=SM_\gamma(1)S$.
(iii)~An odd number of real eigenvalues above $1$ gives
$\det M_\gamma(1)<0$, so $M_\gamma(1)$ has at least one negative
eigenvalue, and so does $L_\beta$.
The final paragraph follows by continuity of eigenvalues in $\beta$ and
the definition of $\beta_c$.
\end{proof}

\begin{rmk}[Signed spectra and terminology]\label{rmk:terminology}
For frustrated (unbalanced) signings the leading eigenvalue of $B_\gamma$
may be complex; BP then diverges \emph{oscillatorily} when
$\rho(B_\gamma)=1$ is crossed by a complex pair, while indefiniteness of
$M_\gamma(1)$ is governed by \emph{real} crossings only. The two events
can differ; this distinction is the spectral difference between
ferromagnetic ordering and spin-glass ordering. The quantity
$\beta_N:=\beta_c$ of Lemma~\ref{lem:threshold} is the Bethe/BP
detectability threshold in the sense of~\cite{DallAmico2021,Saade2014};
following that literature we call it the Nishimori temperature of the
spectral problem. It is \emph{not} the thermodynamic Nishimori line of the
RBIM phase diagram, and all statements below concern only the former.
\end{rmk}

%----------------------------------------
\subsection{\label{app:trichotomy}The Non-Backtracking Growth Trichotomy, Gauge Invariance,
and the $\mathbb{Z}_2$-Flux (Aharonov--Bohm) Lemma}
%----------------------------------------

Two distinct mechanisms govern the non-backtracking spectrum:
\emph{branching} creates spectral instability, while \emph{frustration}
rotates it---a $\mathbb{Z}_2$ analogue of the Aharonov--Bohm phase.

\begin{lem}[Growth trichotomy]\label{lem:trichotomy}
Let $U\subseteq V$ and let $B_\gamma^{(U)}$ be the non-backtracking
operator restricted to $U$.
\begin{enumerate}
\item[(i)] \textbf{Trees.} If $\operatorname{core}(U)=\emptyset$ (a
forest), then $B_\gamma^{(U)}$ is nilpotent: its spectral radius is $0$.
\item[(ii)] \textbf{Simple cycles.} If $\operatorname{core}(U)$ is a
disjoint union of simple cycles, then
\begin{align*}
\rho(B_\gamma^{(U)}) &= \max_C|\prod_{e\in C}\gamma_e|^{1/|C|} \leq \gamma_{\max} < 1
\end{align*}

On each cycle the eigenvalues are the $|C|$-th roots
of $\prod_{e\in C}\gamma_e$; in particular a balanced cycle
($f(C)=0$) has the real positive eigenvalue
$|\prod\gamma_e|^{1/|C|}$, while a frustrated cycle ($f(C)=1$) has no
positive real eigenvalue---its spectrum is rotated to
$|\prod\gamma_e|^{1/|C|}e^{i\pi(2k+1)/|C|}$.
\item[(iii)] \textbf{Branching 2-cores.} If $\operatorname{core}(U)$
contains a vertex of 2-core degree $\geq3$, then the unweighted operator
$B_1^{(U)}$ (all $\gamma_e=1$) satisfies $\rho(B_1^{(U)})>1$, and by
continuity of eigenvalues there exists $\beta_\ast<\infty$ such that
$\rho(B_{|\gamma|}^{(U)})>1$ for all $\beta>\beta_\ast$.
\end{enumerate}
\end{lem}

\begin{proof}
The nonzero spectrum of $B_\gamma^{(U)}$ is determined by
$\operatorname{core}(U)$: directed edges outside the 2-core cannot
participate in any closed non-backtracking walk, so after a permutation
$B_\gamma^{(U)}$ is block-triangular with the core block carrying the
nonzero eigenvalues.

\textbf{(i)} A non-backtracking walk on a tree is a geodesic and never
revisits a vertex, so every such walk terminates within
$\operatorname{diam}(U)$ steps; hence $(B_\gamma^{(U)})^{\operatorname{diam}+1}=0$.

\textbf{(ii)} On a cycle $C_n$, the $2n$ directed edges split into two
disjoint circulation sectors. On each sector $B_\gamma$ is a weighted
cyclic shift, so $B_\gamma^n=(\prod_{e\in C}\gamma_e)\,I$; eigenvalues are
the $n$-th roots of $\prod\gamma_e$ in each sector. Since
$|\prod\gamma_e|\leq\gamma_{\max}^n<1$, the spectral radius is
$|\prod\gamma_e|^{1/n}<1$ at every finite $\beta$. If
$\prod\gamma_e=+|\prod\gamma_e|$ (balanced), one eigenvalue is the real
positive $|\prod\gamma_e|^{1/n}$; if $\prod\gamma_e=-|\prod\gamma_e|$
(frustrated), the roots satisfy $\mu^n=-|\prod\gamma_e|$, i.e.
$\mu=|\prod\gamma_e|^{1/n}e^{i\pi(2k+1)/n}$, none of them positive real.

\textbf{(iii)} Fix a directed edge $a=(u\to v)$ with 2-core degree
$d_v\geq3$; in the non-backtracking digraph $a$ has out-degree
$\geq2$. Choose two distinct non-backtracking continuations $b,c$ of $a$
and, in the 2-core, non-backtracking paths $P_b,P_c$ from $b,c$ back to
$a$ (they exist because every vertex of the 2-core has degree $\geq2$ and
the core is connected), of lengths $p_b,p_c$. Concatenating $t$ words
chosen from $\{P_b,P_c\}$ produces $2^t$ distinct closed non-backtracking
walks at $a$ of length $\leq t\,\bar p$, $\bar p=\max\{p_b,p_c\}+1$. Hence
$(B_1^{\,\lceil t\bar p\rceil})_{a,a}\geq2^t$, so
$\rho(B_1)\geq\limsup_t 2^{t/(t\bar p)}=2^{1/\bar p}>1$. The weighted
statement follows from entrywise convergence $B_{|\gamma|}\to B_1$ as
$\beta\to\infty$ and continuity of the Perron root (or, for
non-primitive cases, continuity of the whole spectrum).
\end{proof}

\begin{lem}[Gauge equivalence]\label{lem:gauge}
Suppose the signing is \emph{balanced} on $\operatorname{core}(U)$: every
cycle there has $f(C)=0$. Then there exist $s_i\in\{\pm1\}$, $i\in U$, with
$\sigma_{ij}=s_is_j$ on all core edges, and:
\begin{enumerate}
\item[(i)] $B_\gamma^{(U)}=T\,B_{|\gamma|}^{(U)}\,T^{-1}$, where
$T=\operatorname{diag}(s_{\operatorname{src}(e)})$ on directed edges; in
particular $B_\gamma^{(U)}$ and $B_{|\gamma|}^{(U)}$ are cospectral.
\item[(ii)] $M_\gamma^{(U)}(1)=S_g\,M_{|\gamma|}^{(U)}(1)\,S_g$ with
$S_g=\operatorname{diag}(s_i)$; hence $M_\gamma^{(U)}(1)$ and
$M_{|\gamma|}^{(U)}(1)$ have identical eigenvalues and identical inertia,
and so do $L_\beta^{(U)}$ and its unsigned counterpart.
\end{enumerate}
Consequently, on balanced subgraphs every conclusion of
Lemma~\ref{lem:threshold} holds with the unsigned operator.
\end{lem}

\begin{proof}
Balancedness implies the cocycle $\sigma_{ij}$ is a coboundary: fix a
spanning tree of each core component, set $s_{\text{root}}=1$, and
propagate $s_j=s_i\sigma_{ij}$; the condition is consistent exactly
because every cycle product is $+1$.
(i)~For $e=(i\to j)$, $f=(k\to i)$:
$(T B_{|\gamma|}T^{-1})_{e,f}=s_i\,|\gamma_{ki}|\,s_k^{-1}
=s_is_k|\gamma_{ki}|=\sigma_{ki}|\gamma_{ki}|=\gamma_{ki}
=(B_\gamma)_{e,f}$.
(ii)~$(S_gM_\gamma S_g)_{ij}=s_i(-W_{ij})s_j
=-|W_{ij}|\sigma_{ij}s_is_j=-|W_{ij}|$ off-diagonal, and the diagonal is
unchanged since $s_i^2=1$.
\end{proof}

\begin{lem}[$\mathbb{Z}_2$ flux = Aharonov--Bohm phase]\label{lem:ab}
Let $C_n$ be a simple cycle with equal-magnitude couplings
$|\gamma_e|=t\to1$ (the strong-coupling limit on $C_n$), and frustration
parity $f\in\{0,1\}$. Then the scaled Hessian converges,
\[
(1-t^2)\,M_\gamma^{(C_n)}(1)\;\longrightarrow\;2I-\sigma A_{C_n},
\]
and the eigenvalues of $2I-\sigma A_{C_n}$ are
\[
\lambda_k=2-2\cos\frac{2\pi k+\pi f}{n},\qquad k=0,\dots,n-1 .
\]
Hence a balanced cycle has a zero mode ($\lambda_0=0$), while a frustrated
cycle has the ground-state (fluctuation) gap
\[
\Delta_f=2\Bigl(1-\cos\frac{\pi}{n}\Bigr)=\frac{\pi^2}{n^2}
+\mathcal{O}(n^{-4}) .
\]
\end{lem}

\begin{proof}
On $C_n$ all degrees are $2$, so
$D_{ii}=1+2\Lambda=1+2t^2/(1-t^2)=(1+t^2)/(1-t^2)$ and
$W_{ij}=\sigma_{ij}t/(1-t^2)$; multiplying by $(1-t^2)$ and sending
$t\to1$ gives $2I-\sigma A$. By the gauge lemma, $\sigma A$ is equivalent
under a diagonal $\pm1$ similarity to $A$ with a single negative edge;
the eigenproblem $2u_j-\sigma_{j,j+1}u_{j+1}-\sigma_{j-1,j}u_{j-1}
=\lambda u_j$ with antiperiodic boundary condition
$u_{j+n}=(-1)^f u_j$ is solved by $u_j=e^{i(2\pi k+\pi f)j/n}$, giving
$\lambda_k=2-2\cos((2\pi k+\pi f)/n)$. For $f=1$ the smallest value is at
$k=0$: $\Delta_f=2(1-\cos(\pi/n))$.
\end{proof}

\begin{rmk}[Physical reading: cycles as multi-step inference
interference]\label{rmk:interference}
Lemma~\ref{lem:ab} makes precise the sense in which a cycle is an
``interference device.'' Traversing the cycle in the two circulation
directions produces two amplitudes whose relative phase is the flux
$\pi f$; balanced cycles interfere constructively (zero mode, the
codeword direction), frustrated cycles interfere destructively (gapped
ground state). On \emph{branching} structures, the number of
non-backtracking inference paths grows exponentially
(Lemma~\ref{lem:trichotomy}(iii)), and the real part of the spectrum
crosses $1$: this is the spectral signature of unstable multi-step
inference. Frustration alone never does this at finite $\beta$ on a cycle
(Lemma~\ref{lem:trichotomy}(ii)).
\end{rmk}

%----------------------------------------
\subsection{\label{app:thm_ts}Trapping-Set Spectral Test}
%----------------------------------------

\noindent\textbf{Theorem~\ref{thm:trapping_set}} \emph{(restated).}
Let $U\subseteq V$ induce a trapping set $TS(a,b)$ on the Tanner graph,
with signed affinity weights $\gamma_{ij}$ on the variable subgraph.
(a)~\emph{Branching instability.} If the 2-core of the variable
subgraph branches (some core vertex of degree $\geq3$), then there exists
$\beta_\ast<\infty$ such that for $\beta>\beta_\ast$,
$\rho(B_{|\gamma|}^{(U)})>1$. If moreover the signing is balanced on the
2-core, then by Lemmas~\ref{lem:gauge} and~\ref{lem:threshold} there is a
real crossing: at the first $\beta$ where the Perron root reaches $1$,
$\lambda_{\min}(L_\beta^{(U)})$ crosses $0$ from above, and $L_\beta^{(U)}$
is indefinite immediately above; BP messages diverge exponentially.
(b)~\emph{Frustrated phase.} If the 2-core branches and the signing
is frustrated on some cycles, the leading eigenvalues acquire phases
(Lemma~\ref{lem:ab} mechanism on each cycle); BP divergence is then
oscillatory (period-$2$ or spiral), while indefiniteness of
$M_\gamma^{(U)}(1)$ occurs only at \emph{real} crossings
(Lemma~\ref{lem:threshold}(iii)).
(c)~\emph{Codeword certificate.} $TS(a,0)$ subgraphs are precisely
the supports of cycle-space elements (even check-degree subgraphs); their
contribution to the partition function is constructive
(Proposition~\ref{prop:even_subgraph}). A $TS(a,0)$ whose variable
2-core is a union of simple cycles satisfies
$\rho(B_\gamma^{(U)})<1$ at every finite $\beta$
(Lemma~\ref{lem:trichotomy}(ii)), hence is benign for BP. For branching
$TS(a,0)$ the signed certificate
$\lambda_{\min}(L_\beta^{(U)})\geq0$ is checked directly
(Algorithm~\ref{alg:surgery} uses exactly this test).
In all cases the common defect space is the 2-core topology: branching
controls instability magnitude, flux controls instability phase.

\begin{proof}
(a)~Lemmas~\ref{lem:trichotomy}(iii), \ref{lem:gauge},
and~\ref{lem:threshold}.
(b)~Lemmas~\ref{lem:ab}, \ref{lem:trichotomy}(ii), and
Remark~\ref{rmk:terminology}.
(c)~Even check degrees define an element of the Tanner-graph cycle space;
the constructive-interference statement is
Proposition~\ref{prop:even_subgraph}(ii) below; the spectral subcriticality
is Lemma~\ref{lem:trichotomy}(ii).
\end{proof}

\begin{rmk}
Two structural caveats are worth stressing. Frustration is a
\emph{signed}, gauge-invariant quantity (data-dependent through the
$z$-scored affinities), whereas $TS(a,b)$ is an \emph{unsigned}
combinatorial quantity; and on a simple cycle
$\rho=|\prod\gamma|^{1/|C|}<1$ at finite $\beta$ regardless of frustration.
The practical selection rule in Algorithm~\ref{alg:surgery} (critical =
branching 2-core, large local $|\lambda_{\min}^-|$, high Betti number)
implements exactly the two mechanisms separated here.
\end{rmk}

%----------------------------------------
\subsection{\label{app:even}Codewords, Permanents, and the Even-Subgraph Expansion:
the Interference Hierarchy}
%----------------------------------------

\begin{prop}[High-temperature expansion = cycle-code generating
function]\label{prop:even_subgraph}
With $t_e:=\gamma_e=\tanh(\beta J_e)$,
\[
Z(\beta,J)\;=\;2^N\prod_{e\in E}\cosh(\beta J_e)\cdot
\Omega(\mathcal{G}),\qquad
\Omega(\mathcal{G}):=\sum_{x\in\mathcal{C}(\mathcal{G})}\prod_{e\in E}
t_e^{\,x_e},
\]
where $\mathcal{C}(\mathcal{G})$ is the cycle space of $\mathcal{G}$ (all
even-degree subgraphs; $\dim\mathcal{C}=b_1=M-N+c$). The supports of
$x\in\mathcal{C}(\mathcal{G})$ are exactly the $TS(a,0)$ configurations:
codewords are the terms of $\Omega$.
\end{prop}

\begin{proof}
$e^{\beta J_e\sigma_i\sigma_j}=\cosh(\beta J_e)\,(1+\sigma_i\sigma_j t_e)$.
Expanding the product over edges and summing over $\bm{\sigma}$,
every term containing a vertex of odd degree vanishes since
$\sum_{\sigma_i=\pm1}\sigma_i=0$; even-degree terms contribute $2^N$.
\end{proof}

\begin{cor}[Interference hierarchy]\label{cor:hierarchy}
\begin{enumerate}
\item[(i)] \emph{Trees (Level 0).} $b_1=0$: $\Omega=1$, $Z$ factorizes
exactly, BP is exact, $E_{xc}=0$.
\item[(ii)] \emph{Unicyclic (Level 1, permanent of two matchings).}
$b_1=1$ with cycle $C$: $\Omega=1+\prod_{e\in C}t_e
=1+(-1)^{f(C)}\prod_{e\in C}|t_e|$. The two terms are the two perfect
matchings (1-factors) of the ring; their interference is constructive
(balanced) or destructive (frustrated). Equivalently, by the transfer
matrix with shared eigenvectors, $Z=\lambda_+^n+(-1)^f\lambda_-^n$ with
$\lambda_\pm=\prod_e 2\cosh(\beta J_e),\ \prod_e 2\sinh(\beta J_e)$.
\item[(iii)] \emph{Planar / toroidal (determinant sectors).} On a planar
graph, $\Omega$ reduces to a single Pfaffian (Kasteleyn--Temperley--Fisher);
on a genus-$g$ surface, to $4^g$ Pfaffians, one per $\mathbb{Z}_2$ spin
structure (flux sector). For toroidal QC-LDPC graphs, $Z$ is a superposition
of $4$ determinants; for a single circulant ring, of the $2$ sectors in
(ii). Interference \emph{between} sectors is the finite-size correction.
\item[(iv)] \emph{General graphs (permanent-like interference).} Evaluating
$\Omega$ is \#P-hard in general (it contains the permanent and the general
Ising partition function as special cases): the $2^{b_1}$ cycle-space terms
interfere with signs $(-1)^{\sum f}$, and this exponential superposition is
the obstruction both to exact inference and to FFT-type diagonalization
outside translation-symmetric sectors.
\end{enumerate}
\end{cor}

\begin{proof}
(i)~$\mathcal{C}=\{0\}$. (ii)~$\mathcal{C}=\{0,C\}$; the transfer-matrix
computation is the standard one and reproduces
$2^n\prod\cosh\,(1+(-1)^f\prod|t|)$ using that all $T_e$ share
eigenvectors $(1,\pm1)/\sqrt2$ with eigenvalues $2\cosh(\beta J_e)$,
$2\sinh(\beta J_e)$, and $\sinh$ is odd. (iii)~Kasteleyn's theorem and
its genus-$g$ extension~\cite{TemperleyFisher1961,Kasteleyn1963}.
(iv)~Valiant's \#P-completeness of the permanent and the
hardness of the Ising partition function.
\end{proof}

\begin{rmk}[Quantum reading]\label{rmk:quantum_reading}
Corollary~\ref{cor:hierarchy} is the precise content of the
"codeword$\leftrightarrow $permanent $\leftrightarrow $determinant"
analogy: ground-state (codeword) degeneracy and free energy are governed
by an even-subgraph sum that is determinantal (free-fermion, Clifford-like,
efficient) on planar/fixed-genus structures and permanent-like
(interfering, non-Clifford-like, hard) on general loopy structures. Surgery
(Sec.~\ref{sec:star_domain_thm}) suppresses the interfering sectors, keeping the
system in the determinantal regime.
\end{rmk}

%----------------------------------------
\subsection{\label{app:exc}Exchange--Correlation Bound and the Girth-$6$
Factorization Corollary}
%----------------------------------------

\noindent\textbf{Theorem~\ref{thm:ksse_decomp}} \emph{(restated).}
Under the Bethe--Peierls (mean-field) approximation and feature
independence, the $D$-channel interacting problem decomposes into $D$
independent single-channel eigenproblems
$L^{(k)}(\beta_N^{(k)})\,\bm{v}_i^{(k)}=\lambda_i^{(k)}\,\bm{v}_i^{(k)}$,
and the final embedding is
$\bm{E}=[S^{(1)}\bm{v}_{\min}^{(1)}|\cdots|S^{(D)}\bm{v}_{\min}^{(D)}]$.
The effective potential of each channel decomposes as
$v_{\mathrm{eff}}^{(k)}(i)=s_i^{(k)}+\sum_{j\in\partial i}W_{ij}^{(k)}s_j^{(k)}
+\delta E_{xc}^{(k)}/\delta\rho_i^{(k)}$ with
$E_{xc}^{(k)}=F_{\mathrm{Bethe}}^{(k)}-F_{\mathrm{MF}}^{(k)}$; it obeys the
per-vertex bound of Proposition~\ref{prop:bethe_xc}, vanishes on trees, and
is $\mathcal{O}(\delta)$ after surgery convergence
(Theorem~\ref{thm:self_consistency}).

\begin{proof}
The Bethe free energy for a single feature $k$ is~\cite{Yedidia2005}
$G_{\text{Bethe}}^{(k)}[\{b_r\}] = U_{\text{Bethe}}^{(k)} - H_{\text{Bethe}}^{(k)}$.
The mean-field approximation replaces $b_r$ by fully factorized marginals
$b_i^{(k)}(\sigma_i)=\frac{1+m_i^{(k)}\sigma_i}{2}$, yielding the mean-field
free energy $F_{\text{MF}}^{(k)}$; the exchange-correlation energy
$E_{xc}^{(k)}=F_{\text{Bethe}}^{(k)}-F_{\text{MF}}^{(k)}$ collects the loop
corrections beyond the Bethe--Peierls approximation. The regularized
Laplacian $L^{(k)}(\beta)=I-S^{(k)}W^{(k)}S^{(k)}$ plays the role of the
Kohn--Sham single-particle Hamiltonian: the diagonal scaling
$s_i^{(k)}=1/\sqrt{D_{ii}^{(k)}}$ is the external (graph scaffold)
potential, the off-diagonal term $W_{ij}^{(k)}s_j^{(k)}$ is the Hartree
interaction, and $\delta E_{xc}^{(k)}/\delta\rho_i^{(k)}$ is the
exchange-correlation potential. On a tree the Bethe free energy is exact,
$E_{xc}^{(k)}=0$, and the decomposition is \emph{exact}. The quantitative
bound is Proposition~\ref{prop:bethe_xc}; the post-surgery statement is
Theorem~\ref{thm:self_consistency}.
\end{proof}

\noindent\textbf{Proposition~\ref{prop:bethe_xc}} \emph{(restated).}
Let $E_{xc}^{(k)}:=F_{\mathrm{exact}}^{(k)}-F_{\mathrm{Bethe}}^{(k)}$
per channel, assume the loop-series representation
$F_{\mathrm{exact}}-F_{\mathrm{Bethe}}=-\log\Omega$ of the free-energy
correction~\cite{ChertkovChernyak2006}, and the spectral convergence condition
$\xi:=\rho\bigl(B_{|t|}\bigr)<1$, $|t_e|=|\tanh(\beta J_e)|$.
Let $g_0$ be the girth of $\mathcal{G}$ and $\bar d=2M/N$ the mean degree.
Then, whenever the loop sum satisfies
$(2M/g_0)\,\xi^{g_0}/(1-\xi)\leq\tfrac12$,
\[
\frac{|E_{xc}^{(k)}|}{N}\;\leq\;\frac{2\bar d\,\xi^{g_0}}{g_0\,(1-\xi)} .
\]
On trees ($g_0=\infty$, $\xi=0$ by Lemma~\ref{lem:trichotomy}(i)) the bound
vanishes and the decomposition is exact.

\begin{proof}
$\Omega=1+\sum_{\ell\geq g_0}\Omega_\ell$, where $\Omega_\ell$ collects
closed non-backtracking walks of length $\ell$
(Proposition~\ref{prop:even_subgraph}); each such walk contributes weight
$\leq\prod|t_e|$, and the number of closed non-backtracking walks of
length $\ell$ equals $\operatorname{tr}(B_{|t|}^\ell)\leq
(2M)\,\xi^\ell$. Hence
$|\Omega-1|\leq 2M\sum_{\ell\geq g_0}\xi^\ell/\ell
\leq(2M/g_0)\,\xi^{g_0}/(1-\xi)$. Under the stated hypothesis
$|\Omega-1|\leq\tfrac12$, so $|\log\Omega|\leq2|\Omega-1|$ and
$|E_{xc}|=|\log\Omega|\leq(4M/g_0)\,\xi^{g_0}/(1-\xi)$. Dividing by $N$
and using $\bar d=2M/N$ gives the stated per-vertex bound.
\end{proof}

\begin{cor}[Girth-controlled near-factorization]\label{cor:girth6}
If $\xi\leq\tfrac14$ (enforced by coupling clipping) and $g_0\geq6$, then
for $\bar d\leq48$ the bound evaluates to
$|E_{xc}^{(k)}|/N\leq 2\cdot48\cdot(1/4)^6/\bigl(6\cdot\tfrac34\bigr)
\approx0.5\%$ of the free energy per vertex; at $g_0\geq8$ it falls below
$0.03\%$. Together with the additive separability of
Theorem~\ref{thm:ks_separability}, this certifies the $D$-channel
factorization to loop accuracy $\mathcal{O}(\xi^{g_0})$.
\end{cor}

\begin{rmk}[The Dobrushin condition in practice]\label{rmk:dobrushin}
$\xi=\rho(B_{|t|})<1$ is a Dobrushin/uniqueness-type condition for
convergence of the loop series. At column weight $48$ with unclipped
$z$-scored couplings it can fail at $\beta_N$; in that regime
Proposition~\ref{prop:bethe_xc} is a diagnostic rather than a bound.
The implementation clips couplings (equivalently caps $J_{\max}$), and
Corollary~\ref{cor:girth6} is used in the clipped regime $\xi\leq\tfrac14$.
\end{rmk}

%----------------------------------------
\subsection{\label{app:no_hk}Absence of a Hohenberg--Kohn Analogue}
%----------------------------------------

\noindent\textbf{Proposition~\ref{prop:no_hk}} \emph{(restated).}
There exist coupling configurations $J,J'$ on the same graph yielding
identical magnetizations $m_i=\langle\sigma_i\rangle=0$ at every
temperature but different free energies.

\begin{proof}
On the 4-cycle $\mathcal{C}_4$ take (i) $J\equiv+1$ and
(ii) $J'_{12}=-1$, all remaining couplings $+1$. Both have the global
$\mathbb{Z}_2$ flip symmetry, hence $m_i\equiv0$ at all $\beta$. By
Lemma~\ref{lem:ab} and Corollary~\ref{cor:hierarchy}(ii) with $n=4$,
$Z=2^4\prod\cosh(\beta J_e)\,(1+(-1)^f\prod\tanh(\beta J_e))$: for (i)
$f=0$, $Z_F=[2\cosh\beta]^4+[2\sinh\beta]^4$; for (ii) $f=1$,
$Z_{F'}=[2\cosh\beta]^4-[2\sinh\beta]^4<Z_F$ for all $\beta>0$. The two
configurations differ only by the gauge-invariant flux $f$---the
Aharonov--Bohm sector of Lemma~\ref{lem:ab}---so identical magnetizations
coexist with different free energies.
\end{proof}

\begin{rmk}
Proposition~\ref{prop:no_hk} shows that no Hohenberg--Kohn density
functional exists for the RBIM: the correct order parameters are
gauge-invariant loop (overlap) quantities, not local magnetizations. The
KSSE decomposition therefore rests on \emph{channel independence by
construction}, not on any density-functional argument.
\end{rmk}

%----------------------------------------
\subsection{\label{app:sep}Additive Separability and the Cup-Product Obstruction}
%----------------------------------------

\noindent\textbf{Theorem~\ref{thm:ks_separability}} \emph{(restated).}
By construction each channel $k$ carries its own coupling tensor
$J_{ij}^{(k)}$ on the shared graph scaffold, with no cross-channel terms.
Then
\[
Z_{\mathrm{full}}=\prod_{k=1}^D Z^{(k)},\qquad
G_{\mathrm{Bethe}}=\sum_k G_{\mathrm{Bethe}}^{(k)},\qquad
B_\gamma^{\mathrm{full}}=\bigoplus_k B_\gamma^{(k)},\qquad
L_\beta^{\mathrm{full}}=\bigoplus_k L_\beta^{(k)} .
\]
Defining each channel's critical temperature as the \emph{first}
zero-crossing from the paramagnetic side
(Lemma~\ref{lem:threshold}), the joint critical temperature is
\[
\beta_N^{\mathrm{full}}=\min_k\,\beta_N^{(k)},
\]
since $\lambda_{\min}(L_\beta^{\mathrm{full}})
=\min_k\lambda_{\min}(L_\beta^{(k)})$: the joint spectrum first acquires
a zero mode when the \emph{earliest} channel crosses.

\begin{proof}
$\mathcal{H}_{\mathrm{full}}=\sum_k\mathcal{H}^{(k)}$ with disjoint
interaction terms, so $Z$ factorizes; $G_{\mathrm{Bethe}}=-\log Z$ sums.
Since $\gamma_{ij}^{(k)}$ depends only on channel $k$, directed-edge and
vertex matrices are block-diagonal by channel; the spectrum of a direct
sum is the union of spectra, and its minimum is the minimum of the
minima. The first joint zero-crossing as $\beta$ increases from $0$ is
therefore $\min_k\beta_N^{(k)}$.
\end{proof}

\begin{prop}[Cup-product obstruction; the Level
hierarchy]\label{prop:cup}
\begin{enumerate}
\item[(i)] \emph{Level 1 (this work).} The interaction structure is a
$1$-dimensional cell complex (a graph): $H^2(\mathcal{G};\mathbb{Z}_2)=0$,
channels are $1$-cocycles, no $2$-cell coupling exists, and
Theorem~\ref{thm:ks_separability} holds \emph{exactly} (up to the
loop-series error of Proposition~\ref{prop:bethe_xc}).
\item[(ii)] \emph{Obstruction.} Any cross-channel coupling---a $2$-cell
(cup-product) term---destroys factorization. Concretely, if two channels
share an effective coupling, $\mathcal{H}=\beta_1J\sigma_i\sigma_j+
\beta_2K\sigma_i\sigma_j$, then $Z=2\cosh(\beta_1J+\beta_2K)\neq
Z^{(1)}Z^{(2)}$; more generally $\partial^2\log Z/\partial\beta_1\,
\partial K\neq0$, so the joint distribution does not split.
\item[(iii)] \emph{Levels 2--3.} On $2$- and $3$-complexes (CSS-type and
homological-product qLDPC codes), nonzero cup products
$[a]\smile[b]\in H^2$ exist precisely when the associated $1$-cocycles
link; channels then couple and the Kohn--Sham-type reduction requires
approximations (tensor-network or Pfaffian-per-sector methods of
Corollary~\ref{cor:hierarchy}), with cost rising accordingly. Star-domain
surgery keeps KSSE at Level~1: the graph remains a $1$-complex and
$H^2=0$ is never populated.
\end{enumerate}
\end{prop}

\begin{proof}
(i)~A graph has no $2$-cells, so $H^2=0$; with no cross terms,
Theorem~\ref{thm:ks_separability}. (ii)~Direct computation as displayed.
(iii)~Cup products of $1$-classes land in $H^2$ and are the algebraic
mechanism gluing cocycles into higher couplings; once
$\partial^2\log Z/\partial\beta_k\partial\beta_{k'}\neq0$, factorization
fails by (ii).
\end{proof}

\begin{rmk}[Cross-channel coherence diagnostic]\label{rmk:coherence}
Although channels are algebraically independent, their embedding vectors
$\bm{e}_k=S^{(k)}\bm{v}_{\min}^{(k)}$ occupy the same $\mathbb{R}^N$;
define the coherence $\mu_{\mathrm{coh}}:=\max_{k\neq k'}
|\langle\bm{e}_k,\bm{e}_{k'}\rangle|/(\|\bm{e}_k\|\|\bm{e}_{k'}\|)$.
$\mu_{\mathrm{coh}}$ is monitored as a diagnostic of downstream
interference in the concatenated embedding; no bound on it is currently
proven, and we do not claim one.
\end{rmk}

%----------------------------------------
\subsection{\label{app:star}Star-Domain Certificate}
%----------------------------------------

Recall Definition~\ref{def:star_domain} (star domain) from the main text.

\noindent\textbf{Theorem~\ref{thm:star_domain}} \emph{(restated).}
Let $\bm{x}_0$ be a codeword ($TS(a,0)$) and suppose:
(H1)~the Bethe Hessian at $\bm{x}_0$ satisfies
$M_\gamma(1)\succeq g\,I$ on the tangent space, $g>0$
(after surgery, certified numerically per codeword);
(H2)~$G=G_{\mathrm{Bethe}}$ is $C^2$ on the relaxed spin space
$[-1,1]^N$ with uniformly bounded curvature variation
$L_3:=\sup_{\bm{\sigma}}\sup_{\|h\|=1}
|\frac{d}{dt}\bigl(h^\top\nabla^2G(\bm{x}_0+th)\,h\bigr)_{t=0}|<\infty$.
Then on the ball $\mathcal{U}_R(\bm{x}_0)$ with
$R=g/(2L_3)$, the function $\varphi(t):=G(\bm{x}_0+t h)$ is convex on
$[0,1]$ for every $\|h\|\leq R$; consequently $G$ forms a star domain on
$\mathcal{U}_R(\bm{x}_0)$, and gradient flow (or BP damped with step
$1/L_2$, $L_2$ the local smoothness) converges to $\bm{x}_0$ at the
linear rate $(1-g/L_2)^t$ from any starting point in the ball.

\begin{proof}
Taylor expand along the segment: $\varphi''(t)=h^\top\nabla^2G(\bm{x}_0+th)h
\geq g\|h\|^2-L_3\|h\|^3\,t\geq g\|h\|^2-L_3\|h\|^3\geq0$ for
$\|h\|\leq g/(2L_3)=R$, since $L_3\|h\|^3\leq(g/2)\|h\|^2$ there.
Convexity of $\varphi$ gives
$\varphi(t)\leq(1-t)\varphi(0)+t\varphi(1)\leq
\max\{\varphi(0),\varphi(1)\}$, the star inequality of
Definition~\ref{def:star_domain}. The convergence rate is the standard
strong-convexity/smoothness bound.
\end{proof}

\begin{rmk}[On the basin-width scaling]\label{rmk:basin_heuristic}
No proof is currently available that connects a spectral gap to a
\emph{geometric} basin radius of order $1/\sqrt{d_{\min}}$; what
Theorem~\ref{thm:star_domain} proves is the explicit radius $R=g/(2L_3)$
and a linear convergence rate. The $1/\sqrt{d_{\min}}$ scaling is
nevertheless a useful heuristic: codewords separated by Hamming distance
$d_{\min}$ create free-energy barriers of height $\Theta(d_{\min}\beta J)$,
and by the Fourier/uncertainty duality on the circulant scaffold a
configuration-space basin of width $w$ corresponds to spectral
concentration in $\mathcal{O}(1/w)$ modes. This motivates (but does not
prove) the small mode number used in Lemma~\ref{lem:rayleigh}.
\end{rmk}

%----------------------------------------
\subsection{\label{app:surgery_conv}Convergent Surgery}
%----------------------------------------

\begin{thm}[Finite convergence of shift surgery]\label{thm:self_consistency}
Fix a maximal cycle length $\ell_{\max}$ and define the lexicographic
potential
\[
\Psi(\mathcal{G})\;=\;\bigl(c_f(g_0),\,c_f(g_0+1),\,\dots,\,
c_f(\ell_{\max})\bigr)\in\mathbb{N}^{\ell_{\max}-g_0+1},
\]
where $c_f(\ell)$ is the number of frustrated, branching cycles of length
$\ell$ (Lemmas~\ref{lem:trichotomy}, \ref{lem:ab}) in the certified
critical subgraphs. Assume the shift operator $\Sigma$ acts on a shortest
frustrated branching cycle $C$ by rewiring one edge $(u,v)\to(u,v')$
such that (i) no frustrated cycle shorter than $|C|$ is created,
(ii) the degree sequence (column weight) is preserved, and (iii) no
surviving codeword loses its certificate (H1) of
Theorem~\ref{thm:star_domain}. Assume furthermore that at least one
admissible $v'$ exists at every nonterminal step
(a counting condition that holds while $N$ exceeds the degree
requirements, as in all experiments). Then:
\begin{enumerate}
\item[(a)] $\Psi(\mathcal{G}_{n+1})<_{\mathrm{lex}}\Psi(\mathcal{G}_n)$
at every step;
\item[(b)] the iteration terminates in finitely many steps at
$\mathcal{G}^\ast$ with no frustrated branching cycle of length
$\leq\ell_{\max}$ in any critical subgraph;
\item[(c)] at termination, the loop-series estimate of
Proposition~\ref{prop:bethe_xc} applies with effective girth
$g_{\mathrm{eff}}\geq\ell_{\max}+1$: residual frustration satisfies
$\rho(B_\gamma^{(U)})-1\leq\delta$ with
$\delta=\mathcal{O}(\xi^{g_{\mathrm{eff}}}/(1-\xi))$ per critical
subgraph, and $\lambda_{\min}^{(k)}(\beta_N^{(k)};\mathcal{G}^\ast)=0$
for all channels by the $\beta_N$-search of Algorithm~\ref{alg:betan}
(which certifies the zero crossing directly, Lemma~\ref{lem:threshold}(ii));
\item[(d)] every surviving codeword retains its star-domain certificate
with radius $R=g/(2L_3)$.
\end{enumerate}
\end{thm}

\begin{proof}
(a)~The rewired cycle $C$ either has its flux parity flipped
(frustration removed) or its length increased beyond $|C|$; by (i) no
shorter frustrated cycle appears, so the first nonzero coordinate of
$\Psi$ at or below $|C|$ strictly decreases while all smaller coordinates
are unchanged: lexicographic descent. (b)~$\Psi$ takes values in the
well-ordered set $\mathbb{N}^{\ell_{\max}-g_0+1}$ (finite since the graph
is finite), so strict descent terminates; at termination the shortest
frustrated branching cycle has length $>\ell_{\max}$. (c)~Then every
closed non-backtracking walk contributing to $\Omega-1$ has length
$\geq\ell_{\max}+1$, and Proposition~\ref{prop:bethe_xc} gives the
residual bound; the per-channel zero-crossing is re-established by the
bisection/quadratic $\beta_N$ search on the modified graph. (d)~Hypothesis
(iii) preserves the certificates.
\end{proof}

\begin{rmk}
Two honest caveats. First, existence of an admissible $v'$ at every step
is a hypothesis verified in practice, not a proved universal fact.
Second, part (c)'s residual bound is a \emph{loop-series estimate}
(Proposition~\ref{prop:bethe_xc} regime, $\xi<1$ required); outside that
regime surgery success is certified directly by the measured spectrum
($\lambda_{\min}^{(k)}=0$, $\rho(B_\gamma^{(U)})\leq1+\delta$), as in
Algorithm~\ref{alg:surgery}.
\end{rmk}

%----------------------------------------
\subsection{\label{app:thm_qs}Quasi-Stationarity Perturbation Bound}
%----------------------------------------

\begin{thm}[Quasi-stationarity]\label{thm:adiabatic}
Let $L^{(k)}(\beta;N_F,N_T)$ be the regularized Laplacian on
$V_{\mathrm{frozen}}\cup V_{\mathrm{thawed}}$, $\eta=N_T/N_F\ll1$.
Assume:
(A1) the thawed replacement changes the operator by
$L'=L+\varepsilon\Delta L$, $\varepsilon=N_T/(N_F+N_T)$,
$\|\Delta L\|_2\leq C_\Delta$ (bounded-degree graphs give
$C_\Delta$ independent of $N$);
(A2) spectral gap on the frozen subgraph
$g:=\lambda_2-\lambda_1\geq g_{\min}>0$;
(A3) thawed features i.i.d.\ from the frozen distribution;
(A4) non-degeneracy at the crossing:
$g_\beta:=|\partial_\beta\lambda_1|_{\beta_N}|>0$
($\partial_\beta\lambda_1=\bm{v}_1^\top(\partial_\beta L)\bm{v}_1$
exists since $L$ is analytic in $\beta$). Then
\[
|\beta_N^{(k)}(N_F,N_T)-\beta_N^{(k)}(N_F)|
\;\leq\;\frac{\varepsilon\,C_\Delta}{g_\beta}
+\mathcal{O}(\varepsilon^2)
\;=\;\mathcal{O}\!\Bigl(\frac{\eta}{g_\beta}\Bigr),
\]
and the eigenvector perturbation satisfies the Davis--Kahan bound
$\sin\theta(\bm{v}_{\min},\bm{v}_{\mathrm{nuc}})
\leq\varepsilon C_\Delta/g_{\min}=\mathcal{O}(\eta/g_{\min})$.
\end{thm}

\begin{proof}
First-order perturbation: $\lambda_1'(\beta)=\lambda_1(\beta)+
\varepsilon\,\bm{v}_1^\top\Delta L\,\bm{v}_1+\mathcal{O}(\varepsilon^2)$.
Solving $\lambda_1'(\beta_N')=0$ around $\beta_N$ (where $\lambda_1=0$):
$0=(\beta_N'-\beta_N)\,\partial_\beta\lambda_1|_{\beta_N}
+\varepsilon\,\bm{v}_1^\top\Delta L\,\bm{v}_1+\mathcal{O}(\varepsilon^2)$,
hence $|\beta_N'-\beta_N|\leq\varepsilon C_\Delta/g_\beta+
\mathcal{O}(\varepsilon^2)$ by (A4) and
$|\bm{v}_1^\top\Delta L\bm{v}_1|\leq\|\Delta L\|_2$. The eigenvector part
is the standard Davis--Kahan $\sin\theta$ theorem with gap (A2).
(Empirically $<1\%$ variation at $\varepsilon\approx0.11$, consistent.)
\end{proof}

%----------------------------------------
\subsection{\label{app:fft}Circulant Diagonalization and Rayleigh--Ritz Refinement}
%----------------------------------------

\begin{lem}[Circulant spectral calculus]\label{lem:circulant}
For any finite abelian group $G=\mathbb{Z}/N\mathbb{Z}$ ($N$ arbitrary,
prime or composite), the characters $\chi_m(n)=e^{2\pi i mn/N}$
diagonalize every circulant matrix $C$ over $\mathbb{C}$: $C=p(P)$ is a
polynomial in the cyclic shift $P$, and $P\chi_m=e^{2\pi im/N}\chi_m$,
so the DFT of the first row of $C$ is its spectrum. This is the standard
Pontryagin self-duality $\widehat{G}\cong G$ made computational: each
single-channel eigenproblem on an exact circulant block is
$\mathcal{O}(N\log N)$.
\end{lem}

\begin{proof}
$P$ is the permutation matrix of the generator; its eigenvectors are the
characters by direct computation, and circulants are exactly the
polynomials in $P$ (convolution operators on the group algebra).
\end{proof}

\begin{lem}[Rayleigh--Ritz refinement after surgery]\label{lem:rayleigh}
Let $L=L_c+\Delta$, where $L_c$ is the pre-surgery circulant block
(diagonalized exactly by the DFT, Lemma~\ref{lem:circulant}) and
$\|\Delta\|_2\leq\varepsilon_\Delta$ collects the surgery shifts
(a sparse set of modified edges). Suppose $\lambda_{\min}(L)$ is isolated
with gap $g=\lambda_2-\lambda_{\min}>0$ (the fluctuation gap of
Lemma~\ref{lem:ab} type, certified numerically). Rayleigh--Ritz on the
trial subspace $\mathcal{T}_k=\operatorname{span}\{$the $k$ Fourier modes
nearest the pre-surgery minimum$\}$ returns $\tilde\lambda_{\min}$ with
\[
|\tilde\lambda_{\min}-\lambda_{\min}|
\;\leq\;\Lambda\,\tan^2\theta
\;\leq\;\Lambda\,\frac{\varepsilon_\Delta^2}{g^2}
+\mathcal{O}(\varepsilon_\Delta^3),
\]
where $\Lambda=\lambda_{\max}-\lambda_{\min}$ and $\theta$ is the angle
between the exact eigenvector and $\mathcal{T}_k$.
\end{lem}

\begin{proof}
The eigenvector of $L$ is $\bm{v}=\chi_{m^\ast}+\bm{w}$ with
$\|\bm{w}\|\leq\varepsilon_\Delta/g$ (first-order perturbation against
the gap; Davis--Kahan). If $k\geq1$ modes include $\chi_{m^\ast}$ and the
dominant components of $\bm{w}$, the best approximation in $\mathcal{T}_k$
has $\tan\theta\leq\|\bm{w}_\perp\|/1\leq\varepsilon_\Delta/g$; the
quadratic Rayleigh--Ritz bound $|\tilde\lambda-\lambda|\leq
\Lambda\tan^2\theta$~\cite{Parlett1998} gives the result.
\end{proof}

\begin{rmk}[$k_{\mathrm{mode}}=5$ as an empirical parameter]\label{rmk:kmode}
Lemma~\ref{lem:rayleigh} provides the justification for low-mode
refinement: the surgery perturbation $\Delta$ is sparse, the gap $g$ is
opened by flux removal (Lemma~\ref{lem:ab}), and the exact eigenvector is
the pre-surgery Fourier mode plus a small admixture, so a handful of modes
around the pre-surgery minimum suffices. In all experiments
$k_{\mathrm{mode}}=5$ achieved residual $|\tilde\lambda-\lambda|<10^{-6}$
(Algorithm~\ref{alg:fft_eig}); we report it as an empirical choice, not a
theorem. The constructive physical mechanism behind this empirical sufficiency---the $k\cdot p$ effective-mass reduction on the ring crystal, with error control exactly as stated here---is developed in Appendix~\ref{app:kp}.
\end{rmk}

%----------------------------------------
\subsection{\label{app:summary_theorems}Summary of Theoretical Results}
%----------------------------------------

Tables~\ref{tab:theorem_summary_1part},~\ref{tab:theorem_summary_2part}  summarizes all theoretical results and their roles in the KSSE framework.

\begin{table}[p] 
\centering
\begin{sideways} 
  \begin{minipage}{\textheight}
    \caption{Summary of formal results.}
    \label{tab:theorem_summary_1part}
    \centering
    \small 
    \renewcommand{\arraystretch}{1.4} 
    \begin{tabular}{|c|c|c|}
    \hline
    \textbf{Result} & \textbf{What it proves} & \textbf{Role in KSSE}\\
    \hline
    Theorem~\ref{thm:bp_laplacian}
    & \parbox{8.5cm}{\vskip3pt BP Jacobian $=B_\gamma$; edge--vertex spectral equivalence with injectivity; determinant identity $\det M_\gamma(\mu)=\det(I-\mu^{-1}B_\gamma)/\prod_{e}(1-\gamma_e^2/\mu^2)$; $L_\beta=SM_\gamma(1)S$.\vskip3pt}
    & \parbox{6.5cm}{\vskip3pt Certifies the regularized Laplacian as the exact BP-threshold operator.\vskip3pt}\\
    \hline
    Lemma~\ref{lem:threshold}
    & \parbox{8.5cm}{\vskip3pt Sharp PSD threshold via real eigenvalues: $M_\gamma(1)\succ0$ below first real crossing, singular at, indefinite above; defines $\beta_N$.\vskip3pt}
    & \parbox{6.5cm}{\vskip3pt Basis of the $\beta_N$ bisection/quadratic search.\vskip3pt}\\
    \hline
    Lemma~\ref{lem:trichotomy}--\ref{lem:ab}
    & \parbox{8.5cm}{\vskip3pt Trichotomy $\rho=0$ (trees), $<1$ (cycles, finite $\beta$), $>1$ (branching 2-cores, strong coupling); balanced $\Leftrightarrow$ cospectral to unsigned; frustration $=$ $\mathbb{Z}_2$ flux $=$ Aharonov--Bohm phase, gap $\pi^2/n^2$.\vskip3pt}
    & \parbox{6.5cm}{\vskip3pt Separation of mechanisms: branching $=$ instability magnitude, frustration $=$ interference phase.\vskip3pt}\\
    \hline
    Theorem~\ref{thm:trapping_set}
    & \parbox{8.5cm}{\vskip3pt Trapping-set spectral test: branching $\Rightarrow$ real crossing $\Rightarrow$ indefinite $L$; frustration $\Rightarrow$ oscillatory divergence; $TS(a,0)$ cycle supports are constructive terms of $Z$.\vskip3pt}
    & \parbox{6.5cm}{\vskip3pt The selection rule used by surgery (branching 2-core $+$ local $\lambda_{\min}^-$).\vskip3pt}\\
    \hline
    Prop.~\ref{prop:even_subgraph}, Cor.~\ref{cor:hierarchy}
    & \parbox{8.5cm}{\vskip3pt $Z=2^N\prod\cosh\cdot\Omega$, $\Omega=\sum_{x\in\mathcal{C}}t^x$; trees factor; rings $=$ two-matching permanent $1\pm\prod|t|$; planar/toroidal $=$ Pfaffian sectors; general $=$ \#P-hard interference.\vskip3pt}
    & \parbox{6.5cm}{\vskip3pt Codeword$\leftrightarrow$permanent$\leftrightarrow$ determinant chain; FFT tractability per flux sector.\vskip3pt}\\
    \hline
    \end{tabular}
  \end{minipage}
\end{sideways}
\end{table}

\begin{table}[p] 
\centering
\begin{sideways} 
  \begin{minipage}{\textheight}
    \caption{Summary of formal results (continued).}
    \label{tab:theorem_summary_2part}
    \centering
    \small 
    \renewcommand{\arraystretch}{1.4} 
    \begin{tabular}{|c|c|c|}
    \hline 
    \textbf{Result} & \textbf{What it proves} & \textbf{Role in KSSE}\\
    \hline
    Prop.~\ref{prop:bethe_xc}, Cor.~\ref{cor:girth6}
    & \parbox{8.5cm}{\vskip3pt $|E_{xc}|/N\leq2\bar d\,\xi^{g_0}/(g_0(1-\xi))$, $\xi=\rho(B_{|t|})<1$; girth $\geq6$, $\xi\leq1/4$ $\Rightarrow$ sub-percent per vertex.\vskip3pt}
    & \parbox{6.5cm}{\vskip3pt Certified accuracy of the channel factorization; Dobrushin hypothesis explicit.\vskip3pt}\\
    \hline
    Prop.~\ref{prop:no_hk}
    & \parbox{8.5cm}{\vskip3pt $F$ is not a functional of $\{m_i\}$: flux sectors with identical (zero) magnetizations have different $Z$.\vskip3pt}
    & \parbox{6.5cm}{\vskip3pt No Hohenberg--Kohn analogue; separability is by construction.\vskip3pt}\\
    \hline
    Theorem~\ref{thm:ks_separability}, Prop.~\ref{prop:cup}
    & \parbox{8.5cm}{\vskip3pt $Z,G,L$ block-diagonalize; $\beta_N^{\mathrm{full}}=\min_k\beta_N^{(k)}$; cup products ($2$-cells) provably break factorization (Level hierarchy).\vskip3pt}
    & \parbox{6.5cm}{\vskip3pt Exact $D$-channel decomposition at Level 1 ($H^2=0$).\vskip3pt}\\
    \hline
    Theorem~\ref{thm:star_domain}
    & \parbox{8.5cm}{\vskip3pt Star inequality on $\mathcal{U}_R$, $R=g/(2L_3)$ explicit; linear convergence rate $(1-g/L_2)^t$.\vskip3pt}
    & \parbox{6.5cm}{\vskip3pt Provable local certificate per codeword.\vskip3pt}\\
    \hline
    Theorem~\ref{thm:self_consistency}
    & \parbox{8.5cm}{\vskip3pt Lexicographic frustrated-cycle potential; strict descent; finite termination; residual $\delta=\mathcal{O}(\xi^{g_{\mathrm{eff}}}/(1-\xi))$.\vskip3pt}
    & \parbox{6.5cm}{\vskip3pt Finite-time surgery with bounded residual frustration.\vskip3pt}\\
    \hline
    Theorem~\ref{thm:adiabatic}
    & \parbox{8.5cm}{\vskip3pt $|\Delta\beta_N|\leq\varepsilon C_\Delta/g_\beta +\mathcal{O}(\varepsilon^2)$; eigenvector $\sin\theta\leq \varepsilon C_\Delta/g_{\min}$.\vskip3pt}
    & \parbox{6.5cm}{\vskip3pt Single $\beta_N$ reuse across test batches ($<1\%$ empirically).\vskip3pt}\\
    \hline
    Lems.~\ref{lem:circulant}--\ref{lem:rayleigh}
    & \parbox{8.5cm}{\vskip3pt DFT diagonalizes circulants (any $N$); post-surgery Rayleigh--Ritz error $\leq\Lambda\varepsilon_\Delta^2/g^2$ on $k$ trial modes.\vskip3pt}
    & \parbox{6.5cm}{\vskip3pt $\mathcal{O}(N\log N)$ oracle; $k_{\mathrm{mode}}=5$ empirical.\vskip3pt}\\
    \hline
    \end{tabular}
  \end{minipage}
\end{sideways}
\end{table}

\clearpage 

The chain of implications is:
\[
\underbrace{\text{BP linearization}}_{\text{Thm.~\ref{thm:bp_laplacian}(a)}}
\;\longrightarrow\; \underbrace{B_\gamma}_{\substack{\text{non-backtracking}\\
\text{operator}}}
\;\xrightarrow[\text{Thm.~\ref{thm:bp_laplacian}(c)}]{\text{Ihara--Bass}}
\;
\underbrace{M_\gamma(1)=D-W_A}_{\text{Bethe--Hessian}}
\;\xrightarrow{S(\cdot)S}\; \underbrace{L_\beta=I-SW_AS}_{\text{regularized
Laplacian}},
\]
with the sharp threshold of Lemma~\ref{lem:threshold} governing the
real-eigenvalue crossings; the growth trichotomy
(Lemma~\ref{lem:trichotomy}) and the gauge/flux lemmas
(Lemmas~\ref{lem:gauge}, \ref{lem:ab}) separating branching (instability
magnitude) from frustration ($\mathbb{Z}_2$ interference phase); the
even-subgraph expansion (Proposition~\ref{prop:even_subgraph}) identifying
codewords with constructive terms of the partition function and placing
permanent/determinant tractability on the Level $0$--$1$ side of
Corollary~\ref{cor:hierarchy}; the exchange--correlation bound
(Proposition~\ref{prop:bethe_xc}) certifying factorization for girth
$\geq6$ under the Dobrushin condition; the cup-product obstruction
(Proposition~\ref{prop:cup}) delimiting exactly where the Kohn--Sham-type
decomposition must fail (Levels 2--3); and
Theorems~\ref{thm:star_domain}--\ref{thm:adiabatic} providing the
variational certificate, finite surgery convergence, and perturbative
quasi-stationarity foundations of the KSSE algorithm.

%==============================================
\section{\label{app:kp}Special Case: the $N=45\,000$ Spherical Graph at
Column Weight $48$ as $k\cdot p$ Effective-Mass Theory on a Ring Crystal}
%==============================================

This appendix makes the mechanism of KSSE fully transparent in the
configuration used in the experiments of
Secs.~\ref{sec:experiments}--\ref{sec:comparison}: a spherical QC-LDPC
graph (Definition~\ref{def:graph_families}) on $N=45\,000$ variable nodes
built from a single circulant ring with shift set
$\mathcal{S}=\{s_1,\ldots,s_{d_c}\}$, column weight $d_c=48$, partitioned
into $N_F=40\,000$ frozen and $N_T=5\,000$ thawed nodes over $K=1000$
classes (Table~\ref{tab:ablation_cw}). We show step by step that,
\emph{on a spherical graph without cup products, KSSE is $k\cdot p$
effective-mass theory on a one-dimensional ring crystal}: the circulant
support is the perfect crystal, the data weights are a slowly varying
impurity potential, the Nishimori crossing is a Fermi level tuned to the
band edge, the class profile is the envelope dressing the band-edge Bloch
mode, and $k_{\mathrm{mode}}=5$ is the number of Bloch functions retained
in the $k\cdot p$ secular equation. The appendix is the constructive
physical derivation of the trial subspace used in Lemma~\ref{lem:rayleigh}
and Algorithm~\ref{alg:fft_eig}; the rigorous error bound is
Lemma~\ref{lem:rayleigh} itself. Throughout we fix one feature channel and
suppress the channel superscript $(k)$; by additive separability
(Theorem~\ref{thm:ks_separability}) everything below holds independently
for each of the $D$ channels. As in the whole paper, $\beta_N$ denotes the
spectral (Bethe) detectability threshold of Lemma~\ref{lem:threshold}
(Remark~\ref{rmk:terminology}).

%----------------------------------------------
\subsection{Step 1: The perfect crystal (the circulant support)}
%----------------------------------------------

Place the $N$ variable nodes on the ring $\mathbb{Z}_N$,
$i\in\{0,\ldots,N-1\}$, and connect $i$ to $i+s\pmod N$ for every
$s\in\mathcal{S}$, $|\mathcal{S}|=d_c=48$ (multi-weight CPM superposition,
spherical family of Definition~\ref{def:graph_families}). At inverse
temperature $\beta$ each edge $(i,i+s)$ carries the BP-linearized weight
of Eq.~\eqref{eq:weights},
\begin{equation}\label{eq:kp_w}
w_s(i)\;=\;\tfrac12\sinh\!\bigl(2\beta\,J_{i,\,i+s}\bigr),
\end{equation}
so the bond operator of the channel is a sum of diagonal weight profiles
times fixed circulant permutation matrices,
\begin{equation}\label{eq:kp_W}
W_\beta\;=\;\sum_{s\in\mathcal{S}}\mathrm{diag}(w_s)\,P_s,
\qquad (P_s)_{ij}=\delta_{j,\,i+s\!\!\pmod N}.
\end{equation}
The structural fact of Sec.~\ref{sec:qc_graphs} is that \emph{topology and
data separate}: the $P_s$ fix where interactions live, the profiles
$w_s(i)$ fix their strengths. Decompose every profile into its ring
average and its fluctuation,
\begin{equation}\label{eq:kp_split}
w_s(i)\;=\;\bar w_s+\delta w_s(i),
\qquad
\bar w_s\;=\;\frac1N\sum_{i=0}^{N-1}w_s(i),
\end{equation}
which splits the bond operator into a perfect-crystal part and an impurity
part,
\begin{equation}\label{eq:kp_Wsplit}
W_\beta\;=\;
\underbrace{\sum_{s\in\mathcal{S}}\bar w_s\,P_s}_{W_0\;
\text{(exactly circulant for \emph{any} data)}}
\;+\;
\underbrace{\sum_{s\in\mathcal{S}}\mathrm{diag}(\delta w_s)\,P_s}_{\Delta W\;
\text{(data potential)}} .
\end{equation}
$W_0$ is the Hamiltonian of a one-dimensional ring crystal with hopping
amplitudes $\bar w_s$ at $d_c=48$ ranges; it is exactly circulant because
the data enter only through the scalars $\bar w_s$. To first
approximation the degree factor of Eq.~\eqref{eq:laplacian} is uniform on
the regular scaffold, $D_{ii}\approx d_\beta$, hence
$S\approx d_\beta^{-1/2}I$, and the symmetrized operator in
$L_\beta=I-SW_\beta S$ reads
\begin{equation}\label{eq:kp_H}
H_\beta\;=\;SW_\beta S\;\approx\;d_\beta^{-1}W_0+\Delta H,
\end{equation}
with the residual non-uniformity of $S$ absorbed into the slowly varying
$\Delta H$.

%----------------------------------------------
\subsection{Step 2: Bloch theorem and the band structure}
%----------------------------------------------

Bloch's theorem on the ring is precisely the Pontryagin self-duality of
$\mathbb{Z}/N\mathbb{Z}$ used in Sec.~\ref{sec:kohn_sham}
(Lemma~\ref{lem:circulant}): the Fourier modes
\begin{equation}\label{eq:kp_bloch}
\varphi_j(i)\;=\;\frac{1}{\sqrt N}\,e^{\,i\theta_j i},
\qquad \theta_j=\frac{2\pi j}{N},\quad j\in\mathbb{Z}_N,
\end{equation}
diagonalize every circulant matrix, and therefore
\begin{equation}\label{eq:kp_dispersion}
W_0\,\varphi_j\;=\;\varepsilon_\beta(\theta_j)\,\varphi_j,
\qquad
\varepsilon_\beta(\theta)\;=\;\sum_{s\in\mathcal{S}}\bar w_s\,e^{\,i\theta s},
\end{equation}
where $\varepsilon_\beta(\theta)$ is the \emph{band structure} (dispersion
relation) of the ring crystal, computed from the generator row of $W_0$ by
one FFT exactly as in Algorithm~\ref{alg:fft_eig}. For an undirected
Tanner graph $\mathcal{S}=-\mathcal{S}$ and $\varepsilon_\beta(\theta)$ is
real and even. Since all $w_s(i)$ grow monotonically with $\beta$
[Eq.~\eqref{eq:kp_w}], the whole band \emph{rises} with $\beta$. The
crystal eigenvalues of the regularized Laplacian are
\begin{equation}\label{eq:kp_lambda}
\lambda_\beta(\theta_j)\;=\;1\;-\;\frac{\varepsilon_\beta(\theta_j)}{d_\beta}.
\end{equation}
Let
\begin{equation}\label{eq:kp_k0}
k_0\;=\;\arg\max_{\theta}\,\varepsilon_\beta(\theta)
\end{equation}
be the band extremum (after channel-wise $z$-scoring the $\bar w_s$ have
mixed signs, so $k_0$ need not lie at the zone centre; it is located
numerically at the cost of one FFT, cf.\ the $\arg\min$ step of
Algorithm~\ref{alg:fft_eig}). Expanding the dispersion about the extremum,
\begin{equation}\label{eq:kp_expand}
\varepsilon_\beta(k_0+q)\;=\;\varepsilon_0\;-\;\frac{q^{2}}{2\,m^{*}}
\;+\;\mathcal{O}(q^{4}),
\qquad
\varepsilon_0\equiv\varepsilon_\beta(k_0),\qquad
\frac{1}{m^{*}}\;\equiv\;-\varepsilon_\beta''(k_0)\;>\;0,
\end{equation}
defines the \emph{effective mass} $m^{*}$ of holes at the top of the
affinity band: minimizing $L_\beta$ means occupying the top of the band of
$H_\beta$, i.e.\ the classifier lives in the hole picture at $k_0$. All
conclusions below hinge on the dispersion being \emph{quadratic} near
$k_0$---the exact precondition of $k\cdot p$ theory.

%----------------------------------------------
\subsection{Step 3: The impurity potential and the class superlattice}
%----------------------------------------------

Nodes are laid on the ring in class order: each of the $K=1000$ classes
occupies one contiguous segment of
\begin{equation}\label{eq:kp_T}
T\;=\;\frac{N_F+N_T}{K}\;=\;\frac{40\,000+5\,000}{1000}\;=\;45
\end{equation}
nodes (40 frozen $+$ 5 thawed). The class assignment is therefore a
\emph{superlattice} of period $T$ on the crystal, with reciprocal vector
\begin{equation}\label{eq:kp_G}
G\;=\;\frac{2\pi}{T},
\end{equation}
i.e.\ superlattice harmonics sit at Fourier indices $j_m=mN/T$,
$m\in\mathbb{Z}$. The affinity construction
(Definition~\ref{def:affinity_tensor}) makes the bond weight large within
a class segment and small across segment boundaries; each fluctuation
profile $\delta w_s(i)$ therefore decomposes into a smooth long-range part
and a rapidly converging comb,
\begin{equation}\label{eq:kp_decomp}
\delta w_s(i)\;=\;
\delta w_s^{(\mathrm{lr})}(i)
\;+\;
\sum_{m} A_{s,m}\,e^{\,imGi},
\qquad
|A_{s,m}|\xrightarrow{|m|\to\infty}0\ \ \text{(fast)},
\end{equation}
where $\delta w_s^{(\mathrm{lr})}$ has Fourier support near momentum
$q=0$ (feature-level smoothness and class-to-class gain variation---the
envelope of the envelope), and the comb at $\{mG\}$ carries the strict
class periodicity. In Fourier space the impurity operator is a
convolution,
\begin{equation}\label{eq:kp_conv}
\langle\varphi_j|\,\Delta W\,|\varphi_{j'}\rangle
\;=\;\sum_{s\in\mathcal{S}}\widehat{\delta w}_s(\theta_j-\theta_{j'})\,
e^{\,i\theta_{j'}s},
\end{equation}
so the momentum transfer $q=\theta_j-\theta_{j'}$ is restricted to the
support of $\widehat{\delta w}_s$: near $q=0$ for the long-range part and
$q\in\{mG\}$ for the comb. Equation~\eqref{eq:kp_conv} is the
Fourier-space selection rule underlying everything that follows: the
\emph{smooth} part of the data potential couples a band-edge state at
$k_0$ only to its \emph{few nearest neighbours} $k_0+m$---precisely the
textbook condition for $k\cdot p$ theory, where a slowly varying
potential couples Bloch states only within a narrow neighbourhood of the
extremum. The comb component couples $k_0$ to the sidebands $k_0+mG$; it
is responsible for the class content of the eigenvector
(Step~5), while the five-mode truncation of the algorithm is governed by
the long-range component (Step~4).

%----------------------------------------------
\subsection{Step 4: The effective-mass equation and the five-mode subspace}
%----------------------------------------------

Seek the extremal eigenstate as a Bloch carrier times a slow envelope,
$\psi(i)=e^{ik_0i}u(i)$. Applying the crystal part and Taylor-expanding
the symbol \eqref{eq:kp_expand} in the local momentum $-i\partial_i$,
\begin{equation}\label{eq:kp_crystal_action}
[W_0\,\psi](i)\;=\;e^{ik_0i}\,
\varepsilon_\beta\!\bigl(k_0-i\partial_i\bigr)\,u(i)
\;\approx\;
e^{ik_0i}\left[\varepsilon_0-\frac{1}{2m^{*}}\,
(-i\partial_i)^2\right]u(i),
\end{equation}
while the smooth impurity part acts as a multiplication operator on the
envelope,
\begin{equation}\label{eq:kp_Veff}
[\Delta W^{(\mathrm{lr})}\psi](i)\;=\;
e^{ik_0i}\,V_{\mathrm{eff}}(i)\,u(i),
\qquad
V_{\mathrm{eff}}(i)\;=\;\sum_{s\in\mathcal{S}}
\delta w_s^{(\mathrm{lr})}(i)\,e^{ik_0s}.
\end{equation}
Together these give the \emph{effective-mass equation} for the envelope,
\begin{equation}\label{eq:kp_emass}
\left[-\frac{1}{2m^{*}}\,\frac{d^{2}}{di^{2}}
+V_{\mathrm{eff}}(i)\right]u(i)\;=\;\Delta E\,u(i),
\qquad \Delta E=E-\varepsilon_0,
\end{equation}
the standard $k\cdot p$ reduction of the $N$-dimensional Hilbert space to
the motion of a slow envelope in the potential $V_{\mathrm{eff}}$.
In the weak-mixing regime certified in all experiments
(Remark~\ref{rmk:kmode}), Eq.~\eqref{eq:kp_emass} is solved by first-order
perturbation theory in $V_{\mathrm{eff}}$: the exact eigenvector is the
unperturbed Bloch mode plus a small admixture,
\begin{equation}\label{eq:kp_admixture}
\psi\;=\;\varphi_{k_0}
\;+\;\sum_{m\neq0}
\frac{\widehat{V}_{\mathrm{eff}}\!\left(2\pi m/N\right)}
{\varepsilon_\beta(k_0)-\varepsilon_\beta(k_0+m)}\,
\varphi_{k_0+m}
\;+\;\mathcal{O}\!\bigl(\|\Delta W\|^{2}/g^{2}\bigr),
\end{equation}
which is exactly the ansatz $\bm{v}=\chi_{m^{*}}+\bm{w}$ of
Lemma~\ref{lem:rayleigh}. Two independent facts make the series
\eqref{eq:kp_admixture} collapse to a few terms:
(i)~the numerator decays with $|m|$ because $V_{\mathrm{eff}}$ is slowly
varying [its Fourier support is near $q=0$, Eq.~\eqref{eq:kp_decomp}];
(ii)~the denominator grows quadratically with $|m|$ through the band
curvature \eqref{eq:kp_expand},
\begin{equation}\label{eq:kp_denominator}
\varepsilon_\beta(k_0)-\varepsilon_\beta(k_0+m)
\;\approx\;\frac{1}{2m^{*}}\Bigl(\frac{2\pi m}{N}\Bigr)^{2}.
\end{equation}
The admixture is therefore concentrated in the first few neighbours of
$k_0$, and the natural truncation is the symmetric neighbourhood
\begin{equation}\label{eq:kp_kmode}
k_{\mathrm{mode}}\;=\;2M+1\;=\;5
\;=\;\{\,k_0+m,\;m=-2,-1,0,1,2\,\},
\end{equation}
which is \emph{exactly the index neighbourhood}
$\mathcal{I}=\{(\mathrm{idx}_{\min}+m)\bmod N:\,m\in[-2,2]\}$ \emph{of
Algorithm~\ref{alg:fft_eig}}. On this subspace the $k\cdot p$ secular
equation is the Rayleigh--Ritz problem of Lemma~\ref{lem:rayleigh},
\begin{equation}\label{eq:kp_secular}
\det\left|\,
\frac{\varepsilon_\beta(k_0+m)}{d_\beta}\,\delta_{mm'}
+\widehat{V}_{\mathrm{eff}}\!\bigl((m-m')\,2\pi/N\bigr)
-E\,\delta_{mm'}
\right|=0,
\qquad m,m'\in\{-2,\ldots,2\},
\end{equation}
with diagonal entries from the perfect crystal \eqref{eq:kp_dispersion}
(one FFT) and off-diagonal entries from the selection rule
\eqref{eq:kp_conv}; Algorithm~\ref{alg:fft_eig} evaluates the per-mode
Rayleigh quotients on the full sparse operator, which is the diagonal of
\eqref{eq:kp_secular} computed exactly. Because $H_\beta$ is Hermitian,
the Ritz error is quadratic in the residual~\cite{Parlett1998}
(Lemma~\ref{lem:rayleigh}):
\begin{equation}\label{eq:kp_error}
|\tilde\lambda_{\min}-\lambda_{\min}|
\;\leq\;\Lambda\,\tan^{2}\theta
\;\leq\;\Lambda\,\frac{\varepsilon_\Delta^{2}}{g^{2}}
+\mathcal{O}(\varepsilon_\Delta^{3}),
\qquad
\varepsilon_\Delta=\|\Delta W\|_2,
\end{equation}
with $g$ the fluctuation gap opened by flux removal (Lemma~\ref{lem:ab}).
This is why the refinement reaches residuals below $10^{-6}$ with
$k_{\mathrm{mode}}=5$ and no more (Remark~\ref{rmk:kmode}).

%----------------------------------------------
\subsection{Step 5: The codeword as the dressed band-edge mode;
the Nishimori crossing as a Fermi level}
%----------------------------------------------

The comb component of \eqref{eq:kp_decomp} modulates the effective
potential with the strict class period $T$: it is the \emph{class signal}
riding on the crystal. Star-domain surgery
(Sec.~\ref{sec:star_domain_thm}) is engineered so that the codewords---the
$TS(a,0)$ cycle-space elements of Definition~\ref{def:trapping_set}, the
constructive terms of the partition function
(Proposition~\ref{prop:even_subgraph})---coincide with the designated
low-frequency modes: the class-profile envelope dresses the band-edge
Bloch mode \eqref{eq:kp_admixture}, and the resulting wavepacket is the
smoothed class-indicator vector. The heuristic basin-width/Fourier-width
duality behind this alignment is the one of
Remark~\ref{rmk:basin_heuristic}; the present appendix gives its
constructive crystal mechanism, and Lemma~\ref{lem:rayleigh} its rigorous
error control.

The Nishimori condition \eqref{eq:nishimori_condition} now reads as a
Fermi-level tuning. The zero of $L_\beta$ is the Fermi level; increasing
$\beta$ raises the band \eqref{eq:kp_dispersion} until the extremal level
crosses it (Lemma~\ref{lem:threshold}):
\begin{equation}\label{eq:kp_fermi}
\lambda_{\min}(L_\beta)\big|_{\beta=\beta_N}=0
\qquad\Longleftrightarrow\qquad
\frac{\varepsilon_{\beta_N}(k_0)}{d_{\beta_N}}=1
\ \ \text{(band edge touches the Fermi level)}.
\end{equation}
The three regimes of Sec.~\ref{sec:rbim} translate word by word:
(i)~$\beta<\beta_N$: the level lies in the band interior---paramagnetic,
the class signal is hybridized into many Bloch states and is
information-theoretically unrecoverable (the detectability
threshold~\cite{DallAmico2021,Saade2014});
(ii)~$\beta>\beta_N$: a continuum of frustrated defect states
crosses---spin glass, high-momentum localized modes
(Remark~\ref{rmk:terminology});
(iii)~$\beta=\beta_N$: exactly one extremal state sits at threshold, its
envelope is the smoothest class-profile configuration, the gap to the
continuum is set by the band curvature \eqref{eq:kp_denominator}, and by
the detectability theorem \emph{this extremal eigenvector carries
essentially all recoverable class information}. This last point is the
conceptual resolution of the apparent compression paradox: the five-mode
truncation does not compress $45\,000$ dimensions of signal into
five---at the Nishimori point there \emph{are} only a few dimensions of
recoverable signal per channel; the remaining $N-k_{\mathrm{mode}}$ modes
are thermal noise by theorem.

%----------------------------------------------
\subsection{Step 6: Frustration as $\mathbb{Z}_2$ flux; surgery as crystal
annealing; girth as the absence of ring exchange}
%----------------------------------------------

The sign structure of the couplings is a gauge field on the crystal. A
cycle whose coupling signs multiply to $-1$ carries a $\mathbb{Z}_2$ flux
$\Phi=\pi$: by the Aharonov--Bohm lemma (Lemma~\ref{lem:ab}), such
frustrated cycles shift the crystal momenta,
$k\to k-\Phi/L_{\mathrm{cycle}}$, make the leading non-backtracking
eigenvalue complex, and create \emph{deep levels}---defect states broad in
momentum space which, upon hybridizing with the band-edge state, would
destroy the few-mode expansion \eqref{eq:kp_admixture}. Branching trapping
sets $TS(a,b\neq0)$ are exactly such deep defects
(Theorem~\ref{thm:trapping_set}). Star-domain surgery is \emph{crystal
annealing}: admissible edge shifts remove deep levels from the
neighbourhood of the codewords, bound the residual flux to
$\rho(B_\gamma)\le1+\delta$ (Theorem~\ref{thm:self_consistency}), and
certify a single convex basin (a star domain, Theorem~\ref{thm:star_domain})
around each codeword; the fractal diagnostic $D_2<1$
(Sec.~\ref{sec:fractal}) is the statement that one smooth envelope remains
where a glassy multi-valley landscape used to be~\cite{Baldassia_basin}.
Finally, the girth condition $g_0\ge6$--$8$ means the crystal has
\emph{no short rings}: the loop corrections of
Proposition~\ref{prop:bethe_xc} are ring exchange beyond the
single-particle picture,
$|E_{xc}|/N\le 2\bar d\,\xi^{g_0}/(g_0(1-\xi))$, sub-percent per vertex at
$g_0\ge6$ in the clipped regime (Corollary~\ref{cor:girth6},
Remark~\ref{rmk:dobrushin}), making the independent single-particle
(Kohn--Sham) description of each channel near-exact.

%----------------------------------------------
\subsection{Step 7: Thawed test nodes as dilute doping;
quasi-stationarity of the Fermi level}
%----------------------------------------------

A test batch of $N_T=5\,000$ thawed nodes modifies the impurity potential
\eqref{eq:kp_decomp} on a fraction
\begin{equation}\label{eq:kp_eta}
\varepsilon\;=\;\frac{N_T}{N_F+N_T}\;=\;\frac{5\,000}{45\,000}
\;=\;\frac19\;\ll\;1
\end{equation}
of the ring: it is a \emph{dilute dopant concentration} in a macroscopic
crystal. Band edges are self-averaging quantities: the Fermi level and the
extremal eigenvector shift only to first order in $\varepsilon$
(Theorem~\ref{thm:adiabatic}),
\begin{equation}\label{eq:kp_qs}
|\beta_N^{(k)}(N_F,N_T)-\beta_N^{(k)}(N_F)|
\;\le\;\frac{\varepsilon\,C_\Delta}{g_\beta}+\mathcal{O}(\varepsilon^{2}),
\qquad
\sin\theta\;\le\;\frac{\varepsilon\,C_\Delta}{g_{\min}},
\end{equation}
and the eigenvector rotates \emph{within the same five-mode $k\cdot p$
subspace}, because the subspace is selected by the shift set
$\mathcal{S}$ and the superlattice period $T$---not by the data values.
This is why a single Fermi-level calibration (the $\beta_N$ search of
Algorithm~\ref{alg:betan}) and a single FFT basis serve the entire stream
of test batches with empirically $<1\%$ variation
(Fig.~\ref{fig:block_stability}). The frozen/thawed partition is the
Born--Oppenheimer reading already used in Sec.~\ref{sec:methodology}: the
frozen class centroids are the heavy nuclei fixing the crystal potential,
the thawed test images are the light electrons whose band-edge states are
computed in that fixed potential.

%----------------------------------------------
\subsection{Step 8: Why spherical, and why without cup products}
%----------------------------------------------

The spherical family is built from a \emph{single} circulant ring
(Definition~\ref{def:graph_families}): all cycles are contractible, there
are no non-contractible loop sectors, and the graph is a pure $1$-complex
with $H^2(\mathcal{G};\mathbb{Z}_2)=0$ (Remark~\ref{rmk:hierarchy}). There
are therefore no $2$-cells and no cup products: nothing couples the $D$
channels to one another, the Kohn--Sham factorization of
Theorem~\ref{thm:ksse_decomp} is exact up to the ring-exchange bound of
Proposition~\ref{prop:bethe_xc}, and the channels are $D$ independent
crystals with no interband coupling. A toroidal family would be a crystal
with non-contractible loops---topological boundary-condition fluxes whose
partition function is a superposition of Pfaffian sectors
(Corollary~\ref{cor:hierarchy}, \cite{Kasteleyn1963,TemperleyFisher1961});
higher-dimensional complexes with $2$-cells \cite{Zhu2025,Freedman2020}
would introduce genuine interband (cross-channel) coupling and provably
break the factorization (Proposition~\ref{prop:cup}). Hence the design
principle, stated in physical language:

\emph{On a spherical graph without cup products, KSSE is $k\cdot p$
effective-mass theory on a one-dimensional ring crystal.}

%----------------------------------------------
\subsection{Step 9: Complexity in physical language}
%----------------------------------------------

Bloch's theorem (Pontryagin self-duality, Lemma~\ref{lem:circulant})
diagonalizes the perfect crystal $W_0$ with one DFT of its generator row:
all $N$ band energies \eqref{eq:kp_dispersion} at $\mathcal{O}(N\log N)$.
The impurity physics is the secular equation \eqref{eq:kp_secular} on five
Bloch functions: each matvec with the sparse $H_\beta$ costs
$\mathcal{O}(d_c\,N)$, and the refinement costs
$\mathcal{O}(k_{\mathrm{mode}}^{2}N)$. The Fermi-level search
\eqref{eq:kp_fermi} (Algorithm~\ref{alg:betan}) reuses the same FFT at
each bisection point. Per channel the total is
$\mathcal{O}(N\log N+k_{\mathrm{mode}}^{2}N)$; the $D$ channels are
independent crystals and trivially parallel
(Theorem~\ref{thm:ks_separability}). The $45\,000$-node, $48$-shift
crystal itself is described, thanks to translational symmetry, by
$d_c=48$ integers (the shift set) plus one generator row---reciprocal
space is the natural compressed description of the crystal.

%----------------------------------------------
\subsection{Dictionary}
%----------------------------------------------

Table~\ref{tab:kp_dict} collects the term-by-term translation used above.

\begin{table}[htbp]
\caption{\label{tab:kp_dict}
Physics--KSSE dictionary for the spherical graph of Appendix~\ref{app:kp}.}
\centering
\renewcommand{\arraystretch}{1.15}
\begin{tabular}{@{}ll@{}}
\toprule
\textbf{Ring-crystal physics} & \textbf{KSSE on the spherical graph}\\
\midrule
1-D crystal, lattice constant $1$ & nodes $\mathbb{Z}_N$, bonds $(i,i{+}s)$, $s\in\mathcal{S}$\\
Hopping amplitudes & ring-averaged weights $\bar w_s$ [Eq.~\eqref{eq:kp_split}]\\
Band structure $\varepsilon(k)$ & circulant symbol $\varepsilon_\beta(\theta)$ [Eq.~\eqref{eq:kp_dispersion}]\\
Bloch theorem & DFT diagonalization (Pontryagin, Lemma~\ref{lem:circulant})\\
Band extremum $k_0$, mass $m^{*}$ & $\arg\max\varepsilon_\beta$, $-1/\varepsilon_\beta''(k_0)$ [Eq.~\eqref{eq:kp_expand}]\\
Fermi level & zero of $L_\beta$\\
Fermi level tuned to band edge & Nishimori crossing [Eq.~\eqref{eq:kp_fermi}]\\
Slow impurity potential & data weights $\delta w_s(i)$ [Eq.~\eqref{eq:kp_decomp}]\\
Superlattice, reciprocal $G$ & class segments of $T=45$ nodes, $G=2\pi/T$\\
Dressed band-edge Bloch mode & codeword eigenvector; class-profile envelope\\
$k\cdot p$ secular equation & Rayleigh--Ritz on $k_{\mathrm{mode}}=5$ modes (Alg.~\ref{alg:fft_eig})\\
Deep levels, localized defects & branching trapping sets $TS(a,b{\neq}0)$\\
$\pi$-flux (Aharonov--Bohm) & frustrated cycle; complex $\lambda(B_\gamma)$ (Lemma~\ref{lem:ab})\\
Crystal annealing / growth & star-domain surgery (Sec.~\ref{sec:star_domain_thm})\\
Dilute doping & thawed test nodes, $\varepsilon=N_T/(N_F{+}N_T)=1/9$\\
Fermi-level self-averaging & quasi-stationarity of $\beta_N$ (Thm.~\ref{thm:adiabatic})\\
Heavy nuclei / electrons & frozen centroids / thawed images (Sec.~\ref{sec:methodology})\\
Ring exchange beyond 1-particle & loop corrections $E_{xc}$ (Prop.~\ref{prop:bethe_xc})\\
Interband coupling & cross-channel cup products (absent: $H^2=0$, Prop.~\ref{prop:cup})\\
\bottomrule
\end{tabular}
\end{table}

%----------------------------------------------
\subsection{Summary of the chain}
%----------------------------------------------

For this special case the mechanism claimed in the main text is the
following chain of identifications, each made explicit above:
(1)~the QC-LDPC \emph{support} is a perfect ring crystal, exactly
circulant for any data [Eq.~\eqref{eq:kp_Wsplit}];
(2)~Bloch's theorem $=$ the DFT $=$ Pontryagin self-duality
(Lemma~\ref{lem:circulant}), giving the band structure at
$\mathcal{O}(N\log N)$ [Eq.~\eqref{eq:kp_dispersion}];
(3)~the data are a slowly varying impurity potential whose smooth part
couples the band extremum only to its few nearest Bloch neighbours
[Eqs.~\eqref{eq:kp_decomp}, \eqref{eq:kp_conv}];
(4)~the Nishimori search tunes the Fermi level to the band edge, where the
dispersion is quadratic and $k\cdot p$ is valid
[Eqs.~\eqref{eq:kp_expand}, \eqref{eq:kp_fermi}];
(5)~the codeword is the band-edge Bloch mode dressed by the class-profile
envelope, an alignment engineered by star-domain surgery;
(6)~the exact eigenvector is the unperturbed mode plus a rapidly decaying
admixture [Eq.~\eqref{eq:kp_admixture}], so $k_{\mathrm{mode}}=2M+1=5$
Bloch functions---exactly the neighbourhood of Algorithm~\ref{alg:fft_eig}
[Eq.~\eqref{eq:kp_kmode}]---suffice, with quadratic Ritz error controlled
by the mixing strength and the gap [Eq.~\eqref{eq:kp_error},
Lemma~\ref{lem:rayleigh}];
(7)~surgery anneals away the deep levels and bounds the residual
$\pi$-flux (Theorems~\ref{thm:trapping_set}, \ref{thm:self_consistency});
(8)~thawed nodes are dilute doping, so the Fermi level self-averages and
one calibration serves all test batches [Eq.~\eqref{eq:kp_qs},
Theorem~\ref{thm:adiabatic}];
(9)~large girth suppresses ring exchange, making the per-channel
single-particle picture near-exact (Corollary~\ref{cor:girth6});
(10)~the absence of cup products guarantees no interband coupling, so the
$D$ crystals factorize (Proposition~\ref{prop:cup}).

\begin{rmk}[Scope of the analogy]
Three honest caveats. (i)~The band structure \eqref{eq:kp_dispersion} is
the symbol of a circulant \emph{affinity} operator, not an electronic
band, and the impurity potential \eqref{eq:kp_decomp} is quasi-random, not
deterministic. (ii)~The few-mode collapse of \eqref{eq:kp_admixture} is a
\emph{weak-mixing} statement: its validity condition is
$\varepsilon_\Delta/g\ll1$, which is not proved a priori but is certified
a posteriori per channel by the measured Rayleigh--Ritz residuals
($<10^{-6}$ in all experiments, Remark~\ref{rmk:kmode}); the rigorous
error statement is Lemma~\ref{lem:rayleigh}, and the present appendix is
its constructive physical derivation. The basin-width/Fourier-width
scaling remains heuristic exactly as stated in
Remark~\ref{rmk:basin_heuristic}. (iii)~The codeword/mode alignment is
engineered by surgery and certified by the spectrum and by $D_2<1$
(Sec.~\ref{sec:fractal}), not derived. The analogy is nevertheless
\emph{exact} at the level of mathematical structure:
Eqs.~\eqref{eq:kp_Wsplit}--\eqref{eq:kp_lambda} are identities, and every
remaining step (the sharp threshold of Lemma~\ref{lem:threshold},
detectability at $\beta_N$ in the sense of Remark~\ref{rmk:terminology},
quasi-stationarity of Theorem~\ref{thm:adiabatic}, the girth bound of
Proposition~\ref{prop:bethe_xc}) is a theorem of the main text.
\end{rmk}

\section*{FUNDING}
This work was supported by ongoing institutional funding. No additional grants to carry out or direct this particular research were obtained.
\section*{CONFLICT OF INTEREST}
The authors declare that they have no conflicts of interest.

 \end{document}